\documentclass[twoside,11pt]{article}
\usepackage[abbrvbib, preprint]{jmlr2e}
\setcitestyle{square,comma, open={[}, close={]}}
\setcitestyle{numbers}
\usepackage{amsmath, amssymb}
\usepackage{hyperref}
\hypersetup{hidelinks}
\usepackage{tabularx}
\usepackage[ruled,vlined]{algorithm2e}
\usepackage{algorithmic}
\usepackage{mathrsfs}
\usepackage{pgfplots}
\usepackage[skip=2pt,font=scriptsize]{caption}
\usepackage{float}
\usepackage{bbm}

\usepackage{lipsum}
\usepackage [english]{babel}

\usepackage{amsmath}
\usepackage{bm}

\newtheorem{prob}{Problem}

\newtheorem{thm}{Theorem}[section]{\bfseries}{\itshape}
\newtheorem{lem}{Lemma}[section]{\bfseries}{\itshape}
{\bfseries}{\itshape}
\newtheorem{rema}{Remark}[section]{\bfseries}{\itshape}
\newtheorem{coro}{Corollary}[section]{\bfseries}{\itshape}
\newtheorem{defi}{Definition}[section]{\bfseries}{\itshape}

\DeclareMathOperator{\trace}{Trace}
\DeclareMathOperator{\var}{Var}

\newcommand{\edit}[1]{\textcolor{blue}{#1}}

\newcolumntype{b}{X}
\newcolumntype{s}{>{\hsize=.5\hsize}X}
\newcolumntype{t}{>{\hsize=.3\hsize}X}

\makeatletter
\renewcommand*{\@opargbegintheorem}[3]{\trivlist
  \item[\hskip \labelsep{\bfseries #1\ #2}] \textbf{(#3)}\ \itshape}
\makeatother

\usepackage{lastpage}
\jmlrheading{24}{2023}{1-\pageref{LastPage}}{1/22; Revised
6/23}{11/23}{22-0005}{Shizhou Xu and Thomas Strohmer}
\ShortHeadings{Fair Data Representation for Machine Learning at the Pareto Frontier}{Xu and Strohmer}

\firstpageno{1}

\begin{document}

\title{Fair Data Representation for Machine Learning \\ at the Pareto Frontier}

\author{\name Shizhou Xu \email shzxu@ucdavis.edu \\
       \addr Department of Mathematics\\
       University of California Davis\\
       Davis, CA 95616-5270, USA
       \AND
       \name Thomas Strohmer \email strohmer@math.ucdavis.edu \\
       \addr Department of Mathematics\\
       Center of Data Science and Artificial Intelligence Research\\
       University of California Davis\\
       Davis, CA 95616-5270, USA}
       
\editor{Manuel Gomez-Rodriguez}

\maketitle

\begin{abstract}%   <- trailing '%' for backward compatibility of .sty file
As machine learning powered decision-making becomes increasingly important in our daily lives, it is imperative to strive for fairness in the underlying data processing. We propose a pre-processing algorithm for fair data representation via which supervised learning results in estimations of the Pareto frontier between prediction error and statistical disparity. In particular, the present work applies the optimal affine transport to approach the post-processing Wasserstein barycenter characterization of the optimal fair $L^2$-objective supervised learning via a pre-processing data deformation. Furthermore, we show that the Wasserstein geodesics from the conditional (on sensitive information) distributions of the learning outcome to their barycenter characterize the Pareto frontier between $L^2$-loss and the average pairwise Wasserstein distance among sensitive groups on the learning outcome. Numerical simulations underscore the advantages: (1) the pre-processing step is compositive with arbitrary conditional expectation estimation supervised learning methods and unseen data; (2) the fair representation protects the sensitive information by limiting the inference capability of the remaining data with respect to the sensitive data; (3) the optimal affine maps are computationally efficient even for high-dimensional data.
\end{abstract}


\section{Introduction}

% importance of the topic
Our society is increasingly influenced by artificial intelligence as (direct or indirect) decision-making processes become more reliant on statistical inference and machine learning. The potentially significant long-term impact from sequences of automated (facilitate of) decision-making has brought large concerns about bias and discrimination in machine learning \cite{angwin2022machine, sweeney2013discrimination}.  Machine learning based on unbiased algorithms can naturally inherit the historical biases that exist in data and hence reinforce the bias via automated decision-making process~\cite{calders2013unbiased}.

% disparate treatment vs. disparate impact
One straightforward partial remedy is to exclude the sensitive variables from the data set used in the learning and decision process. But such exclusion merely eliminates disparate treatment, which refers to direct discrimination, and leaves disparate impact, which refers to unintended or indirect discrimination, remaining in both data and learning outcome \cite{feldman2015certifying}. Examples of the legal doctrine of disparate impact include Griggs v.\ Duke Powers Co.\ \cite{blumrosen1972strangers} and Ricci v.\ DeStefano \cite{adamson2011ricci}, where the decision is based on factors that are strongly correlated with race, such as intelligence qualification in the former and the racially disproportionate test result in the latter, are ruled illegal by the US supreme court. As a result, along with the trending development of automated decision making, the need for more sophisticated but practical techniques has made fairness in machine learning an important research area \cite{executive2014big}.

% statistical parity vs. individual fairness
Two important but potentially conflicting goals of fair machine learning are {\em group fairness}, which aims to achieve similarity in predictions conditioned on sensitive information, and {\em individual fairness}, which aims for similar treatment of similar individuals regardless of the sensitive information. The present work targets an important definition in group fairness: {\em statistical parity} \cite{dwork2012fairness}, because it is closely related to disparate impact and hence long-term structural influence \cite{zhou2021bias}, while individual fairness focuses more on the short-term individual consequence. %\edit{Then, we extend the solution developed for statistical parity to equalized odds, which is another important definition of group fairness.} 
In the remainder of this paper, fairness and statistical parity are used interchangeably\footnote{There are many other notions of fairness, such as equalized odds or equal opportunity, which all have their benefits and shortcomings~\cite{chouldechova2018frontiers}. A discussion of the advantages or disadvantages of the different concepts of fairness is beyond the scope of this paper.}.% \edit{, except for Section \ref{s:Equalized Odds} where the focus is on equalized odds.}

% fairness in general
Before further discussing statistical parity, we note that fairness in machine learning should not be defined by a single condition without considering the application context. The goal of the present work is to provide theoretically reliable and explainable tools to help practitioners obtain the optimal (w.r.t.\ utility) solutions at any chosen statistical disparity level, provided one chooses to adopt statistical parity (or limited statistical dependence between the learning outcome and the sensitive information) as a meaningful fairness definition in one's particular application context.

Remark~\ref{remark1} below provides a more detailed discussion on statistical parity, namely how the utility optimization solves some major insufficiency of the original statistical parity definition and improves statistical parity to proportional equality, a fairness concept similar to equity in modern ethics which can be traced back to Aristotle and Plato \cite{aristotle1984complete, cooper1997plato}.

% Advantage and disadvantages of statistical parity
\begin{rema}[Statistical parity enhanced by utility optimization]\label{remark1}
Statistical parity is one of the most important definitions of group fairness. It has advantages such as (1)~legal support on mitigating adverse impact and (2) the long-term effect resulting from the enforced involvement of minority groups or diversity in learning outcome via affirmative action \cite{hu2018short}. On the other hand, there are three major criticisms about statistical parity that are often mentioned, e.g.\ see~\cite{dwork2012fairness,hardt2016equality}: (1)~reduced utility, (2)~self-fulfilling prophecy, (3)~subset targeting. However, we notice that the first two are insufficiencies with respect to utility. Therefore, the proposed method mitigates these two insufficiencies.
\begin{itemize}
\item[1] (Utility) The development of the Pareto frontier allows us to achieve a desirable statistical disparity level with theoretically provable minimum (hence necessary) utility sacrifice. Equivalently, practitioners can choose a tolerable utility sacrifice level so that the Pareto frontier will provide a learning outcome with the minimum statistical disparity while not violating the utility sacrifice tolerance.
\item[2] (Self-fulfilling prophecy) As mentioned in \cite{dwork2012fairness,hardt2016equality}, self-fulfilling prophecy results from random, careless, or malicious selection in minority groups. But the barycenter characterization method guarantees the optimal fair model to make good selections in all sensitive groups to maximize utility. Section \ref{s:Challenge and Contribution} contribution point 4 and Section \ref{s:prelim general distribution case} provides, respectively, the intuitive and technical explanation of how the utility maximization enforces the model to give similar learning outcomes to data points sharing relatively (within their sensitive groups) similar qualifications. For example, if race is the sensitive information and an admission test score is the only qualification variable, a barycenter-characterized optimal fair admission model would give admission to the same percentage of top-score students in each of their racial groups.
\end{itemize}
Interestingly, the interpretation is consistent with the philosophical definition of fairness involving proportional equality: a model is fair (with respect to the sensitive information) if it distributes proportional chance or prediction to proportionally qualified independent variables within each of the sensitive groups. 
\end{rema}

% different approaches to fairness
Beginning with \cite{dwork2012fairness}, there is now a sizable body of research studying fair machine learning solutions. The resulting approaches can be categorized into the following: (1) pre-processing: deform data before training to mitigate sensitive information in the learning outcome \cite{calmon2017optimized, kamiran2012data}; (2) in-processing: implement the definition of fairness in the training process by penalizing unfair outcome \cite{berk2017convex, zafar2017fairness}; (3) post-processing: enforce the definition of fairness directly on the learning outcome \cite{hardt2016equality, jiang2020wasserstein}.

% relative works to us
In recent years, the post-processing approach has received significant attention due to the following remarkable result: the optimal fair distribution of supervised learning, such as classification \cite{jiang2020wasserstein} and regression \cite{chzhen2020fair, gouic2020projection}, can be characterized as the Fr\'echet mean of the learning outcome marginals on the Wasserstein space, which is also known as the Wasserstein barycenter in the optimal transport literature. (See Remark~\ref{r:name of marginals} for more details on learning outcome marginals.) The following remark provides an intuition of the Wasserstein ($\mathcal{W}_2$) barycenter characterization, on which we develop our theoretical results and algorithms.

% intuition behind wasserstein barycenter characterization
\begin{rema}[Intuition of Wasserstein barycenter characterization]\label{r:Intuition of Wasserstein Barycenter Characterization}
The Fr\'echet mean is the closest point to a set of points in a metric space and, therefore, a generalization of the mean on the Euclidean space to general metric spaces such as the Wasserstein space. Intuitively, one can consider the barycenter (Fr\'echet mean in Wasserstein space) characterization of optimal fair learning outcome as an analog of representing a set of points by their average, which thereby optimally (with respect to total moving distance) removes the disparity among those points, except that each point is now in Wasserstein space, and hence a distribution. 
% \edit{In fact, this Euclidean analog is exactly what happened in the point-wise scope on the Wasserstein barycenter. \ts{What does pointwise scope mean?}
See Section \ref{s:Challenge and Contribution} contribution point 4 below for more details.
\end{rema}

Despite the theoretical elegance of the post-processing barycenter characterization, challenges remain in theory and practice (see Section \ref{s:Challenge and Contribution} for a detailed explanation of the challenges), especially compared to pre-processing or data representation methods.

Fair machine learning using a pre-processing approach has been considered in~\cite{calmon2017optimized,feldman2015certifying, hajian2012methodology, silvia2020general, kamiran2012data}. While the Wasserstein barycenter provides a mathematically rigorous characterization of the post-processing optimal learning outcome, optimal fair data representation for general supervised learning models still lacks a theoretical characterization. See, for example, \cite[Section 3.4, 3.5]{chouldechova2018frontiers} for more details on the current challenges in fair data representation design for general machine learning models beyond classification, not to mention data representations that provide the optimal trade-off between accuracy and fairness.
% or methods to solve or estimate the characterization.}

The goal of the present work is to develop an optimal fair data representation characterization so that any supervised learning model, which aims to estimate the conditional expectation, trained via the fair data representation results in a fair estimation of the post-processing Wasserstein barycenter characterized optimal fair learning outcome. The ultimate goal is to develop a method that enjoys both the mathematically rigorous characterization of post-processing and the flexibility of pre-processing.

\subsection{Optimization Problems with Sensitive Variable Independence Constraint} \label{s:Theoretical Contribution}

The statistical parity constraint for supervised learning or data representation in a nutshell is a constraint on the dependence between the learning outcome and a chosen sensitive variable. Equivalently, the constraint limits the ability to access or reverse engineer the sensitive variable from the learning outcome or data representation. Therefore, although the theory and methods in the present work aim to solve current challenges in machine learning fairness, they can also be useful in other areas where sensitive or undesirable information needs to be eliminated within the existing learning outcome or data. One example of such an area other than fair machine learning is machine (feature) unlearning. It starts from \cite{cao2015towards} and now has a sizable body of research works.

Here, we summarize the constrained optimization problems solved in the present work. We prove existence (and uniqueness, if possible) results via a constructive characterization approach so that an explicit formula of the solutions becomes available. Practitioners and researchers interested in limiting the statistical dependence between the learning outcome or data representation and certain feature variables can directly refer to the corresponding section for results. We leave the underlying motivations resulting from machine learning fairness to the following two subsections. 

In Section \ref{s:Post-Processing}, we target the following problem: \begin{prob}[Optimal fair $L^2$-objective learning outcome]\label{prob:Optimal Fair $L^2$-objective Learning Outcome}
\begin{equation} \label{eq:post-processing constrained optimization for optimal learning outcome}
\inf_{f \in L^2(\mathcal{X} \times \mathcal{Z},\mathcal{Y})} \{||Y - f(X,Z)||_2^2 : f(X,Z) \perp Z\}
\end{equation}
\end{prob}
Here, $Y$ is the dependent variable, and $f(X,Z)$ is an estimator that uses the independent variable $X$ and sensitive variable $Z$ to estimate $Y$. The loss function aims to maximize utility by minimizing the $L^2$-norm between $Y$ and $f(X,Z)$: $$||Y - f(X,Z)||_2^2 = \int_{\Omega} ||Y - f(X,Z)||^2 d\mathbb{P}.$$ $(\Omega,\Sigma,\mathbb{P})$ is a probability space. For $S \in \{X,Y,Z\}$, $S: \Omega \rightarrow \mathcal{S}$ is a random variable (equivalently a measurable function) from $\Omega$ to the state space $\mathcal{S}$. $||\cdot||$ denotes the Euclidean norm.  The constraint $f(X,Z) \perp Z$ guarantees that the final result is independent of the sensitive information $Z$ and hence satisfies statistical parity. Finally, the admissible function space $L^2(\mathcal{X} \times \mathcal{Z},\mathcal{Y})$ is the space of all square-integrable measurable functions from $\mathcal{X} \times \mathcal{Z}$ to $\mathcal{Y}$. (Our proof shows Problem \ref{prob:Optimal Fair $L^2$-objective Learning Outcome} does not change if one allows all measurable functions $\mathcal{X} \times \mathcal{Z}$ to $\mathcal{Y}$.) The reason of allowing all measurable functions in our problem setting is due to the recent development of deep neural networks that are capable of estimating arbitrary measurable functions.

In Section \ref{s:Pareto Frontier}, we relax the above strict independence constraint by applying a quantification of statistical disparity: the {\em Wasserstein disparity}, which is the average pairwise Wasserstein distance among conditional (on $Z$) distributions of $f(X,Z)$, denoted by $D(f(X,Z),Z)$. It has the following desirable properties: (1) $D(f(X,Z)) = 0$ if and only if $f(X,Z) \perp Z$. (2) The larger $D$ is, the more disparities there are among the marginals (w.r.t. $Z$) of $f(X,Z)$. (3) $D$ has a meaningful interpretation in physics as the minimum expected amount of work required to remove the distributional discrepancy between two randomly chosen sensitive groups on the learning outcome. Therefore, fixing a disparity tolerance level $d \in [0,\infty)$, \begin{prob}[Optimal $L^2$-objective learning Pareto frontier]\label{prob:Optimal $L^2$-objective Learning Pareto Frontier}
\begin{equation}
\inf_{f \in L^2(\mathcal{X} \times \mathcal{Z},\mathcal{Y})} \{||Y - f(X,Z)||_2^2 : D(f(X,Z),Z) < d\}
\end{equation}
\end{prob}
gives us the corresponding Pareto optimal solution. That is, if one wants a lower $L^2$-loss than provided by the infimum in Problem \ref{prob:Optimal $L^2$-objective Learning Pareto Frontier}, then it is necessary to increase the tolerance level $d$. Equivalently, if one wants to lower the tolerance level $d$, then it is necessary to sacrifice more $L^2$-loss than the infimum.

In Section \ref{s:Optimal Fair Data Representation}, we provide a theoretical characterization of the solution to \begin{prob}[Optimal fair data representation for conditional expectation estimation]\label{prob:Optimal Fair Data Representation for Conditional Expectation Estimation}
\begin{equation}\label{eq:optimal fair data representation problem}
\inf_{(\tilde{X}, \tilde{Y}) \in \mathcal{D}} \{ ||Y - \mathbb{E}(\tilde{Y}|\tilde{X})||^2_2 : \tilde{X}, \mathbb{E}(\tilde{Y}|\tilde{X},Z) \perp Z\},
\end{equation}
\end{prob}
where $\mathcal{D}$ is the admissible data representation set we define later. Here, the objective function aims to maximize the potential utility remaining within the deformed data $(\tilde{X},\tilde{Y})$ by minimizing the $L^2$ distance between the perfect estimator $\mathbb{E}(\tilde{Y}|\tilde{X})$ on $(\tilde{X},\tilde{Y})$ and the original $Y$, so that better estimation of $\mathbb{E}(\tilde{Y}|\tilde{X})$ leads to better prediction of $Y$. The constraint $\tilde{X}, \mathbb{E}(\tilde{Y}|\tilde{X},Z) \perp Z$ guarantees: (1)~$f(\tilde{X}) \perp Z$ for $\forall f: \mathcal{X} \rightarrow \mathcal{Y}$, such that any estimator of $E(\tilde{Y}|\tilde{X})$ is independent of $Z$; (2)~The perfect adversarial estimator $\mathbb{E}(\tilde{Y}|\tilde{X},Z)$ is independent of $Z$, so that a better estimation of $E(\tilde{Y}|\tilde{X},Z)$ leads to more independence of $Z$ (alignment between the training objective and independence constraint).  In addition, one may choose the following alternative constraints according to the application context: (1)~$\tilde{X} \perp Z$, which guarantees $f(\tilde{X}) \perp Z$ for all measurable $f$ as mentioned above; (2)~$(\tilde{X}, \tilde{Y}) \perp Z$, which guarantees any (adversarial) supervised or unsupervised learning on $(\tilde{X},\tilde{Y})$ to be independent of $Z$. The first alternative is useful if only measurable functions of $X$ are allowed, whereas the second should be applied when one does not know which features are dependent or independent. See Section \ref{s:Post-processing and Pre-processing Approach} for a more detailed derivation and explanation of the data representation objective function and constraints.

\subsection{Challenges and Contributions in Machine Learning Fairness} \label{s:Challenge and Contribution}

Now, we go back to the motivation behind the optimization problems listed above: fair machine learning. We first summarize the limitations of the current post-processing characterization and the current methods based on it to estimate the optimal fair learning outcome. 

\begin{enumerate}
\item The post-processing barycenter characterization lacks theoretical and computational generalization to high-dimensional data spaces, such as text or image spaces. From a theoretical perspective, the current works \cite{chzhen2020fair, gouic2020projection, silvia2020general} focus on classification and 1-dimensional regression. From a computational perspective, the current works apply the coupling of cumulative distribution functions (cdf) of the learning outcome sensitive conditionals to find the barycenter and the inverse of the cdf to compute the optimal transport map. Both the coupling and the inverse of the cdf are computationally expensive. Furthermore, since the inverse of the cdf cannot be generalized to high-dimensional spaces, the current methods lack the generalization to supervised learning with high-dimensional dependent variables.

Due to the recent development of generative AI models, it is now important to have fair machine learning methods for arbitrarily high-dimensional data. We hope the present work on the $L^2$ space can be a starting point for fair machine learning or data representation on more general spaces for high-dimensional data.

\item The current post-processing barycenter characterization lacks both theoretical and computational generalization to (an estimation of) the optimal trade-off, also known as the Pareto frontier, between prediction accuracy and fairness. In theory, there is a lack of characterization of the Pareto frontier (optimal trade-off) between utility and fairness. Current works on the Pareto frontier, such as \cite{silvia2020general}, apply tight inequalities based on the convexity of distance metrics to suggest the optimal trade-off coincide with the Wasserstein geodesic path. While such inequalities are tight for a broad type of metrics on the space of probability measures, they are {\em not tight} for the Wasserstein metric. Hence, the inequalities are not able to extend the mathematically rigorous Wasserstein barycenter characterization of the optimal fair learning outcome to a Pareto frontier. From a computational perspective, current methods, such as \cite{silvia2020general}, apply interpolation between the inverses of the sensitive conditional cdf's (more specifically, interpolating the data points that share the same image under the sensitive conditional cdf's) to estimate the geodesics. In addition to the drawbacks mentioned  above, the inverse of the cdf also does not come with an explicit form, which makes the computation of an interpolation between two cdf inverses even more cumbersome.

\item The post-processing nature of the characterization requires explicit or implicit sensitive information in the training and decision-making process. More specifically, in order to apply the barycenter characterization to find the optimal fair learning outcome or to make predictions to newly incoming data, one needs the following steps: (1) Estimate the conditional expectation and obtain its conditional distributions with respect to the sensitive information; (2) Find the Wasserstein barycenter of the sensitive conditionals of the conditional expectation estimation or the learning outcome; (3) Compute the optimal transport maps from each sensitive conditional to the barycenter; (4) Apply each transport map to the conditional with the matched sensitive information. Here, not only does the trained model still inherit unfairness, but it is also clear that sensitive information needs to be attached to both the dependent variable or incoming data and its learning outcome or prediction, until the very last post-processing step of finding the barycenter comes to the rescue. Hence, we say that the characterization has a post-processing nature. As a result, the user needs access to the sensitive information of each individual incoming data at every step during the learning process. Such a strong access to sensitive information makes the supervised learning process vulnerable to attack and sensitive information leakage.

The post-processing nature of the characterization also suffers from the lack of flexibility in model selection, modification, and composition. For model selection and modification, a practitioner would have to perform the post-processing step for every model and every modification in order to compare the corresponding optimal fair learning outcomes. See Table \ref{time table} for more details on the additive computational cost of the post-processing approach compared to the one-time cost of the proposed pre-processing approach. For model composition, we consider the simple example $task_2 \circ task_1$ where $task_i, i \in \{1,2\}$ are trained supervised learning models. In practice, there is a good chance that $task_1$ and $task_2$ belong to different practitioners or organizations, denoted by practitioners $1$ and $2$, respectively. Therefore, to protect sensitive information from practitioner $2$, practitioner $1$ will perform the post-processing step to obtain a fair learning outcome and provide it as an input variable for the training task of practitioner $2$. But unless $task_2$ needs no more input variables other than the dependent variables of $task_1$ (in that case, $task_1$ would be fair data representation design), still practitioner $2$ needs full access to the sensitive variable attached to its input data, which includes the desensitized $task_1$ output and other input variables. Such attachment makes the post-processing step performed by practitioner $1$ meaningless. Considering the recent development of decentralized learning in practice, such drawback in model composition makes a model-independent fair data representation more applicable than a post-processing solution.

\item Many of the current fair machine learning methods are proposed without utility guarantee or explainability. Such a lack of utility guarantee or explainability prevents the study of fair machine learning from practical use. For instance, Wells Fargo \cite{zhou2021bias} concluded recently that current fair machine learning methods are black-box methods, and hence they hesitate to adopt fair machine learning techniques.
\end{enumerate}

We provide a road map of the tools that we have developed in response to each of the listed challenges and how the present work combines all the tools to provide (exact solution and estimation of) the fair data representation at the Pareto frontier.

\begin{enumerate}
\item In response to the theoretical part of the first challenge, Lemma \ref{l:Optimal Fair $L^2$-Objective Supervised Learning Characterization} in Section \ref{s:Post-Processing} provides a characterization (with explicit construction) of the exact solution to Problem \ref{prob:Optimal Fair $L^2$-objective Learning Outcome} (the optimal fair $L^2$-objective learning). The result shows that the infimum loss value of Problem \ref{prob:Optimal Fair $L^2$-objective Learning Outcome} can be nicely decomposed into two parts: (1) $L^2$ orthogonal projection loss and (2) independence projection loss. Also, the result now allows the data spaces $\mathcal{X}, \mathcal{Y}, \mathcal{Z}$ to be $[k]^d, \mathbb{N}^d, [0,l]^d$, or $\mathbb{R}^d$ for arbitrary dimension $d < \infty$. 

To address the challenge of computing the Wasserstein barycenter in high-dimensional data spaces~\cite{altschuler2022wasserstein}, we propose a method that applies affine transport maps to find the {\em optimal affine estimation} of the post-processing optimal fair $L^2$-objective supervised learning outcome with an arbitrarily finite-dimensional dependent variable, which responds to the first challenge listed above. In particular, by restricting admissible transport maps to be affine and making a corresponding relaxation to the fairness constraint, we derive a relaxed version of Problem \ref{prob:Optimal Fair $L^2$-objective Learning Outcome}, stated as Problem \ref{prob:optimal affine estimation of barycenter problem}. Applying the optimal affine transport maps \cite{agueh2011barycenters}, Definition \ref{d:Post-processing Pseudo-barycenter} introduces the post-processing {\em pseudo-barycenter}, Lemma \ref{l:Post-processing Pseudo-barycenter Characterization in Gaussian Case} shows the proposed pseudo-barycenter coincides with the true barycenter when the sensitive conditionals are Gaussian, and finally, Theorem \ref{th:Optimal Affine Estimation of Barycenter: Pseudo-barycenter} proves that the pseudo-barycenter is the optimal affine estimation of the true barycenter in the general conditional distribution case and provides the estimation error. Optimal affine transport and pseudo-barycenter have the advantage of computational efficiency, compared to the current methods, due to the explicit matrix form of the transport map and the nearly closed-form solution to the pseudo-barycenter. 

The importance of optimal affine maps encompasses much more than a solution to the first challenge. The optimal affine maps together with McCann interpolation \cite{mccann1997convexity} help us in obtaining an explicit form of the geodesic path characterization of the Pareto frontier in Section 4. More importantly, Section 5 shows that optimal affine maps and the pseudo-barycenter are necessary tools to overcome the post-processing nature of the Wasserstein barycenter characterization by exploiting the linearity of conditional expectation and thereby generating optimal fair data representations. 

\item In Section 4, we prove an exact characterization of the solution to Problem \ref{prob:Optimal $L^2$-objective Learning Pareto Frontier} (the optimal utility-parity trade-off or Pareto frontier) in response to the theoretical part of the second challenge. In particular, Theorem \ref{th:Geodesics Characterization of the Pareto Frontier} shows that, when utility loss and disparity are quantified respectively by the $L^2$ distance (between the true outcome $Y$ and the prediction $\hat{Y} = f(X,Z)$) and the average pairwise $\mathcal{W}_2$ distance among the sensitive conditionals of $\hat{Y}$, the optimal trade-off happens if and only if the conditionals of $\hat{Y}$ travel along the Wasserstein geodesic path from the conditionals of $\mathbb{E}(Y|X,Z)$ to their barycenter. Therefore, we say that the Pareto frontier is on the Wasserstein space. Corollary \ref{corr:Pareto Optimal Fair L2-objective Learning} then derives an explicit form of the Pareto optimal solution to Problem \ref{prob:Optimal $L^2$-objective Learning Pareto Frontier}. The result is a natural extension to the post-processing Wasserstein barycenter characterization of the optimal fair learning outcome: the barycenter characterization coincides with the point at zero disparity on the Pareto frontier. Interestingly, our result shows that the Pareto frontier is linear.  

To solve the computational challenge of the geodesic path, Remark \ref{r:Linear Interpolation Formula for Geodesic Path} applies McCann interpolation together with the optimal affine maps and the pseudo-barycenter to derive a computationally efficient (nearly) closed-form formula to estimate the Pareto frontier, which results in Algorithm \ref{a:independent}.

\item In response to the third challenge, the present work proposes in Section \ref{s:Post-processing and Pre-processing Approach} Problem \ref{prob:Optimal Fair Data Representation for Conditional Expectation Estimation} (optimal fair data representation problem), which makes the objective function and the fairness (statistical parity) constraint {\em model-independent} and therefore suitable for fair data representation design. More specifically, by applying the Minkowski inequality, we use an objective function to maximize the potential utility remaining in the data. On the other hand, a fair data representation should provide a fairness guarantee to arbitrary $L^2$-objective supervised learning models. Therefore, the present work proposes a pre-processing fairness constraint to guarantee fairness in the learning outcome of arbitrary $L^2$-objective models trained via the fair data representation. 

In Section \ref{s:Optimal Fair Data Representation}, Lemma \ref{l:Characterization of Optimal Fair Data Representation} first provides a characterization of the exact solution to Problem \ref{prob:Optimal Fair Data Representation for Conditional Expectation Estimation} under a mild assumption. Next, Definition~\ref{d:Dependent Pseudo-barycenter} and Definition~\ref{d:Independent Pseudo-barycenter} define the dependent and independent pseudo-barycenter, respectively. Then, similar to solving a relaxation of the post-processing characterization to obtain the optimal affine estimation, Theorem \ref{th:Justification of Dependent Pseudo-barycenter in Gaussian Case} proves that the dependent and independent pseudo-barycenter pair coincides with the true solution to the optimal fair data representation when the conditional data distributions are Gaussian, and Theorem \ref{th:Justification of Pseudo-Barycenter in General Distribution Case} proves that the pseudo-barycenter pair forms the optimal affine estimation of the optimal fair data representation. 

To derive (an estimation of) fair data representation at the Pareto frontier, Corollary~\ref{corr:pre pareto characterization} in Section \ref{s:Optimal Fair Data Representation at the Pareto Frontier} first provides a characterization of the Pareto frontier for conditional expectation on a fixed sigma-algebra. Finally, combining optimal affine map, pseudo-barycenter, together with a diagonal argument in Remark \ref{r:Diagonal Estimate of the Post-processing Pareto Frontier}, we derive an estimation of the fair representation at the Pareto frontier, which results in Algorithm~\ref{a:independent} and Algorithm~\ref{a:dependent}.

Furthermore, in Section \ref{s:Numerics}, experiments show that the proposed fair data representations preserve as large an amount of information (w.r.t.\ the $L^2$ objective) as the fairness constraint allows. Therefore, it provides a better and more flexible solution to fair learning compared to encoding-based data representations ~\cite{calmon2017optimized, zemel2013learning}, which encode the information of the original data into some binary feature variables designed to guarantee statistical parity for classification. Surprisingly, experiments also show that applying the pseudo-barycenter results in nearly zero utility loss compared to the post-processing barycenter characterization solution.

\item In addition to the provable utility guarantee resulting from the Pareto frontier, the proposed method also has a meaningful interpretation from a {\em datapoint-wise perspective} in how it achieves the statistical parity requirement: A data point of the optimal fair learning outcome is the Euclidean average of the optimally matched data points from each of the sensitive groups. Here, matching means partitioning the original data set into subsets consisting of one point from each sensitive group. Each subset is called a match. The points within a match are called matched points. Optimality in matching is equivalent to minimization of the expected variance within a randomly chosen match. Such expected (hence total) variance minimization enforces points with similar relative positions in their sensitive marginal to form a match. For example, assume that there are two sensitive conditionals $A = \{1 \text{ } (\text{low in A}),4 \text{ } (\text{high in A})\}$ and $B = \{2 \text{ } (\text{low in B}),3 \text{ } (\text{high in B})\}$, then the optimal matching is $$\{ \{1 \text{ } (\text{low in A}),2 \text{ } (\text{low in B})\}, \{3 \text{ } (\text{high in B}),4 \text{ } (\text{high in A})\} \}$$ to minimize the expected or total variance within the matches. The optimal matching in high-dimensional $L^2$ spaces shares the same geometric intuition with the simple example. That is, from a point-wise perspective, the optimal fair learning achieves statistical parity by first matching the points with similar relative positions in their sensitive groups and then representing the matched ones with their Euclidean average.

%\item[4] We shed light on the application of the pseudo-barycenter to $L^2$-objective unsupervised learning to achieve diverse data allocation \edit{which is also known as anti-clustering problems, reference will be added} and representation, which provides potential access to fairness in unsupervised learning and deserves further study, see Figure \ref{figure:k-means on barycenter}.
\end{enumerate}

\begin{figure}[H]
\centering
\includegraphics[width=0.8\textwidth]{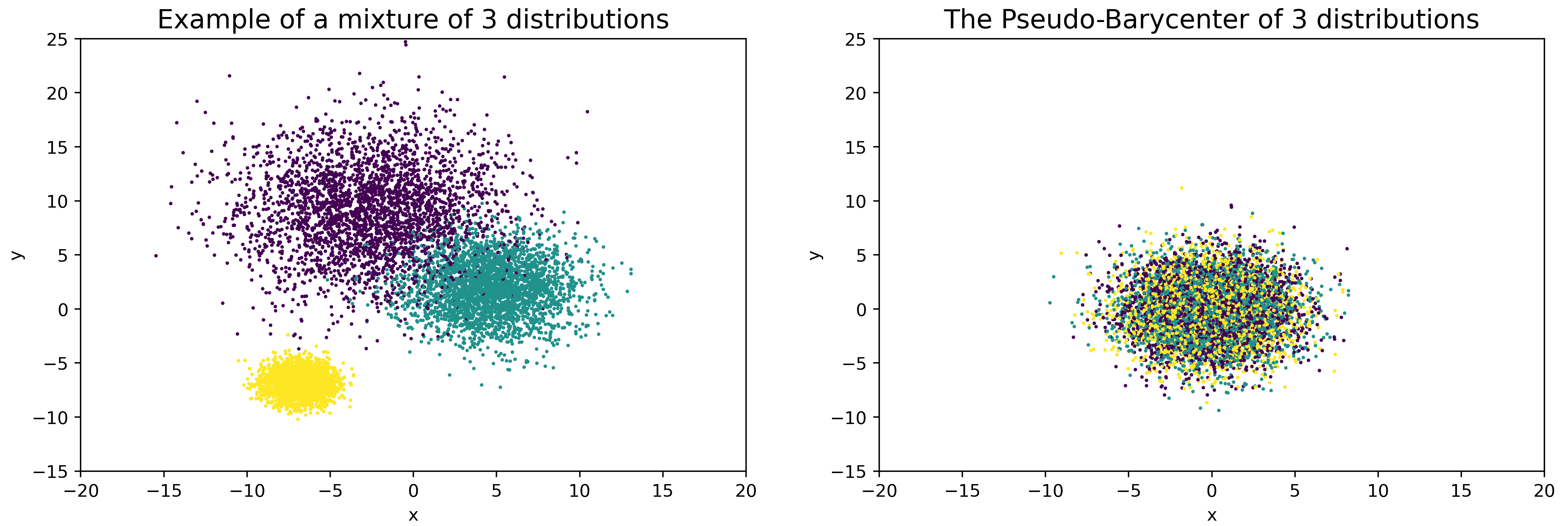}
\caption{The left panel depicts three distributions, sampled from an isotropic Gaussian distribution with different first two moments. The right panel shows the pseudo-barycenter of the three sample distributions.}
%The lower-left panel gives the result of K-means (K=8) on the pseudo-barycenter. The lower-right panel shows the clusters of pseudo-barycenter on the original data. It is clear that data points that share the same pseudo-barycenter K-means label share similar relative positions within their original marginal distributions because the pseudo-barycenter groups together ``similar'' points among the marginals. This provides not only an intuitive explanation for how training via pseudo-barycenter leads to fair models, but also a useful fairness definition for unsupervised learning. For more details, see Section 6.3 below.}
\label{figure:k-means on barycenter}
\end{figure}

\subsection{Fair Data Representations: From Theory to Practice}  \label{s:Post-processing and Pre-processing Approach}

%\subsection{\edit{Post-processing and Pre-processing Approach}} \label{s:Post-processing and Pre-processing Approach}

In this subsection, we derive a fairness objective function that is both theoretically tractable and practically appealing.  This task is more involved than one initially might expect, and it sheds light on some subtleties of both the post-processing and the pre-processing approaches. 

Before proceeding, we need some preparation.
Let $X$, $Y$, and $Z$ represent respectively the independent, dependent, and sensitive random variable, with the same underlying probability space $(\Omega,\Sigma,\mathbb{P})$.  We use the term `random variables' to denote random vectors with an arbitrary but finite dimension. That is, $S: \Omega \rightarrow \mathcal{S}$ where $\mathcal{S} \in \{[k_\mathcal{S} ]^{d_{\mathcal{S}}}, \mathbb{N}^{d_{\mathcal{S}}}, [0,l_\mathcal{S} ]^{d_{\mathcal{S}}}, \mathbb{R}^{d_{\mathcal{S}}}\}$ with $k_\mathcal{S} \in \mathbb{N}, l_\mathcal{S}  \in \mathbb{R}$ and $d_{\mathcal{S}} < \infty$ for $S \in \{X,Y,Z\}$.

It follows from \cite{chzhen2020fair, gouic2020projection} that the optimal fair regression outcome can be characterized by the Wasserstein barycenter. In Lemma \ref{l:Optimal Fair $L^2$-Objective Supervised Learning Characterization} we will generalize their result from regression to all functions in $L^2(\mathcal{X} \times \mathcal{Z},\mathcal{Y})$, which shows that the optimal fair $L^2$-objective supervised learning outcome can be characterized by solutions to Problem \ref{prob:Optimal Fair $L^2$-objective Learning Outcome}:
\begin{equation}
\inf_{f \in L^2(\mathcal{X} \times \mathcal{Z},\mathcal{Y})} \{||Y - f(X,Z)||_2^2 : f(X,Z) \perp Z\}
\end{equation}
The utility loss is quantified by $L^2$-norm: $||Y - f(X,Z)||_2^2 = \int_{\Omega} ||Y - f(X,Z)||^2 d\mathbb{P}$, where $||\cdot||$ is the Euclidean norm. The constraint $f(X,Z) \perp Z$ guarantees that the final result satisfies statistical parity and, therefore, is fair. 

Since it follows from $L^2$ orthogonal decomposition that 
\begin{equation}
||Y - f(X,Z)||_2^2 = ||Y - \mathbb{E}(Y|X,Z)||_2^2 + ||\mathbb{E}(Y|X,Z) - f(X,Z)||_2^2
\end{equation}
and only the second term on the right hand side depends on the choice of $f \in L^2(\mathcal{X} \times \mathcal{Z},\mathcal{Y})$, we conclude that \eqref{eq:post-processing constrained optimization for optimal learning outcome} is equivalent to
\begin{equation} \label{eq:post-processing optimal learning characterization objective}
\inf_{f \in L^2(\mathcal{X} \times \mathcal{Z},\mathcal{Y})} \{||\mathbb{E}(Y|X,Z) - f(X,Z)||_2^2 : f(X,Z) \perp Z\}.
\end{equation}
It turns out---see Lemma \ref{l:Optimal Fair $L^2$-Objective Supervised Learning Characterization}---that the solution to \eqref{eq:post-processing optimal learning characterization objective} is exactly the Wasserstein barycenter. Therefore, we say that the optimal fair $L^2$-objective supervised learning outcome is characterized by the Wasserstein barycenter.  But notice that the Wasserstein barycenter characterization \eqref{eq:post-processing optimal learning characterization objective} assumes knowledge of the learning outcome $\mathbb{E}(Y|X,Z)$. That is, if practitioners apply the characterization to estimate the optimal learning outcome, it is necessary to obtain an estimator of $\mathbb{E}(Y|X,Z)$ via supervised learning before solving the post-processing rescue step \eqref{eq:post-processing optimal learning characterization objective}.  Therefore, we say that the characterization has a post-processing nature and hence call it a post-processing characterization. 
 
Now, notice that the estimator of $\mathbb{E}(Y|X,Z)$ is obtained via the training process
\begin{equation}
\inf_{f \in \mathcal{F}} \{||Y - f(X,Z)||_2^2\},
\end{equation}
where the admissible function set $\mathcal{F}$ depends on the choice of supervised learning models. Denote the estimator by $f'(X,Z)$. Then in practice~\eqref{eq:post-processing optimal learning characterization objective} becomes
\begin{equation} \label{eq:post-processing objective in practice}
\inf_{f \in L^2(\mathcal{X} \times \mathcal{Z},\mathcal{Y})} \{||f'(X,Z) - f(X,Z)||_2^2 : f(X,Z) \perp Z\}.
\end{equation}
That is, the application of the post-processing characterization is model-dependent. The fundamental reason for model dependence is that \eqref{eq:post-processing constrained optimization for optimal learning outcome} is optimizing over all $L^2$ functions while in practice it is necessary to reduce the admissible set from $L^2$ to some $\mathcal{F}$ which depends on the choice of the model. As a result, the optimizer is necessarily dependent on the choice of the model. Therefore, the constrained optimization \eqref{eq:post-processing constrained optimization for optimal learning outcome} and its characterization are not suitable for our ultimate goal of deriving a model-independent pre-processing approach to the optimal fair learning outcome. The present work proposes a different constrained optimization problem that characterizes the optimal fair data representation for all $L^2$-objective supervised learning models. 

To make a constraint optimization problem suitable for fair data representation design, we require both the objective function and the fairness constraint to be model-independent. Furthermore, the data representation design objective and the training objective given the data representation have to be consistent in the following sense: the better training and testing result on the fair data representation leads to less $L^2$-fitting error with respect to the true data.

We now derive an objective function that is suitable for fair data representation design purpose. To start, notice that our goal is to generate a synthetic data representation $(\tilde{X},\tilde{Y})$, a deformation of $(X,Y)$, via which any $L^2$-objective model that is trained by
\begin{equation} \label{eq:total objective via fair data representation and training}
\inf_{f \in \mathcal{F}} ||\tilde{Y} - f(\tilde{X})||^2_2
\end{equation}
would result in (an estimation of) the optimal fair learning outcome. In the rest of this paper we denote the solution to \eqref{eq:total objective via fair data representation and training} by $f_{\tilde{Y}}$.

Also, because conditional expectation is an orthogonal projection operator on $L^2$-space, we obtain the following orthogonal decomposition of the objective in \eqref{eq:total objective via fair data representation and training}:
\begin{equation}
||\tilde{Y} - f(\tilde{X})||^2_2 = ||\tilde{Y} - \mathbb{E}(\tilde{Y}|\tilde{X})||^2_2 + ||\mathbb{E}(\tilde{Y}|\tilde{X}) - f(\tilde{X})||^2_2.
\end{equation}
Only the second term on the right hand side depends on the choice of $f \in \mathcal{F}$, hence the training step objective \eqref{eq:total objective via fair data representation and training} is equivalent to the following:
\begin{equation}\label{eq:training objective via fair data representation}
\inf_{f \in \mathcal{F}}  ||\mathbb{E}(\tilde{Y}|\tilde{X}) - f(\tilde{X})||^2_2.
\end{equation}
Thus, the solution to \eqref{eq:training objective via fair data representation} is also $f_{\tilde{Y}}$, which depends on the choice of $\mathcal{F}$. 

The key observation is that, given a data representation $(\tilde{X}, \tilde{Y})$, \eqref{eq:training objective via fair data representation} is the objective that practitioners try to achieve via model selection, modification, and parameter turning. Furthermore, it follows from the triangle or Minkowski inequality that
\begin{equation}\label{eq:decomposition of total objective}
\underbrace{||Y - f_{\tilde{Y}}(\tilde{X})||_2}_{\text{total utility loss}} \leq \underbrace{||Y - \mathbb{E}(\tilde{Y}|\tilde{X})||_2}_{\text{data representation utility loss}} + \underbrace{||\mathbb{E}(\tilde{Y}|\tilde{X}) - f_{\tilde{Y}}(\tilde{X})||_2}_{\text{learning utility loss}}.
\end{equation}
The second term on the right-hand side is the target of a supervised learning task which should be left to practitioners. Thus, the natural choice of the model-independent objective of the optimal fair synthetic data design is to minimize the first term:
\begin{equation}\label{eq:optimal fair data representation objective}
\inf_{(\tilde{X},\tilde{Y}) \in \mathcal{D}}||Y - \mathbb{E}(\tilde{Y}|\tilde{X})||_2,
\end{equation}
where $\mathcal{D}$ is some admissible set of deformed versions of the original data $(X,Y)$ that we define later. Intuitively, the loss function can be interpreted as the potential utility sacrifice resulting from deforming $(X,Y)$ to $(\tilde{X},\tilde{Y})$ for $L^2$-objective supervised learning, while leaving the task of minimizing the second term on the right-hand side to practitioners via model selection, modification, or parameter tuning.

Next, we derive a fairness constraint for synthetic data design purposes. That is, the goal is to design $(\tilde{X},\tilde{Y})$ such that $f_{\tilde{Y}}(\tilde{X}) \perp Z$ for any admissible function set $\mathcal{F} \subset L^2(\mathcal{X},\mathcal{Y})$.  The flexibility of model choice becomes important due to the increasing complexity of models in practice nowadays, such as neural networks. The key observation here is that, due to the potential dependence of $f_{\tilde{Y}}$ on $Z$, one needs to look at both models that use merely measurable functions from $\mathcal{X}$ to $\mathcal{Y}$ and more complicated models consisting of $Z$-dependent measurable functions:
\begin{itemize}
\item[1] For measurable functions from $\mathcal{X}$ to $\mathcal{Y}$, if we require $\tilde{X} \perp Z$, then it follows that for any $f: \mathcal{X} \rightarrow \mathcal{Y}$, it is guaranteed that $f(\tilde{X}) \perp Z$. Hence, we require $\tilde{X} \perp Z$ to prevent models from exploiting sensitive information from the independent variables.
\item[2] For advanced or adversarial models that use $Z$-dependent functions from $\mathcal{X} \times \mathcal{Z}$ to $\mathcal{Y}$, the trained model $f_{Y}$ could still depend on $Z$ because $Y$ and $Z$ are not independent. For example, consider the extreme case where $Y = kZ, k \in \mathbb{R}$ and a perfect model results in $\mathbb{E}(kZ|\tilde{X},Z) = kZ$ which fully depends on $Z$ even if we require $\tilde{X} \perp Z$. Therefore, we also require $f_{\tilde{Y}}(\tilde{X},Z) \perp Z$ to prevent such a model from exploiting sensitive information from the dependent variables.
\end{itemize}
But notice that the second requirement leads us back to the post-processing nature of fairness constraints as in \eqref{eq:post-processing objective in practice}. For fair data representation design purposes,  it is necessary to keep the constraint model-independent. Therefore, instead of enforcing $f_{\tilde{Y}}(\tilde{X},Z) \perp Z$, the present work requires $\mathbb{E}(\tilde{Y}|\tilde{X},Z) \perp Z$ for the following two reasons: (1) Under the modified constraint $\mathbb{E}(\tilde{Y}|\tilde{X},Z) \perp Z$, the better $f_{\tilde{Y}}(\tilde{X},Z)$ estimates $\mathbb{E}(\tilde{Y}|\tilde{X},Z)$, the more independent of $Z$ becomes $f_{\tilde{Y}}(\tilde{X},Z)$. Such alignment between training objective and fairness makes the modification a natural choice under the assumption that the goal of $L^2$-objective (adversarial) supervised learning tasks is to minimize $||\mathbb{E}(\tilde{Y}|\tilde{X},Z) - f_{\tilde{Y}}(\tilde{X},Z)||^2_2$, which is equivalent to minimizing $||\tilde{Y} - f_{\tilde{Y}}(\tilde{X},Z)||^2_2$. (2) Since a supervised learning model with poor prediction accuracy already results in severe unfairness, the dependence on sensitive information is of less concern when designing a fair data representation. 

Based on the fairness requirement for both measurable functions on merely $\mathcal{X}$ and $Z$-dependent functions, a natural choice of (pre-processing) statistical parity constraint for data representation has the following form:
\begin{equation}\label{eq:optimal fair data representation constraint}
\tilde{X}, \mathbb{E}(\tilde{Y}|\tilde{X},Z) \perp Z.
\end{equation}
It guarantees: (1) statistical parity for any model that uses only a deterministic function and any model that results in a perfect estimation of $\mathbb{E}(\tilde{Y}|\tilde{X})$; (2) the better $f_{\tilde{Y}}(\tilde{X},Z)$ estimates $\mathbb{E}(\tilde{Y}|\tilde{X},Z)$, the more independent  $f_{\tilde{Y}}(\tilde{X},Z)$ becomes of $Z$. 

While the fairness constraint \eqref{eq:optimal fair data representation constraint} is not the only choice, it does balance utility and fairness. The following remark discusses two alternative fairness constraint choices, which are more polarized in optimizing utility or fairness.

\begin{rema}[Alternative fair data representation constraints]
There are two alternative choices of fairness constraints that are valuable in practice:
\begin{itemize}
\item[1] $\tilde{X} \perp Z$: the weaker constraint guarantees any model using merely a deterministic function, even if sub-optimal, to result in statistical parity. But it does not protect $Z$ from advanced models, which exploit the dependence of $Y$ on $Z$ and apply $Z$-dependent functions. Therefore, $\tilde{X} \perp Z$ provides more utility but less sensitive information protection, compared to our choice.
\item[2] $(\tilde{X},\tilde{Y}) \perp Z$: the stronger constraint guarantees statistical parity in the learning outcome of any supervised learning model, even for those that adopt $Z$-dependent functions and are suboptimal. But it sacrifices more utility. This stronger constraint is particularly useful in practice when one does not know which variables are dependent and which ones are independent.
\end{itemize}
Our choice is a compromise of the two alternatives in terms of balancing utility sacrifice and protecting sensitive information. Furthermore, simple modifications of our analysis and algorithm would solve the two alternatives because they are essentially simplified versions of our choice. Hence, the present work targets  \eqref{eq:optimal fair data representation constraint}.
\end{rema}

Finally, combining the objective and constraint for synthetic data design, we aim to solve Problem \ref{prob:Optimal Fair Data Representation for Conditional Expectation Estimation}:
\begin{equation}
\inf_{(\tilde{X}, \tilde{Y}) \in \mathcal{D}} \{ ||Y - \mathbb{E}(\tilde{Y}|\tilde{X})||^2_2 : \tilde{X}, \mathbb{E}(\tilde{Y}|\tilde{X},Z) \perp Z\}.
\end{equation}
The solution provides a fair data representation via which the trained $L^2$-objective supervised learning models become estimations of the optimal fair conditional expectation.

Compared to the original constrained optimization problem \eqref{eq:post-processing constrained optimization for optimal learning outcome} which results in the post-processing nature of its barycenter characterization \eqref{eq:post-processing optimal learning characterization objective}, the proposed constrained optimization problem \eqref{eq:optimal fair data representation problem} has the following advantages by design:
\begin{itemize}
\item[1] It provides a fairness guarantee for arbitrary $L^2$-objective models.
\item[2] The model-independence together with the alignment between training objective and fairness enables practitioners to enjoy flexibility in model selection, modification, and parameter tuning on the fair data representation.
\item[3] The fair data representation approach has more applicable models than the post-processing approach. See Remark \ref{r:Interpretation of L2-objective Models} below for a detailed explanation of two different interpretations of $L^2$-objective models.
\end{itemize}

In the following remark, we explain the different interpretations of $L^2$-objective models in the post-processing and pre-processing approaches.

\begin{rema}[Interpretation of $L^2$-objective models]\label{r:Interpretation of L2-objective Models}
For the post-processing approach, it follows from \eqref{eq:post-processing optimal learning characterization objective} and \eqref{eq:post-processing objective in practice} that the barycenter characterization works only if the supervised learning model comes with an objective function in explicit $L^2$-form.
For the proposed pre-processing approach, the applicable $L^2$-objective models include all the models that aim to estimate the conditional expectation. In particular, it follows from \eqref{eq:decomposition of total objective} and \eqref{eq:optimal fair data representation objective} that the proposed fair data representation works for any supervised learning model that aims to estimate conditional expectation or conditional probability, even though some of them do not come with an explicit objective function in $L^2$-form. For example, all classification models share the goal of estimating the conditional probability of $\{Y = 1\}$ given an observation of $\{X = x\}$, which is $\mathbb{E}(\mathbbm{1}_{Y = 1} | X = x)$. Therefore, the resulting synthetic data can be used for any classification model, even models such as logistic regression and random forest that do not have $L^2$-based objective functions.
\end{rema} 

\subsection{Setting and Notation}

In the rest of the work, $\mathcal{L}(X) = \mathbb{P} \circ X^{-1}: \mathcal{B}_{\mathcal{X}} \rightarrow [0,1]$ denotes the distribution or law of $X$, which is a function that assigns each event in the Borel sigma-algebra, $\mathcal{B}_{\mathcal{X}}$, a probability.  Let $\lambda := \mathcal{L}(Z)$ denote the law of the sensitive random variable to simplify notation. To remove sensitive information $Z$, the method we propose is to find a set of maps $T_x := \{T_x(\cdot,z)\}_z$ such that $T_x(\cdot,z): \mathcal{X} \rightarrow \mathcal{X}$ pushes the conditional (on $\{Z = z\}$) distribution (see the definition of conditional distribution $\mathcal{L}(X_z)$ below) forward to a common probability measure $\mathcal{L}(\tilde{X})$ for $\lambda$-a.e.\ $z \in \mathcal{Z}$. Also, when restricting $T$ to be a linear map or a matrix, we use $T \succ 0$ to denote $T$ is positive definite, and $||T||_F$ to denote its Frobenius norm.

Given a measurable map $T: \mathcal{X} \rightarrow \mathcal{X}$ and a probability measure $\mu \in \mathcal{P}(\mathcal{X})$, $T_{\sharp} \mu$ denotes the push-forward probability measure that is defined as the following: for any event, $A$, in the Borel sigma-algebra, $\mathcal{B}_{\mathcal{X}}$,  $T_{\sharp} \mu(A) := \mu(T^{-1}(A))$. In the rest of the paper, we often say $T$ pushes $\mu$ forward to $T_{\sharp}\mu$.

The conditional distributions $\{\mathcal{L}(X_z)\}_z$ are defined uniquely $\lambda$-a.e.\ by the disintegration theorem \cite[Box 2.2]{santambrogio2015optimal}. Hence, $z \rightarrow \mathcal{L}(X_z)$ is Borel measurable and, for all Borel measurable sets $E \in \mathcal{B}_{\mathcal{X}}$, $\mathbb{P}(E) = \int_{\mathcal{X}}\mathbb{P}(X_z^{-1}(E)) d \lambda(z)$. The application of the disintegration theorem aims to allow $\mathcal{Z}$ to be uncountably infinite, such as the real line or the real vector space. In the practical case of a finite data set, when the data set $(X,Z)$ is $\{(x_i,z_i)\}_{i \in [N]}$, for each $z \in \mathcal{Z}$, the empirical conditional random variable (with uniform distribution) is defined as follows: $$X_z := \{x_i: (x_i,z_i) \in (X,Z), z_i = z\}.$$ Therefore, on the product data space $\mathcal{X} \times \mathcal{Z}$ with a joint distribution, the law of the random variable or vector $X_z$ is the conditional distribution on $\{Z = z\}$.

The present work often assumes the conditionals $\{\mathcal{L}(X_z)\}_{z \in \mathcal{Z}} \subset \mathcal{P}_{2,ac}(\mathcal{X})$. Here, $\mathcal{P}_{2,ac}(\mathcal{X})$ denotes the set of probability measures on $\mathcal{X}$ that have finite second moments and are absolutely continuous with respect to the Lebesgue measure.  The finite second moment assumption guarantees the Wasserstein distance to be well-defined without being infinite. The absolute continuity assumption guarantees the existence of their Wasserstein barycenter (See Definition \ref{d:barycenter definition}) and the respective (almost surely invertible) optimal transport maps that map them to the barycenter. The present work denotes the barycenter by $\overline{\mathcal{L}(X_z)}$ or $\overline{\mathcal{L}(X)}$ interchangeably, and denotes the optimal transport map that pushes $\mathcal{L}(X_z)$ to $\overline{\mathcal{L}(X)}$ by $T_z$ or $T(\cdot, z)$.

To simplify notation and proof, we define $\bar{X}$ to be the random variable that satisfies the following: for $\lambda$-a.e. $z \in \mathcal{Z}$,
\begin{equation}
\bar{X}_z = T_z(X_z).
\end{equation}
In other words, the couple $(X_z,\bar{X}_z)$ is a coupling of $(\mathcal{L}(X_z), \overline{\mathcal{L}(X)})$ and satisfies:
\begin{equation}
||X_z - \bar{X}_z||_2^2 = \mathcal{W}^2_2(\mathcal{L}(X_z), \overline{\mathcal{L}(X)})
\end{equation}
for $\lambda$-a.e. $z \in \mathcal{Z}$. We refer interested readers to ~\cite{villani2021topics, villani2009optimal} for more details on the assumption of $\mathcal{P}_{2,ac}(\mathcal{X})$ and the coupling of measures. In the rest of the paper, we call $\bar{X}$ the Wasserstein barycenter of $\{X_z\}_z$.

In solving the post-processing characterization, with the assumption of $\mathbb{E}(Y|X,Z)$, one first finds the Wasserstein barycenter of $\{\mathcal{L}(\mathbb{E}(Y|X,Z)_z)\}_z$, denoted by $\overline{\mathcal{L}(\mathbb{E}(Y|X,Z)_z))}$. Here, $\mathbb{E}(Y|X,Z)_z$ denotes the conditional of $\mathbb{E}(Y|X,Z)$ on $\{Z = z\}$ for $\lambda$-a.e. $z \in \mathcal{Z}$. Then one applies the optimal transport map $T(\cdot,z): \mathcal{Y} \rightarrow \mathcal{Y}$ which pushes $\mathbb{E}(Y|X,Z)_z$ forward to $\overline{\mathbb{E}(Y|X,Z)}_z$ for $\lambda$-a.e.\ $z \in \mathcal{Z}$.

In solving the pre-processing characterization, one has two different optimal transport maps to deform $X$ and $Y$. For the dependent variable, we define $T_y = \{T_y(\cdot,z)\}_z$, $\mathcal{L}(Y_z)$, and $\mathcal{L}(\tilde{Y})$ analogously, but require merely the agreement of $\mathcal{L}(\mathbb{E}(\tilde{Y}|\tilde{X},Z)_z)$ for $\lambda$-a.e.\  $z \in \mathcal{Z}$. The $\lambda$-a.e. agreement of $\mathcal{L}(\mathbb{E}(\tilde{Y}|\tilde{X},Z)_z)$ means that the laws of the random variables or vectors $\mathbb{E}(\tilde{Y}|\tilde{X},Z)_z$ are equal, except for some $z$ on a $\lambda$-null set on $\mathcal{Z}$. In other words, on the Borel measurable space $(\mathcal{Y},\mathcal{B}_{\mathcal{Y}})$, for any set $B$ in the Borel sigma-algebra $\mathcal{B}_{\mathcal{Y}}$, we have $\mathbb{P} \circ [\mathbb{E}(\tilde{Y}|\tilde{X},Z)_{z_1}]^{-1}(B) = \mathbb{P} \circ [\mathbb{E}(\tilde{Y}|\tilde{X},Z)_{z_2}]^{-1}(B)$ for all $z_1, z_2 \in \mathcal {Z}$, except on a set $N \subset \mathcal{Z}$ such that $\lambda(N) = 0$.

Therefore, by generating and applying $(T_x, T_y)$ to the data, we achieve $\mathbb{E}(\tilde{Y}|\tilde{X},Z) \perp Z$, i.e.\ {\em statistical parity}, due to the enforced $\lambda$-a.e.  agreement of $\mathcal{L}(\mathbb{E}(\tilde{Y}|\tilde{X},Z)_z))$. Combining the application of deformation maps and \eqref{eq:optimal fair data representation problem}, we obtain the fair data representation optimization problem

%But notice that the post-processing constraint $f_{\tilde{Y}}(\tilde{X}) \perp Z$ depends heavily on the choice of admissible functions $\mathcal{F}$ and hence is insufficient for our pre-processing purpose. In order to guarantee fairness for any deterministic function $f$ in arbitrary admissible family $\mathcal{F}$, we also require $\tilde{X} \perp Z$. (See Remark \ref{post and pre} for more detailed explanation of the reason underlying the additional constraint.)

%The deformation of data for fairness leads to a necessary increase in estimation error due to some loss of information. In order to minimize the resulting increase in estimation error, which is quantified by an increase in $L^2$ norm in the present work, one needs to choose the pair of measurable maps $(T_x,T_y)$ that solves

%\begin{equation} \label{map_objective}
%\inf_{(T_x,T_y)} \{ ||Y - \mathbb{E}(T_y(Y,Z)|T_x(X,Z))||^2_2 : T_x(X,Z),  \mathbb{E}(T_y(Y,Z)|T_x(X,Z)) \perp Z\}
%\end{equation}

\begin{equation}\label{bary_objective}
\inf_{(\tilde{X}, \tilde{Y}) \in \mathcal{D}} \{ ||Y - \mathbb{E}(\tilde{Y}|\tilde{X})||^2_2 : \tilde{X}, \mathbb{E}(\tilde{Y}|\tilde{X},Z) \perp Z\}
\end{equation}
with the admissible set $\mathcal{D}$ is defined as
\begin{equation}
\mathcal{D} := \{ (\tilde{X},\tilde{Y}):  \tilde{X} = T_x(X,Z) , \tilde{Y} = T_y(Y,Z) \},
\end{equation}
Here, $T_x(\cdot,z): \mathcal{X} \rightarrow \mathcal{X}$ and $T_y(\cdot,z): \mathcal{Y} \rightarrow \mathcal{Y}$ are Borel measurable maps. We denote the set of admissible $\tilde{X}$ and $\tilde{Y}$ by $\mathcal{D}|_{\mathcal{X}}$ and $\mathcal{D}|_{\mathcal{Y}}$, respectively. The reason underlying the definition of $\mathcal{D}$ is that the fair data should still has its foundation from the real data, albeit suitably ``deformed''.

\subsection{Paper Organization}

The rest of the paper is organized as follows: Section~\ref{S:Prelim} reviews the tools in optimal transport that are needed to derive results in the present work: Wasserstein space, Wasserstein barycenter, and optimal affine transport within a location-scale family. Section~\ref{s:Post-Processing} first generalizes the current barycenter characterization of optimal regression to optimal $L^2$-objective supervised learning, then defines pseudo-barycenter, and proves pseudo-barycenter is the optimal affine estimation of the true barycenter.  Section~\ref{s:Pareto Frontier} is concerned with both the theoretical characterization and an explicit formula of the Pareto frontier on the Wasserstein space. Section~\ref{s:Optimal Fair Data Representation} studies the exact solution to the optimal data representation and the optimal affine estimation of the exact solution. %\edit{Section~\ref{s:Equalized Odds} studies the exact solution to the optimal equalized odds data representation and shows how the previously developed methods can be applied to estimate the exact solution.} 
Section~\ref{s:Algorithm} proposes an algorithm based on the theoretical results in the previous sections. Section~\ref{s:Numerics} provides an extensive numerical study regarding the application of the pseudo-barycenter and the optimal affine maps to (1) the estimation of optimal fair learning outcome compared to the known fair machine learning techniques on different learning models; and (2) Pareto frontier estimation for different disparity definitions.  % (3) K-means for diverse or fair data allocation.

\section{Preliminaries on Optimal Transport} \label{S:Prelim}

In this section, we review the theoretical results on optimal transport and the Wasserstein barycenter that are important for the development of the main theoretical results on efficient algorithm design, Wasserstein geodesic characterization of the Pareto frontier, and the pre-processing approach resulting in the optimal fair data representation. For our purposes, we focus on $\mathbb{R}^d$. We refer readers who are interested in more generalized versions, e.g.\ on compact Riemannian manifolds, to for example~\cite{kim2017wasserstein}.

\subsection{General Distribution Case} \label{s:prelim general distribution case}

Given $\mu, \nu \in \mathcal{P}(\mathbb{R}^d)$,  which is the set of all probability measures on $\mathbb{R}^d$, Monge asked for an optimal transportation map $T_{\mu \nu}: \mathbb{R}^d \rightarrow \mathbb{R}^d$ that solves

\begin{equation}\label{Monge}
\inf_{ T_{\sharp} \mu = \nu } \Big\{\int_{\mathbb{R}^d} ||x - T(x)||^2 d \mu\Big\}
\end{equation}
Here, $||\cdot||$ denotes the Euclidean norm on $\mathbb{R}^d$. The problem remained open until Brenier showed that Monge's problem coincides with Kantorovich's relaxed version:

\begin{equation}\label{Kantorovich}
\inf_{ \gamma \in \prod (\mu,\nu) } \Big\{\int_{\mathbb{R}^d \times \mathbb{R}^d} ||x_1 - x_2||^2 d \gamma(x_1,x_2)\Big\}
\end{equation}
and admits a unique solution provided $\mu \in \mathcal{P}_{2,ac}(\mathbb{R}^d)$. Here, $\mathcal{P}_{2,ac}(\mathbb{R}^d)$ denotes the space of probability measures on $\mathbb{R}^d$ that have finite first two moments and are absolutely continuous w.r.t. (with respect to) the Lebesgue measure. That is, the optimal solution to $\eqref{Kantorovich}$ has the form: $\gamma = (Id, T_{\mu \nu})_{\sharp} \mu$, where $T_{\mu \nu}$ solves $\eqref{Monge}$. Here, $\prod (\mu,\nu)$ denotes all the probability measures on $(\mathbb{R}^{2d}, \mathcal{B}(\mathbb{R}^d) \otimes \mathcal{B}(\mathbb{R}^d))$ such that the marginals are $\mu$ and $\nu$. The relaxed problem is easy to solve due to the weak* compactness of $\prod (\mu,\nu)$. We refer interested readers to~\cite{villani2021topics, villani2009optimal} for more detailed existence and uniqueness results. 

\begin{rema}
The uniqueness is in the weak sense for $\gamma$ and $\mu$-a.e.\ for $T_{\mu \nu}$.
\end{rema}

Kantorovich's problem provides a certain kind of ``distance" on $\mathcal{P}(\mathbb{R}^d)$ except for the possibility of being infinite.

\begin{defi}[Wasserstein distance\footnote{Throughout this paper we work with the Wasserstein-2 distance, and thus simply call it the Wasserstein distance.}]
Given $\mu, \nu \in \mathcal{P}(\mathbb{R}^d)$,
\begin{equation}
\mathcal{W}_2(\mu,\nu) := \left(\inf_{ \gamma \in \prod (\mu,\nu) } \Big\{\int_{\mathbb{R}^d \times \mathbb{R}^d} ||x_1 - x_2||^2 d \gamma(x_1,x_2)\Big\}.\right)^{\frac{1}{2}}
\end{equation}

\end{defi}

It is not hard to verify that the Wasserstein distance defined above satisfies the axioms of a metric except for finiteness of $\mathcal{W}_2(\mu,\nu)$ for arbitrary $\mu,\nu \in \mathcal{P}(\mathbb{R}^d)$. In order to guarantee finiteness, one needs to put more restrictions on the set of all probability measures:

\begin{defi}[Wasserstein space]
Define $\mathcal{W}_2$ as above and 
\begin{equation}
\mathcal{P}_2(\mathbb{R}^d):= \Big\{\mu \in \mathcal{P}(\mathbb{R}^d): \int_{\mathbb{R}^d} ||x||^2 d\mu < \infty\Big\}.
\end{equation}
The couple $(\mathcal{P}_2(\mathbb{R}^d),\mathcal{W}_2)$ is called Wasserstein space.

\end{defi}

The Wasserstein space has gained increasing popularity in image processing, economics \cite{ekeland2010existence, carlier2010matching}, and machine learning in recent years due to its useful properties such as polishness (of the space) and robustness (w.r.t.\ perturbation on the marginal probability measures and hence on sampling).

Since the Wasserstein space is a metric space, the Fr\'echet mean on the space is well-defined and it is called the Wasserstein barycenter in the optimal transport literature.

\begin{defi}[Wassserstein barycenter \cite{agueh2011barycenters}]\label{d:barycenter definition}
Given $\{\mu_z\}_{z \in \mathcal{Z}} \subset (\mathcal{P}_2(\mathbb{R}^d),\mathcal{W}_2)$ for some index set $\mathcal{Z}$, the barycenter of $\{\mu_z\}_z$ is the Fr\'echet mean of the set on $(\mathcal{P}_2(\mathbb{R}^d),\mathcal{W}_2)$. That is, $\bar{\mu}$ is the solution to
\begin{equation}\label{barycenter}
\inf_{\mu \in \mathcal{P}_2(\mathbb{R}^d)} \Big\{\int_{\mathcal{Z}} \mathcal{W}_2^2(\mu_z,\mu) d\lambda(z)\Big\},
\end{equation}
where $\bar{\mu}$ denotes the Fr\'echet mean or barycenter.
\end{defi}
Here, for our purpose, we focus on the case where the index set $\mathcal{Z} \in \{[k], \mathbb{N}, [0,1], \mathbb{R}^n\}$. 

Next, we look at optimal transport and the barycenter problem from the perspective of optimal coupling. The goal is to show that the multi-marginal coupling problem is equivalent to the Wasserstein barycenter problem. The equivalence is an essential tool in proving our result in optimal affine transport, the optimality of the pseudo-barycenter, and the geodesic characterization of the Pareto frontier.

First, notice that Kantorovich's problem is in fact a 2-marginal coupling problem: Let $X_1,X_2$ be the random variable satisfy $\mathcal{L}(X_1) = \mu, \mathcal{L}(X_2) = \nu$, the problem looks for a $\gamma$ with marginals being $\mu, \nu$ that minimizes $\mathbb{E}_{\gamma} ||X_1 - X_2||^2$. It follows naturally by the existence and uniqueness result of the optimal transport map (also known as Brenier's map) \cite{brenier1991polar}, that the Wasserstein distance admits the form in the classic probability language:

\begin{equation}
\mathcal{W}_2(\mu,\nu) = (\mathbb{E}_{\mu} ||X_1 - T(X_1)||^2)^{\frac{1}{2}},
\end{equation}
where $T$ is the optimal transport map that pushes $\mu = \mathcal{L}(X_1)$ forward to $\nu = \mathcal{L}(X_2)$.

More recent work in mathematics~\cite{kim2017wasserstein, pass2013optimal} and economics~\cite{carlier2010matching, ekeland2010existence} has generalized the Kantorovich problem to the multi-marginal coupling problem: 
\begin{equation}\label{multi-marginal coupling}
\inf_{\gamma \in \prod(\{\mu_z\}_{z \in \mathcal{Z}})} \big\{\mathbb{E}_{\gamma}(\int_{\mathcal{Z}^2} ||X_{z_1} - X_{z_2}||^2 d\lambda(z_1)d\lambda(z_2) )\big\},
\end{equation}
where $\prod(\{\mu_z\}_{z \in \mathcal{Z}})$  denotes all the Borel probability measures on $(\mathbb{R}^d)^{|\mathcal{Z}|}$ with marginals being $\mu_z = \mathcal{L}(X_z) \in \mathcal{P}(\mathbb{R}^d)$ $\lambda$-a.e..  Hence, one can consider $\lambda \in \mathcal{P}(\mathcal{P}(\mathbb{R}^d))$.  It can be shown that the above is equivalent to the following:

\begin{equation}\label{maxcov}
\sup_{\gamma \in \prod(\{\mu_z\}_{z \in \mathcal{Z}})} \big\{\mathbb{E}_{\gamma}(||\int_{\mathcal{Z}} X_{z} d\lambda(z)||^2) \big\}
\end{equation}

\begin{rema}[Justification for the name of marginals]\label{r:name of marginals}
Since $\{X_z\}_z$ are the marginals for the admissible couplings in \eqref{multi-marginal coupling}, with the equivalence between the multi-marginal coupling and Wasserstein barycenter (see Remark \ref{equivalence between couple and bary} below) in mind, we often call $\{X_z\}_z$ and $\{\mathcal{L}(X_z)\}_z$ the sensitive marginals, even though they are also the conditional random variables and distributions constructed by disintegration.
\end{rema}

Intuitively, $\eqref{maxcov}$ tends to find a family of random variables parametrized by $z$ with fixed marginals $\mu_z$ such that the variance of the matched (by $\gamma$) group average is maximized. For readers who are more familiar with stochastic processes, consider $z = t$ as a time variable, then $X_t$ is a stochastic process with fixed time marginals, and $\eqref{maxcov}$ tends to find a way ($\gamma$) to group the fixed marginals into trajectories so that the variance of the trajectory-wise (sample path) average is maximized. (Hence, the expected variance within a randomly chosen sample path is minimized.)

As shown in \cite{agueh2011barycenters, pass2013optimal}, the above multi-marginal problem is equivalent to the barycenter problem:

\begin{rema}[Equivalence between multi-marginal coupling and barycenter] \label{equivalence between couple and bary}
Assume $\{\mu_z\}_z$ are absolutely continuous w.r.t.\ the Lebesgue measure and let $\gamma^*$ and $\bar{\mu}$ be the solution to \eqref{maxcov} and \eqref{barycenter}, respectively. It follows that $\bar{\mu} = \gamma^* \circ T^{-1}$ where $T(\{x_z\}_z) := \int_{\mathcal{Z}} x_z d\lambda(z)$. 
\end{rema}

The importance of this equivalence is twofold:

\begin{itemize}
\item[1] It is the key to proving the non-degenerate Gaussianity of the Wasserstein barycenter of non-degenerate Gaussian marginal distributions;
\item[2] It provides technical support for the interpretation (Section \ref{s:Post-processing and Pre-processing Approach} point 4) of how the Wasserstein barycenter solves data-related fairness issues on a point-wise scale.
\end{itemize}
Therefore, we generalize the equivalence to the case where $\mathcal{Z}$ is a Polish space, which is a metric space that is separable and complete. In particular, $[k]^d, [0,l]^d, \mathbb{N}^d, \mathbb{R}^d$ mentioned above are all examples of Polish spaces. This generalization is important for our purpose as it provides a theoretical foundation for removing $Z$ in the form of random vectors.

Now, the following result provides the existence and uniqueness result of the barycenter problem that is suitable for our purpose.

\begin{thm} [Existence and uniqueness of barycenter \cite{le2017existence}(Theorem 2 and Proposition 6) ] \label{Existence and Uniqueness of Barycenter}

Assume that $\mathcal{Z}$ is a Polish space and that $\lambda := \mathbb{P} \circ Z^{-1}$ satisfies $\int_{\mathcal{Z}} \mathcal{W}_2^2(\mu_z,\nu)d\lambda(z) < \infty$ for some $\nu \in \mathcal{P}_2(\mathcal{X})$ (hence, for all $\nu  \in \mathcal{P}_2(\mathcal{X})$). Then the following properties hold:

\begin{itemize}
\item[1] There exists a barycenter of $\{\mu_z\}_{z \in \mathcal{Z}}$ w.r.t. $\lambda$.
\item[2] If, in addition, $\lambda(\{z: \mu_z \in \mathcal{P}_{ac}(\mathcal{X}) \}) > 0$, then the barycenter is unique.
\end{itemize}

\end{thm}

\begin{rema}[Applicability of assumptions in Theorem \ref{Existence and Uniqueness of Barycenter}]
The assumption that $\int_{\mathcal{Z}} \mathcal{W}_2^2(\mu_z,\nu)d\lambda(z) < \infty$ in the above result is satisfied in our application to the optimal fair learning outcome or data representation: When generating the optimal transport maps $\{T_z\}_z$, the training set has a finite number of data and hence finite different values of $z$ in the discrete case or after discretization in the continuous case. Therefore, since $\{\mu_z\}_z \subset \mathcal{P}_2(\mathcal{X})$, pick a value $z_0$ that is in the training set, we have that $\mathcal{W}_2^2(\mu_{z},\mu_{z_0})$ are essentially (w.r.t. $\lambda$) uniformly bounded. That implies $\int_{\mathcal{Z}} \mathcal{W}_2^2(\mu_z,\mu_{z_0})d\lambda(z) < \infty$.
\end{rema}

Now, we have the theoretical results that are needed to prove the main results, except for the McCann interpolation, which will be introduced in Section 4. The next step is to develop a computationally efficient method to compute (an estimation of) the Wasserstein barycenter, (the McCann interpolation of) optimal transport maps, and thereby the optimal fair model and Pareto frontier. More specifically, we focus on positive definite affine optimal transport maps.

\subsection{Rigid Translation}

Before deriving our main result on optimal positive definite affine maps, we first study the case where admissible maps are restricted to the set of rigid translations. The following property of rigid translations makes our results on the optimal affine maps simpler: we can assume, without loss of generality, that the first moments of the marginal measures are zero: $m_{X_z} := \mathbb{E}(X_z) = 0$ and $m_{Y_z} := \mathbb{E}(Y_z) = 0$.

\begin{lem} \label{l:Rigid translation property}

Let $\mu, \nu \in \mathcal{P}_2$, $m_{\mu} := \int x d\mu(x)$, and $m_{\nu} := \int x d\nu(x)$. Also, let $\mu', \nu'$ be the centered versions of $\mu,\nu$, respectively. It follows that 

\begin{equation}
\mathcal{W}_2^2(\mu,\nu) = \mathcal{W}_2^2(\mu',\nu') + ||m_{\mu} - m_{\nu}||^2.
\end{equation}

\end{lem}

\begin{proof}
See Appendix \ref{A:Section 2 Appendix}.
\end{proof}

Notice that the above result allows us to assume measures to have vanishing first moments when deriving the optimal transport maps. Indeed, if $T_{\mu' \nu'}$ is the Brenier's map between $\mu'$ and $\nu'$, then $T_{\mu \nu} : =  T_{+m_{\nu}} \circ T_{\mu' \nu'} \circ T_{-m_{\mu}}$ is the optimal transport map between $\mu$ and $\nu$. Here, $T_{+m_{\nu}}(x) := x + m_{\nu}$ and $T_{-m_{\mu}}$ are defined analogously.

In the rest of Section~\ref{S:Prelim}, we assume without loss of generality that the first moments of the measures are all equal to zero.

\subsection{Location-Scale Case and Optimal Affine Transport}

A sufficient condition for Brenier's maps to be positive definite affine is to require a certain ``similarity" between the marginal data distributions. One natural choice is to assume $\{Y_z\}_z$ and $\{X_z\}_z$ to be non-degenerate Gaussian vector $\lambda$-a.e.. As shown in \cite{alvarez2016fixed}, the assumptions of Gaussian vector can easily be generalized to a {\em location-scale family}. In the definition below, $\mathcal{S}^d_{++}$ denotes the set of all $d \times d$ positive definite matrices.

The generalization from Gaussian to location-scale families is important for the main result in the next section, where we consider computationally efficient solutions to a relaxation of the Wasserstein barycenter problem in the case of general marginal distributions.

\begin{defi}[Location-Scale Family]
For any $\mathcal{L}(X_0) \in \mathcal{P}(\mathbb{R}^d)$, define

\begin{equation} \label{Location-scale family}
\mathcal{F}(\mathcal{L}(X_0)) := \big\{\mathcal{L}(AX_0 + m): A \in \mathcal{S}^d_{++}, m \in \mathbb{R}^d \big\}.
\end{equation}
The set $\mathcal{F}(\mathcal{L}(X_0))$ is called a location-scale family characterized by $\mathcal{L}(X_0)$.
\end{defi}

In other words, under the assumption of vanishing first moments, the random variables that share laws in the same location-scale family can be transformed into each other by a positive definite linear transformation.

In~\cite{alvarez2016fixed} it is shown that Brenier's map between two probability measures, each having a vanishing first moment, within the same location-scale family is linear and has a closed form. 

\begin{lem}[Optimal affine map] \label{l:Optimal Affine Map}
If $\mu, \nu \in \mathcal{F}(\mathcal{L}(X_0))$ for some $X_0$ such that $m_{\mu} = m_{\nu} = 0$, then the Brenier's map that pushes $\mu$ forward to $\nu$ is given by:

\begin{equation} \label{closed-form map}
T_{\mu \nu} = \Sigma_{\mu}^{-\frac{1}{2}} (\Sigma_{\mu}^{\frac{1}{2}} \Sigma_{\nu} \Sigma_{\mu}^{\frac{1}{2}} )^{\frac{1}{2}} \Sigma_{\mu}^{-\frac{1}{2}} 
\end{equation}
where $\Sigma_{\mu} := \int xx^T d\mu$ and $\Sigma_{\nu} := \int xx^T d\nu$.
\end{lem}

\begin{proof}
See, for example, Theorem 2.3 in \cite{alvarez2016fixed}.
\end{proof}

\begin{rema}
The optimal affine map is also the midpoint of the geodesic path joining $\Sigma_{\mu}^{-1}$ and $\Sigma_{\nu}$ on the manifold of positive definite matrices. We refer interested readers to, for example, Chapter 6.1 in~\cite{bhatia2009positive} for more details.
\end{rema}

Now, back to the barycenter problem. It follows from Lemma \ref{l:Optimal Affine Map} that, if one assumes that all the marginals belong to the same location-scale family, then the barycenter also belongs to the family and a nearly closed-form solution to the barycenter is available.

\begin{lem}[Barycenter in the location-scale case] \label{l:Location-scale Barycenter}
Assume $\{\mu_z\}_z$ belong to the same location-scale family $\mathcal{F}(P_0)$ and satisfy $m_{\mu_z} = 0, \Sigma_{\mu_z} \succ 0, \lambda-a.e.$, then there exists a unique solution, denoted by $\bar{\mu}$, to $\eqref{barycenter}$. Moreover, $\bar{\mu}$ also belongs to $\mathcal{F}(P_0)$ and is characterized by $m_{\bar{\mu}} = 0$ and $\Sigma_{\bar{\mu}} = \Sigma$ where $\Sigma$ is the unique solution to the following equation:

\begin{equation}\label{barycenter equation}
\int_{\mathcal{Z}} (\Sigma^{\frac{1}{2}} \Sigma_{\mu_z} \Sigma^{\frac{1}{2}})^{\frac{1}{2}} d\lambda(z) = \Sigma,
\end{equation}
where $\Sigma_{\mu_z}$ is the second moment of $\mu_z, \forall z \in \mathcal{Z}$.

\end{lem}

\begin{proof}
See Appendix \ref{A:Section 2 Appendix}.
\end{proof}

In the case where $m_{\mu_z} \neq 0$, it follows from Lemma $\ref{l:Rigid translation property}$ that

\begin{align*}
\int_{\mathcal{Z}} \mathcal{W}_2^2(\mu_z,\mu) d\lambda(z) & = \int_{\mathcal{Z}} \mathcal{W}_2^2(\mu'_z,\mu') d\lambda(z) + \int_{\mathcal{Z}} ||m_{\mu_z} - m_{\mu}||^2 d\lambda(z) 
\end{align*}
where $\mu'$ denotes the centered version of $\mu$. By Lemma $\ref{l:Location-scale Barycenter}$, we know the first term on the right is minimized at $\bar{\mu}' \sim \mathcal{N}(0,\Sigma_{\bar{\mu}})$. Also, the second term on the right is minimized at the Fr\'{e}chet mean with Euclidean metric, which is equal to the expectation. That is, $m_{\bar{\mu}} = \int_{\mathcal{Z}} m_{\mu_z} d\lambda(z)$. As a result, the optimal transport map is

\begin{equation} \label{eq:rigid translation composition}
T_{\mu_z \bar{\mu}} = T_{+m_{\bar{\mu}}} \circ T_{\mu_z' \bar{\mu}'} \circ T_{-m_{\mu_z}}
\end{equation}

\begin{rema}[Solution to \eqref{barycenter equation}] \label{r:solution to baryceneter equation}
The non-linear matrix equation \eqref{barycenter equation} has a unique solution that can be approached via the following iterative process:

\begin{equation} \label{eq:iterative method}
\int_{\mathcal{Z}} (\Sigma_i^{\frac{1}{2}} \Sigma_{\mu_z} \Sigma_i^{\frac{1}{2}})^{\frac{1}{2}} d\lambda(z) \rightarrow \Sigma_{i+1}.
\end{equation}
We refer interested readers to \cite{alvarez2016fixed} for more details on the fixed point approach to the Wasserstein barycenter. The present work only applies this fact in the algorithm design in Section~\ref{s:Algorithm}.
\end{rema}

\section{Wasserstein Barycenter Characterization of the Optimal Fair Learning Outcome} \label{s:Post-Processing}

Optimal transport has been considered an adversarial or constrained optimization problem in its application to machine learning. In particular, some of the most popular unsupervised learning methods, such as K-means and PCA, are specific examples of the Wasserstein barycenter problems when putting restrictions on the admissible transport maps and relaxation on the weak equivalence requirement of the push-forwards w.r.t.\ test functions. See, for example, \cite{tabak2018explanation} for more details. But we apply optimal transport in an opposite direction so that the independence or imperceptibility of the sensitive variable $Z$ becomes theoretically provable.

In this section, the primary goal is to develop the optimal affine map and pseudo-barycenter as tools to solve the challenge of the high computational cost of Wasserstein barycenter and optimal transport maps in high-dimensional data space. More specifically, we restrict the admissible transport maps to be merely affine maps while relaxing the fairness constraint to a sufficient and necessary level. The importance of efficiency in computing the barycenter and optimal transport maps will soon be clear in Section \ref{s:Pareto Frontier} when we compute the Pareto frontier along the Wasserstein geodesic path. Furthermore, the importance of affinity of transport maps will also be soon clear in Section \ref{s:Optimal Fair Data Representation} when solving the optimal fair data representation problem \eqref{eq:optimal fair data representation problem}.

The organization of the current section is as follows: we first generalize the Wasserstein barycenter characterization of the optimal regression to all $L^2$-objective supervised learning models, then apply the optimal affine maps to estimate high-dimension optimal learning outcome. Now, we show that the (unique) solution to Problem \ref{prob:Optimal Fair $L^2$-objective Learning Outcome} can be characterized as the Wasserstein barycenter of the conditional expectation sensitive marginals. The barycenter characterization of the optimal fair regression is first proved in \cite{chzhen2020fair, gouic2020projection}.

\subsection{Wasserstein Barycenter Characterization} \label{s:post bary}

We start with a characterization of the optimal learning outcome of the $L^2$-objective supervised learning task.  Let $\mathbb{E}(Y|X,Z)_z$ be the sensitive marginals of $(\mathbb{E}(Y|X,Z),Z)$ (or, equivalently, the sensitive conditionals of $\mathbb{E}(Y|X,Z)$ on $\{Z = z\}$ by Remark \ref{r:name of marginals}) for $\lambda$-a.e. $z \in \mathcal{Z}$, $\mathcal{L}(\mathbb{E}(Y|X,Z)_z) := \mu_z$, and $\bar{\mu}$ denote the Wasserstein barycenter of $\{\mu_z\}_{z \in \mathcal{Z}}$. Also, let $T(\cdot,z)$ denote the optimal transport map from $\mu_z$ to $\bar{\mu}$.

\begin{lem}[Optimal fair $L^2$-objective supervised learning characterization]\label{l:Optimal Fair $L^2$-Objective Supervised Learning Characterization}
Assume that the conditional expectation marginals $\{\mu_z\}_{z \in \mathcal{Z}} \subset \mathcal{P}_{2,ac}(\mathcal{Y})$, then
\begin{equation}
\overline{\mathbb{E}(Y|X,Z)} = T(\mathbb{E}(Y|X,Z),Z)  := \{T(\mathbb{E}(Y|X,Z)_z,z)\}_{z \in \mathcal{Z}} 
\end{equation}
is the unique solution to Problem \ref{prob:Optimal Fair $L^2$-objective Learning Outcome}. Furthermore, we have
\begin{align*}
||Y - T(\mathbb{E}(Y|X,Z),Z)||_2^2 & =  \inf_{f \in L^2(\mathcal{X} \times \mathcal{Z},\mathcal{Y})} \{||Y - f(X,Z)||_2^2 : f(X,Z) \perp Z\}\\
& = ||Y - \mathbb{E}(Y|X,Z)||_2^2 + \int_{\mathcal{Z}} \mathcal{W}_2^2(\mu_z, \bar{\mu}) d\lambda
\end{align*}
\end{lem}

\begin{proof}%[Proof of Lemma \ref{l:Optimal Fair $L^2$-Objective Supervised Learning Characterization}]
First, notice that the fairness constraint $ f(X,Z) \perp Z$ is equivalent to $\mathcal{L}(f(X,Z)_z) = \mu$ $\lambda$-a.e. for some $\mu \in \mathcal{P}(\mathcal{Y})$. Now, we prove the lower bound: let $f \in L^2(\mathcal{X} \times \mathcal{Z},\mathcal{Y})$ satisfies $f(X,Z) \perp Z$, we have
\begin{align*}
||Y - f(X,Z)||_2^2 = & ||Y - \mathbb{E}(Y|X,Z)||_2^2 + ||\mathbb{E}(Y|X,Z) - f(X,Z)||_2^2\\
= & ||Y - \mathbb{E}(Y|X,Z)||_2^2 + \int_{\mathcal{Z}} ||\mathbb{E}(Y|X,Z)_z - f(X,Z)_z||_2^2 d\lambda\\
\geq & ||Y - \mathbb{E}(Y|X,Z)||_2^2 + \int_{\mathcal{Z}} \mathcal{W}_2^2(\mu_z, \mathcal{L}(f(X,Z)_z)) d\lambda\\
\geq & ||Y - \mathbb{E}(Y|X,Z)||_2^2 + \int_{\mathcal{Z}} \mathcal{W}_2^2(\mu_z, \bar{\mu}) d\lambda
\end{align*}
Here, the first line follows from the $L^2$ projection characterization of conditional expectation, the second follows from disintegration, the third from the definition of $\mathcal{W}_2$, and the fourth from the definition of the Wasserstein barycenter and the fairness restriction $ f(X,Z) \perp Z$.

Next, we construct a $f_Y \in L^2(\mathcal{X} \times \mathcal{Z},\mathcal{Y})$ such that the lower bound is obtained. Let $T_z$ denote the optimal transport map such that $(T_z)_{\sharp}\mu_z = \bar{\mu}$ for $\lambda$-a.e. $z \in \mathcal{Z}$. Define $T(\cdot,z) := T_z(\cdot)$ and
\begin{equation}
f_Y(X,Z) := T((\mathbb{E}(Y|X,Z),Z)).
\end{equation}
Here, $\pi := (Id,T_z)_{\sharp}\mu_zd\lambda$ defines $\pi \in \mathcal{P}(\mathcal{Z} \times \mathcal{Y} \times \mathcal{Y})$. Hence, we have $\pi = \pi_{(y,z)}d\lambda_{(\mathbb{E}(Y|X,Z),Z)}$ and $\pi_{(y,z)} = \delta_{T(y,z)}$ $\lambda_{(\mathbb{E}(Y|X,Z),Z)}-a.e.$. Since $(y,z) \rightarrow \pi_{(y,z)} = \delta_{T(y,z)}$ is $\mathcal{Y} \times \mathcal{Z}/ \mathcal{P}(\mathcal{Y})$ measurable, we have $(y,z) \rightarrow T(y,z)$ is $\mathcal{Y} \times \mathcal{Z}/ \mathcal{Y}$ measurable. It follows from $\mathbb{E}(Y|\cdot,\cdot) \otimes Id|_{\mathcal{Z}}(\cdot,\cdot): \mathcal{X} \times \mathcal{Z} \rightarrow \mathcal{Y} \times \mathcal{Z}$ being $\mathcal{X} \times \mathcal{Z}/ \mathcal{Y} \times \mathcal{Z}$ measurable that $f_Y = T \circ (\mathbb{E}(Y|\cdot,\cdot) \otimes Id|_{\mathcal{Z}}(\cdot,\cdot))$ is $\mathcal{X} \times \mathcal{Z}/ \mathcal{Y}$ measurable. Also, $\bar{\mu} \in \mathcal{P}_2(\mathcal{Y}) \implies ||f_Y(X,Z)||_2 < \infty$. This proves $f_Y \in L^2(\mathcal{X} \times \mathcal{Z},\mathcal{Y})$. It remains to show that the lower bound is obtained at $f_Y(X,Z)$. Indeed, by construction, we have
\begin{align*}
||\mathbb{E}(Y|X,Z) - f_Y(X,Z)||_2^2
= & \int_{\mathcal{Z}} ||\mathbb{E}(Y|X,Z)_z - f_Y(X,Z)_z||_2^2 d\lambda\\
= & \int_{\mathcal{Z}} ||\mathbb{E}(Y|X,Z)_z - T(\mathbb{E}(Y|X,Z)_z,z)||_2^2 d\lambda \\
= & \int_{\mathcal{Z}} \mathcal{W}_2^2(\mu_z, (T_z)_{\sharp}\mu_z) d\lambda\\
= & \int_{\mathcal{Z}} \mathcal{W}_2^2(\mu_z, \bar{\mu}) d\lambda.
\end{align*}
It follows from the derivation of the lower bound above that
\begin{equation}
||Y - f_Y(X,Z)||_2^2 = \inf_{f \in L^2(\mathcal{X} \times \mathcal{Z},\mathcal{Y})} \{||Y - f(X,Z)||_2^2 : f(X,Z) \perp Z\}
\end{equation}
Uniqueness follows from the uniqueness of $\bar{\mu}$ and the uniqueness of $T(\cdot,z)$. We are done.
\end{proof}

The above result shows that the minimum $L^2$-loss for statistical parity can be nicely decomposed into two parts: (1) an $L^2(\mathcal{X} \times \mathcal{Z},\mathcal{Y})$ orthogonal projection loss due to the inference capability of $(X,Z)$ w.r.t. $Y$ and (2) an independence projection loss due to the statistical parity constraint. That is, $$\underbrace{\inf_{f} \{||Y - f(X,Z)||_2^2 : f(X,Z) \perp Z\}}_{\text{minimum loss for statistical parity}} =  \underbrace{||Y - \mathbb{E}(Y|X,Z)||_2^2}_{\text{orthogonal projection loss}} + \underbrace{\int_{\mathcal{Z}} \mathcal{W}_2^2(\mu_z, \bar{\mu}) d\lambda}_{\text{independence projection loss}}.$$ Furthermore, to construct the optimal fair $L^2$ learning outcome, one first performs $L^2$ orthogonal projection to obtain the conditional expectation $\mathbb{E}(Y|X,Z)$, then outputs the Wasserstein barycenter of the sensitive marginals of $\mathbb{E}(Y|X,Z)$ as the optimal (with respect to $L^2$-objective) fair (for statistical parity) result. 

Unfortunately, in practice, the characterization suffers from a lack of efficient methods to compute the Wasserstein barycenter and obtain an explicit formula of the optimal transport maps~\cite{altschuler2022wasserstein}. Current methods restrict the sensitive variable $Z$ to be binary mainly because the computation of a multi-marginal barycenter is expensive. Furthermore, notice the current methods restrict the dependent variable $Y$ to be one-dimensional, because the only well-known exact solution to transport maps is the inverse of cumulative function that merely works for one-dimensional variables.

Therefore, to provide methods using the characterization in high-dimensional dependent variable cases, we introduce the optimal affine map and the associated pseudo-barycenter.

\subsection{Optimal Affine Estimation: Pseudo-barycenter} \label{s:affine estimate}

To solve the challenge of deriving an explicit formula for the Wasserstein barycenter and optimal transport maps, we restrict the admissible transport maps to be affine and show that the estimation of the Wasserstein barycenter via optimal affine maps coincides with the true Wasserstein barycenter in the Gaussian case, and that the estimation error is bounded in the case of general distributions. In other words, we consider the choice of positive definite affine maps under two circumstances:
\begin{itemize}
\item[1] We assume the marginals are non-degenerate Gaussian. That is, $\{\mathbb{E}(Y|X,Z)_z\}_z$ are assumed to be non-degenerate Gaussian vectors $\lambda$-a.e..
\item[2] Instead of making assumptions on the data distribution, we relax the independence constraint to the independence between $Z$ and merely the first two moments of $f(X,Z)$.
\end{itemize}

From a theoretical perspective, affine maps allow us to derive (nearly) closed-form solutions under either of the assumptions mentioned above.
Also, affine maps allow us to develop a pre-processing approach by directly applying the obtained maps to the original data before training, even though such maps are constructed to push the post-training marginals toward their barycenter.

From a practical perspective, the advantage is obvious: the computation of affine maps only uses (sample estimation of) the first two moments of the marginal distributions and hence is highly efficient compared to the computation of general Brenier's maps, especially in the case of high-dimension data. 

Before developing the pseudo-barycenter, the following remarks  compare in more detail  the exact barycenter with its  affine approximation.

\begin{rema}[Applying pseudo-barycenter vs  exact barycenter] \label{r:Applying Pseudo-barycenter vs the Exact Barycenter}
The comparison between the pseudo-barycenter method and the exact barycenter is an analog of the comparison between the linear regression model and the exact conditional expectation: When there is no worry about over-fitting, a practitioner who cares more about the strict goal of minimizing $L^2$ error (analog: the strict statistical parity guarantee) should always try to find the exact conditional expectation function (analog: the exact barycenter and the corresponding exact transport maps) by using more complicated models. But the simplicity, robustness, and interpretability of linear regression (analog: pseudo-barycenter and optimal affine maps) are often useful in practice.
\end{rema}

We define the pseudo-barycenter, using merely matrix calculations, as follows:

\begin{defi} \label{d:Post-processing Pseudo-barycenter} The post-processing pseudo-barycenter $\hat{Y}^{\dag}$ is given via
\begin{equation}\label{eq:post-processing pseudo-barycenter}
\hat{Y}^{\dag} := T_{\text{affine}}(\hat{Y},Z),
\end{equation}
where
\begin{equation} \label{eq:post-processing maps}
T_{\text{affine}}(\cdot,z) :=  \Sigma_{\hat{Y}_z}^{-\frac{1}{2}} (\Sigma_{\hat{Y}_z}^{\frac{1}{2}} \Sigma \Sigma_{\hat{Y}_z}^{\frac{1}{2}} )^{\frac{1}{2}} \Sigma_{\hat{Y}_z}^{-\frac{1}{2}},
\end{equation}
and $\Sigma$ is the unique solution to
\begin{equation} \label{eq:post-processing bary cov estimation}
\int_{\mathcal{Z}} (\Sigma^{\frac{1}{2}} \Sigma_{\hat{Y}_z} \Sigma^{\frac{1}{2}})^{\frac{1}{2}} d\lambda(z) = \Sigma.
\end{equation}
\end{defi}

To obtain (an approximation of) the unique solution, we apply the iterative method \eqref{eq:iterative method} in Remark \ref{r:solution to baryceneter equation} when designing our algorithm in Section \ref{s:Algorithm}.

Now, Lemma \ref{l:Optimal Affine Map} shows that under the assumption of Gaussianity of the learning outcome marginals, the optimal transport map is affine and the pseudo-barycenter is indeed the Wasserstein barycenter. Moreover, Lemma \ref{l:Location-scale Barycenter} shows that the barycenter of Gaussian marginals is still Gaussian. Therefore, the optimal maps from the marginals to the barycenter are determined entirely by the first two moments.

\begin{lem}[Post-processing pseudo-barycenter  in the Gaussian case] \label{l:Post-processing Pseudo-barycenter Characterization in Gaussian Case}
Assume $\hat{Y}_z \sim \mathcal{N}(0,\Sigma_z)$ for $\lambda$-a.e. $z \in \mathcal{Z}$, then $\hat{Y}^{\dag}$ is the Wasserstein barycenter of $\{\hat{Y}_z\}_z$.
\end{lem}

It follows from Theorem \ref{l:Post-processing Pseudo-barycenter Characterization in Gaussian Case} that, if $\hat{Y} = \mathbb{E}(Y|X,Z)$, then $Y^{\dag}$ is the solution to the Wasserstein barycenter characterization of the optimal fair learning outcome.

Finally, we show that the pseudo-barycenter is the optimal affine estimation of the Wasserstein barycenter in the case of general marginal distributions. To do so, we need to first put restrictions on the admissible transport maps. However, such a restriction on admissible maps leads to a necessary relaxation of the fairness constraint. To see the necessity, Lemma \ref{l:Optimal Affine Map} shows positive definite affine maps transform distributions within the same location-scale family. Therefore, given marginals $Y_1$ and $Y_2$ from different location-scale families, affine maps are not able to transform them to each other. That implies the non-existence of the barycenter under the original independence restriction. Indeed, if a barycenter of $\{Y_z\}_{z \in \{1,2\}}$ exists under the restriction of positive definite affine maps, then $Y_1$ and $Y_2$ belong to the same location-scale family as their barycenter, which contradicts the assumption of general distributions.  That is, the Wasserstein barycenter characterization does not have a solution when we admit merely affine transport maps in the general marginal distribution case.

On the other hand, notice that the best affine maps can achieve is to map $Y_1$ to a $Y'_2$, which shares the same first two moments with $Y_2$ within the $Y_1$ location-scale family. We call such $Y'_2$ a $Y_1$ location-scale family analog of $Y_2$. Therefore, we propose the following relaxation of the fairness constraint that suffices to guarantee the existence of a solution to the relaxed version of $\eqref{eq:post-processing constrained optimization for optimal learning outcome}$ with merely positive definite affine transport maps:
\begin{equation}
m_{f(X,Z)}, \Sigma_{f(X,Z)} \perp Z
\end{equation}
where $m_{f(X,Z)},$ and $\Sigma_{f(X,Z)},$ denotes respectively the first and second moment of $f(X,Z)$.

\begin{rema}[Fairness guarantee of the relaxation]\label{r:Fairness Guarantee of the Relaxation}
The adversarial task of testing and exploiting probabilistic independence between $f(X,Z)$ and $Z$ is equivalently difficult to enforcing the independence. One common strategy is to explore its equivalence to the independence between all moments of $f(X,Z)$ and $Z$, provided the boundedness of the two random variables. But the verification or enforcement of independence among higher moments is extremely vulnerable to data noise in practice. Thus, instead of enforcing $f(X,Z) \perp Z$, one could relax the constraint to the independence between $Z$ and some of the moments of $f(X,Z)$. In this section, we focus on the first two moments. That is, $m_{f(X,Z)}, \Sigma_{f(X,Z)}$ where $m_{f(X,Z)} := \mathbb{E}(f(X,Z))$ and $\Sigma_{f(X,Z)} := \mathbb{E}((f(X,Z) - \mathbb{E}(f(X,Z)))(f(X,Z) - \mathbb{E}(f(X,Z)))^T)$. It is not hard to notice that the relaxation is already strong enough to result in imperceptibility to any unsupervised learning algorithm that uses merely the mean and covariance of data to extract information, such as K-means and PCA.
\end{rema}

Therefore, the optimal affine estimation of the Wasserstein barycenter characterization is given by:
\begin{prob}[Optimal affine estimation of barycenter problem]\label{prob:optimal affine estimation of barycenter problem}
\begin{equation} \label{eq:optimal affine estimation of barycenter problem}
\inf_{f \in L^2(\mathcal{X} \times \mathcal{Z},\mathcal{Y})} \{||Y - f(X,Z)||_2^2: m_{f(X,Z)}, \Sigma_{f(X,Z)} \perp Z\}.
\end{equation}
\end{prob}

Now, we show that the pseudo-barycenter defined above is indeed the solution to Problem \ref{prob:optimal affine estimation of barycenter problem} and hence the optimal affine estimate of the optimal fair learning outcome. To prove the main result, we need the following result: given any fixed covariance matrix, the optimal positive definite affine maps result in the lowest Wasserstein distance such that the push-forwards all share the same fixed covariance matrix. To simplify notation, let $\mu_z := \mathcal{L}(\mathbb{E}(Y|X,Z)_z)$. Also, let $m_{Y|X_z}$ and $\Sigma_{Y|X_z}$ denote the mean and covariance matrix of $\mathbb{E}(Y|X,Z)_z$ respectively.

\begin{lem}[Projection Lemma] \label{l:Projection Lemma}
Assume $\{\mu_z\}_z \subset \mathcal{P}_{2,ac}(\mathcal{Y})$. If $m_{Y|X_z} = 0, \Sigma_{Y|X_z} \succ 0$ $\lambda$-a.e., for any $\Sigma \succ 0$,
\begin{equation}
\inf_{\hat{Y}: \Sigma_{\hat{Y}_z} = \Sigma} \int_{\mathcal{Z}} \mathcal{W}^2_2(\mu_z,  \mathcal{L}(\hat{Y}_z)) d\lambda(z)
\end{equation}
admits a unique solution, denoted by $\hat{Y}_{\Sigma}$, that satisfies
\begin{equation}
\hat{Y}_{\Sigma,z} := T_{\Sigma}(\hat{Y}_z,z)
\end{equation}
where $T_{\Sigma}(\cdot,z) := \Sigma_{Y|X_z}^{-\frac{1}{2}} (\Sigma_{Y|X_z}^{\frac{1}{2}} \Sigma  \Sigma_{Y|X_z}^{\frac{1}{2}})^{\frac{1}{2}} \Sigma_{Y|X_z}^{-\frac{1}{2}}$.
\end{lem}

\begin{proof}
\begin{align*}
\int_{\mathcal{Z}} \mathcal{W}^2_2(\mu_z,  \mathcal{L}(\hat{Y}_z) d\lambda(z) & =  \int_{\mathcal{Z}} ||\mathbb{E}(Y|X,Z)_z - T_{\Sigma}(\hat{Y}_z,z)||^2_2 d\lambda(z)\\
& = \int_{\mathcal{Z}} \inf_{\nu: \Sigma_{\nu} = \Sigma} \mathcal{W}_2^2(\mu_z,\nu) d\lambda(z)\\
& = \inf_{\nu: \Sigma_{\nu_z} = \Sigma} \int_{\mathcal{Z}} \mathcal{W}_2^2(\mu_z,\nu_z) d\lambda(z),
\end{align*}
where the second equality follows from the characterization of Gelbrich's bound, see for example Proposition 2.4 in \cite{cuesta1996lower}. Now, let $\hat{Y}' \neq \hat{Y}_{\Sigma}$ but also satisfy $\Sigma_{\hat{Y}'} = \Sigma$ $\lambda$-a.e., then we have
\begin{align*}
\int_{\mathcal{Z}} ||\mathbb{E}(Y|X,Z)_z - \hat{Y}_{\Sigma,z}||^2_2 d\lambda(z) & <  \int_{\mathcal{Z}} \mathcal{W}_2^2(\mu_z,\mathcal{L}(\hat{Y}'_z)) d\lambda(z)\\
& \leq \int_{\mathcal{Z}} || \mathbb{E}(Y|X,Z)_z - \hat{Y}'_z||^2_2 d\lambda(z),
\end{align*}
where the first inequality is strict due to the uniqueness of Brenier's maps $T_{\Sigma}(\cdot,z)$ and hence of $T_{\Sigma}(\hat{Y}_z,z)$ $\lambda$-a.e.. The proof is complete.
\end{proof}

\begin{rema}[Intuition of the Projection Lemma] \label{r:Intuition of the Projection Lemma}
Intuitively, for an arbitrary positive definite matrix $\Sigma$, one can consider $T_{\Sigma}(\cdot,z)$ as the projection map (w.r.t. $\mathcal{W}_2$ distance) onto
\begin{equation}
\{\nu \in \mathcal{P}_{2}(\mathcal{Y}): \Sigma_{\nu} = \Sigma\}
\end{equation}
which is the set of centered probability measures with fixed covariance matrix $\Sigma$ in $(\mathcal{P}_2(\mathcal{Y}),\mathcal{W}_2)$. In other words, given a probability measure, the maps $\{T_{\Sigma}(\cdot,z)\}_z$ finds the closest (w.r.t. the Wasserstein distance) point in the set for each of the marginals.
\end{rema}

Finally, we are ready to prove the justification of the pseudo-barycenter in the case of general distributions.

\begin{thm}[Optimal affine estimation of $\mathcal{W}_2$ barycenter: Pseudo-barycenter] \label{th:Optimal Affine Estimation of Barycenter: Pseudo-barycenter}
$\mathbb{E}(Y|X,Z)^{\dagger} := \{T_{\text{affine}}(\mathbb{E}(Y|X,Z)_z,z)\}_z$ is the unique solution to Problem \ref{prob:optimal affine estimation of barycenter problem}:
\begin{equation}
\inf_{f \in L^2(\mathcal{X} \times \mathcal{Z}, \mathcal{Y})}  \{ ||Y - f(X,Z)||^2_2: m_{f(X,Z)}, \Sigma_{f(X,Z)} \perp Z\},
\end{equation}
provided $\{\mu_z\}_z \subset \mathcal{P}_{2,ac}(\mathcal{Y})$.
\end{thm}

\begin{proof}
First, we fix $\Sigma \succ 0$ arbitrary and denote $\hat{Y}_{\Sigma,z} := T_{\Sigma}(\mathbb{E}(Y|X,Z)_z,z)$ for $\lambda$-a.e. $z \in \mathcal{Z}$, we have
\begin{equation}
||Y - T_{\Sigma}(\mathbb{E}(Y|X,Z),Z)||^2_2 - ||Y - \mathbb{E}(Y|X,Z)||^2_2 = \int_{\mathcal{Z}} ||\mathbb{E}(Y|X,Z)_z - \hat{Y}_{\Sigma,z}||^2_2 d\lambda(z)
\end{equation}
and it follows from Lemma $\ref{l:Projection Lemma}$ that
\begin{align*}
\int_{\mathcal{Z}} ||\mathbb{E}(Y|X,Z)_z - \hat{Y}_{\Sigma,z} ||^2_2 d\lambda(z) = & \int_{\mathcal{Z}} \mathcal{W}_2^2(\mu_z,\mathcal{L}(T_{\Sigma}(\mathbb{E}(Y|X,Z)_z,z)) d\lambda(z)\\
= & \min_{\nu: \Sigma_{\nu_z} = \Sigma} \int_{\mathcal{Z}} \mathcal{W}_2^2(\mu_z,\nu_z) d\lambda(z).
\end{align*}
Therefore, \eqref{eq:optimal affine estimation of barycenter problem} boils down to the following:
\begin{equation} \label{eq:optimal sigma}
\inf_{\Sigma \succ 0} \Big\{ \int_{\mathcal{Z}} ||\mathbb{E}(Y|X,Z)_z - T_{\Sigma}(\mathbb{E}(Y|X,Z)_z,z)||^2_2 d\lambda(z) \Big\}.
\end{equation}
Finally, notice that 
\begin{align*}
& \int_{\mathcal{Z}} ||\mathbb{E}(Y|X,Z)_z - T_{\Sigma}(\mathbb{E}(Y|X,Z)_z,z)||^2_2 d\lambda(z)\\
= & \int_{\mathcal{Z}} ||\mathbb{E}(Y|X,Z)_z||^2_2  + ||T_{\Sigma}(\mathbb{E}(Y|X,Z)_z,z)||^2_2 - 2 \langle \mathbb{E}(Y|X,Z)_z,  T_{\Sigma}(\mathbb{E}(Y|X,Z)_z,z)\rangle_2 d\lambda(z)\\
= & \int_{\mathcal{Z}} \trace(\Sigma_{Y|X_z})  + \trace(\Sigma) - 2 \mathbb{E}(\mathbb{E}(Y| X,Z)_z^T T_{\Sigma}(\mathbb{E}(Y|X,Z)_z,z)  d\lambda(z)\\
= & \int_{\mathcal{Z}} \trace(\Sigma_{Y|X_z})  + \trace(\Sigma) - 2 \langle T_{\Sigma},  \Sigma_{Y|X_z} \rangle_F  d\lambda(z)\\
= & \int_{\mathcal{Z}} ||\mathbb{E}(Y|X,Z)_z' -T_{\Sigma}(\mathbb{E}(Y|X,Z)_z',z) ||^2_2 d\lambda(z),
\end{align*}
where $\langle \cdot,\cdot \rangle_F$ denotes the Frobenius inner product and $X' \sim \mathcal{N}(m_X,\Sigma_X)$ denotes the Gaussian analog of $X$. It follows from definition of $T_{\text{affine}}(\mathbb{E}(Y|X,Z)_z,z)$ with $T_{\text{affine}}(\cdot,z) := \Sigma_{Y|X_z}^{-\frac{1}{2}} (\Sigma_{Y|X_z}^{\frac{1}{2}} \Sigma \Sigma_{Y|X_z}^{\frac{1}{2}} )^{\frac{1}{2}} \Sigma_{Y|X_z}^{-\frac{1}{2}}$ and Lemma $\ref{l:Location-scale Barycenter}$ that $\int_{\mathcal{Z}} ||\mathbb{E}(Y|X,Z)_z - \mathbb{E}(Y|X,Z)^{\dagger}_z||^2_2 d\lambda(z)$ is the unique lower bound of the objective function in \eqref{eq:optimal sigma}. It then follows from the uniqueness of Brenier's map that $\mathbb{E}(Y|X,Z)^{\dagger}$ is the unique solution to \eqref{eq:optimal affine estimation of barycenter problem}. 
\end{proof}

In this section, we focus on applying the optimal affine transport map and the pseudo-barycenter to find a computationally efficient estimation of the optimal fair learning outcome in high-dimensional space. As we mentioned above, it will soon become clear in the next two sections and numerical experiments that a combination of McCann interpolation and optimal affine maps in matrix form results in not only a mathematically neat solution to estimate the Pareto frontier, which significantly reduces computational expense in practice, but also a necessary tool to help us circumvent the post-processing nature and solve the optimal fair data representation problem \eqref{eq:optimal fair data representation problem}. 

Now, we are ready to address the lack of a  precise theoretical characterization of the Pareto frontier between utility and fairness, which turns out to be a natural generalization of the Wasserstein barycenter characterization of the optimal fair $L^2$-objective learning outcome.

\section{Wasserstein Geodesics Characterization of Pareto Frontier} \label{s:Pareto Frontier}

In reality, rather than looking for the optimal fair learning outcome, practitioners may have to choose a middle ground: sacrificing some prediction accuracy while tolerating a certain level of disparity. Therefore, it is tempting to generalize the barycenter characterization of the optimal fair learning outcome to the entire Pareto frontier between prediction error and statistical disparity. In this section, we show that the constant-speed geodesics from the conditional expectation sensitive marginals to their Wasserstein barycenter characterize the Pareto frontier on the Wasserstein space, in which utility loss and statistical disparity are quantified respectively by the $L^2$ norm and the average pair-wise Wasserstein distance among the sensitive marginals. As a result, given the optimal transport maps, one can derive a closed-form solution to the geodesics and thereby the Pareto frontier using McCann interpolation.

Here, we first provide a post-processing characterization of the Pareto frontier, Theorem \ref{th:Geodesics Characterization of the Pareto Frontier}, which is of theoretical interest and great generality. Then, we derive a closed-form solution to Problem \ref{prob:Optimal $L^2$-objective Learning Pareto Frontier} based on this characterization. The results form a direct generalization of the barycenter characterization, which is Lemma \ref{l:Optimal Fair $L^2$-Objective Supervised Learning Characterization}, and practitioners can apply the result together with the pseudo-barycenter and McCann interpolation to obtain the optimal affine estimation to the post-processing Pareto frontier. Later in Section \ref{s:Optimal Fair Data Representation}, we further apply the result to provide a characterization of the exact solution and an optimal affine estimation of the solution to the optimal fair data representation problem \eqref{eq:optimal fair data representation problem}.

%\subsection{Wasserstein Geodesics Characterization}

 Now, we start to characterize the Pareto frontier. 
In the rest of the section, we denote $\mathcal{L}(\mathbb{E}(Y|X,Z)) =: \mu, \mathcal{L}(\mathbb{E}(Y|X,Z)_z) =: \mu_z$.
For utility, given any measurable function $f: \mathcal{X} \times \mathcal{Z} \rightarrow \mathcal{Y}$, we define the increased prediction error by the $L^2$-norm of the difference between $f(X,Z)$ and the orthogonal projection $\mathbb{E}(Y|X,Z)$:
\begin{equation}
L(f(X,Z)) := ||\mathbb{E}(Y|X,Z) - f(X,Z)||_2 = (\int_{\mathcal{Z}} ||\mathbb{E}(Y|X,Z)_z - f(X,Z)_z)||_2^2 d\lambda(z))^{\frac{1}{2}}.
\end{equation}
To simplify notation, we also denote
\begin{equation} \label{d:increased prediction error (L)}
L(T') := L(T'(\mathbb{E}(Y|X,Z),Z)) = (\int_{\mathcal{Z}} ||\mathbb{E}(Y|X,Z)_z - T'_z(\mathbb{E}(Y|X,Z)_z)||_2^2 d\lambda(z))^{\frac{1}{2}}.
\end{equation} 
for any measurable $T': \mathcal{Y} \times \mathcal{Z} \rightarrow \mathcal{Y}$.

To relax the hard independence constraint for the Pareto frontier, we quantify the statistical disparity of a given learning outcome or prediction $\hat{Y}$ by the average pairwise Wasserstein distance among its sensitive marginals:
\begin{defi}[Wasserstein disparity] \label{d:wasserstein disparity}
\begin{equation}
D(\hat{Y},Z) := \Big(\int_{\mathcal{Z}^2} \mathcal{W}_2^2(\mathcal{L}(\hat{Y}_{z_1}), \mathcal{L}(\hat{Y}_{z_2})) d\lambda({z_1})d\lambda({z_2})\Big)^{\frac{1}{2}}.
\end{equation}
\end{defi}
In our setting, $\hat{Y} = f(X,Z)$ for some $f: \mathcal{X} \times \mathcal{Z} \rightarrow \mathcal{Y}$. Also, to simplify notation, we denote the Wasserstein disparity that remains in the already deformed (by applying $T'$) conditional expectation by
\begin{equation} \label{d:remaining disparity (D)}
D(T') := D(T'(\mathbb{E}(Y|X,Z),Z),Z) = (\int_{\mathcal{Z}^2} \mathcal{W}_2^2((T_{z_1}')_{\sharp}\mu_{z_1},(T'_{z_2})_{\sharp}\mu_{z_2}) d\lambda({z_1})d\lambda({z_2}))^{\frac{1}{2}}
\end{equation}
for any measurable $T': \mathcal{Y} \times \mathcal{Z} \rightarrow \mathcal{Y}$. Here, $T_z = T(\cdot,z): \mathcal{Y} \rightarrow \mathcal{Y}$ for $\lambda$-a.e. $z \in \mathcal{Z}$.

We adopt the Wasserstein disparity as a statistical disparity quantification due to the following desirable properties:
\begin{itemize}
\item \textbf{Wasserstein disparity characterizes statistical parity:} $$\mathcal{D}(f(X,Z),Z) =0 \iff f(X,Z) \perp Z$$
\item \textbf{Physics interpretation:} Due to the definition based on the Wasserstein distance, Wasserstein disparity can be understood as the expected minimum amount of work that is required to move one randomly chosen marginal to another random chosen one. Therefore, the larger $\mathcal{D}(f(X,Z),Z)$ is, the more necessary work is expected to remove the distributional discrepancy among the sensitive groups on $f(X,Z)$.
\end{itemize}

Now, let $T: \mathcal{Y} \times \mathcal{Z} \rightarrow \mathcal{Y}$ satisfy $T(\cdot,z)$ being the optimal transport maps from $\{\mu_z\}_z$ to their barycenter $\bar{\mu}$ for $\lambda$-a.e. $z \in \mathcal{Z}$ (See construction of $T$ in the proof of Lemma \ref{l:Optimal Fair $L^2$-Objective Supervised Learning Characterization}), we define

\begin{align} \label{d:the optimal prediction error (V)}
V := L(T) & =( \int_{\mathcal{Z}} ||\mathbb{E}(Y|X,Z)_z - T(\mathbb{E}(Y|X,Z)_z,z)||_2^2 d\lambda(z))^{\frac{1}{2}}\\
& = ( \int_{\mathcal{Z}} ||\mathbb{E}(Y|X,Z)_z - \overline{\mathbb{E}(Y|X,Z)}_z||_2^2 d\lambda(z) )^{\frac{1}{2}}.
\end{align} 

As shown in Lemma \ref{l:Optimal Fair $L^2$-Objective Supervised Learning Characterization}, $V$ is the minimum increase of $L^2$ error (or, in physics, the minimum work/energy required) to deform $\mathbb{E}(Y|X,Z)$ to satisfy statistical parity. Before showing the main result, we need to define the geodesic on metric space to show the explicit form of constant speed geodesic on the Wasserstein space, which plays a key role in the proof.

\begin{defi} [Constant-speed geodesic between two points on metric space]
Given a metric space $(X,d)$ and $x,x' \in X$, the constant-speed geodesic between $x$ and $x'$ is a continuously parametrized path $\{x_t\}_{t \in [0,1]}$ such that $x_0 = x$, $x_1 = x'$, and $d(x_s,x_t) = |t-s|d(x,x'), \forall s,t \in [0,1]$.
\end{defi}

The following lemma, which is well known as the McCann (displacement) interpolation \cite[Chapter 7]{villani2009optimal} in the optimal transport literature, shows that a linear interpolation using the optimal transport plan results in the constant-speed geodesic on the Wasserstein space.

\begin{lem}[Constant-speed geodesic on Wasserstein space, \cite{mccann1997convexity, villani2009optimal}] \label{l:Mccann Interpolation}
Given $\mu_0,\mu_1 \in (\mathcal{P}_2(\mathbb{R}^d), \mathcal{W}_2)$ and $\gamma$ the optimal transport plan in between, let $\pi_t(x,y):= (1-t)x + ty$, then

\begin{equation}
\mu_t:= (\pi_t)_{\sharp} \gamma, t \in [0,1]
\end{equation}
is the constant-speed geodesic between $\mu_0$ and $\mu_1$.
\end{lem}

\begin{proof}
See Appendix \ref{A:Section 4 Appendix}
\end{proof}

\begin{rema}[Linear interpolation formula for $\mathcal{W}_2$ deodesics]\label{r:Linear Interpolation Formula for Geodesic Path}
If there exists an optimal transport map $T$ such that $T_{\sharp}(\mu_0) = \mu_1$, then the McCann interpolation has the simple form

\begin{equation} \label{eq:mccann interpolation}
\mu_t = ((1-t)Id + tT)_{\sharp}\mu_0, t \in [0,1].
\end{equation}

\end{rema}
We apply this simple formula to obtain a closed-form estimation of the Pareto frontier in algorithm design, see Section~\ref{s:Algorithm}.

Now, we are ready to establish the main result, which shows that $V$ is a lower bound of $L(f(X,Z)) + \frac{1}{\sqrt{2}} D(f(X,Z),Z)$ for any measurable function $f: \mathcal{X} \times \mathcal{Z} \rightarrow \mathcal{Y}$ and is achieved along the constant-speed geodesics from the sensitive marginals of the conditional expectation to their barycenter on the Wasserstein space. 

\begin{thm} [$\mathcal{W}_2$ geodesics characterization of a linear Pareto frontier] \label{th:Geodesics Characterization of the Pareto Frontier}
Define $L, D, V$ as above and assume $\mu_z \in \mathcal{P}_{2, ac}(\mathcal{Y}), \lambda-a.e.$. It follows that
\begin{equation} \label{eq:lower bound for pareto frontier}
V \leq L(f(X,Z)) + \frac{1}{\sqrt{2}}D(f(X,Z),Z)
\end{equation}
for any measurable function $f: \mathcal{X} \times \mathcal{Z} \rightarrow \mathcal{Y}$. Furthermore, define $T(t)$ such that $T(t)(\cdot,z) := (1-t)Id + t(T(\cdot,z)), t \in [0,1]$ is the linear interpolation between the identity map and the optimal transport map for $\lambda$-a.e. $z \in \mathcal{Z}$, then equality holds in~\eqref{eq:lower bound for pareto frontier} if and only if $f(X,Z) = T(t)(\mathbb{E}(Y|X,Z),Z), t \in [0,1]$ as

\begin{equation}
L(T(t)) = tL(T(1)) = tV
\end{equation}

\begin{equation}
\frac{1}{\sqrt{2}}D(T(t)) = \frac{1}{\sqrt{2}}(1-t)D(T(0)) = (1-t)V.
\end{equation}
\end{thm}

\begin{proof}
See Appendix \ref{A:Section 4 Appendix}.
\end{proof}

\begin{rema}[Intuition of Theorem \ref{th:Geodesics Characterization of the Pareto Frontier}: a Euclidean analog]
Here, we provide a Euclidean analog of Theorem \ref{th:Geodesics Characterization of the Pareto Frontier}. In fact, our proof is based on the observation of the analog and equivalent to it when one considers $x \rightarrow \delta_{x}$ as an embedding from $\mathcal{X}$ to $\mathcal{P}_2(\mathcal{X})$.

Let $X:= \{x_i\}_{i = 1}^N$ be a fixed data set on the Euclidean space $\mathcal{X}$ ($N = 3$ in Figure \ref{geodesic_path_intuition}),  $\tilde{X} := \{\tilde{x}_i\}_{i = 1}^N$ be a data set consisting of $N$ arbitrarily chosen data points on $\mathcal{X}$, and define the following:
\begin{itemize}
\setlength{\parsep}{-0.2ex}
\setlength{\itemsep}{-0.2ex}
\item \text{[Euclidean analog of $V$]} $\text{std}(X) := (\frac{1}{N} \sum_{i = 1}^N ||x_i - m_x||^2)^{\frac{1}{2}}$ with $m_x := \frac{1}{N} \sum_{i = 1}^N x_i$,
\item \text{[Euclidean analog of $L$]} $\text{solid}(\tilde{X}) := (\frac{1}{N} \sum_{i = 1}^N ||x_i - \tilde{x}_i||^2)^{\frac{1}{2}}$,
\item \text{[Euclidean analog of $D$]} $\text{dotted}(\tilde{X}) := (\frac{1}{N^2} \sum_{i,j = 1}^N ||\tilde{x}_i - \tilde{x}_j||^2)^{\frac{1}{2}}$.
\end{itemize}

\begin{figure}[H]
\centering
\includegraphics[width=0.7\textwidth]{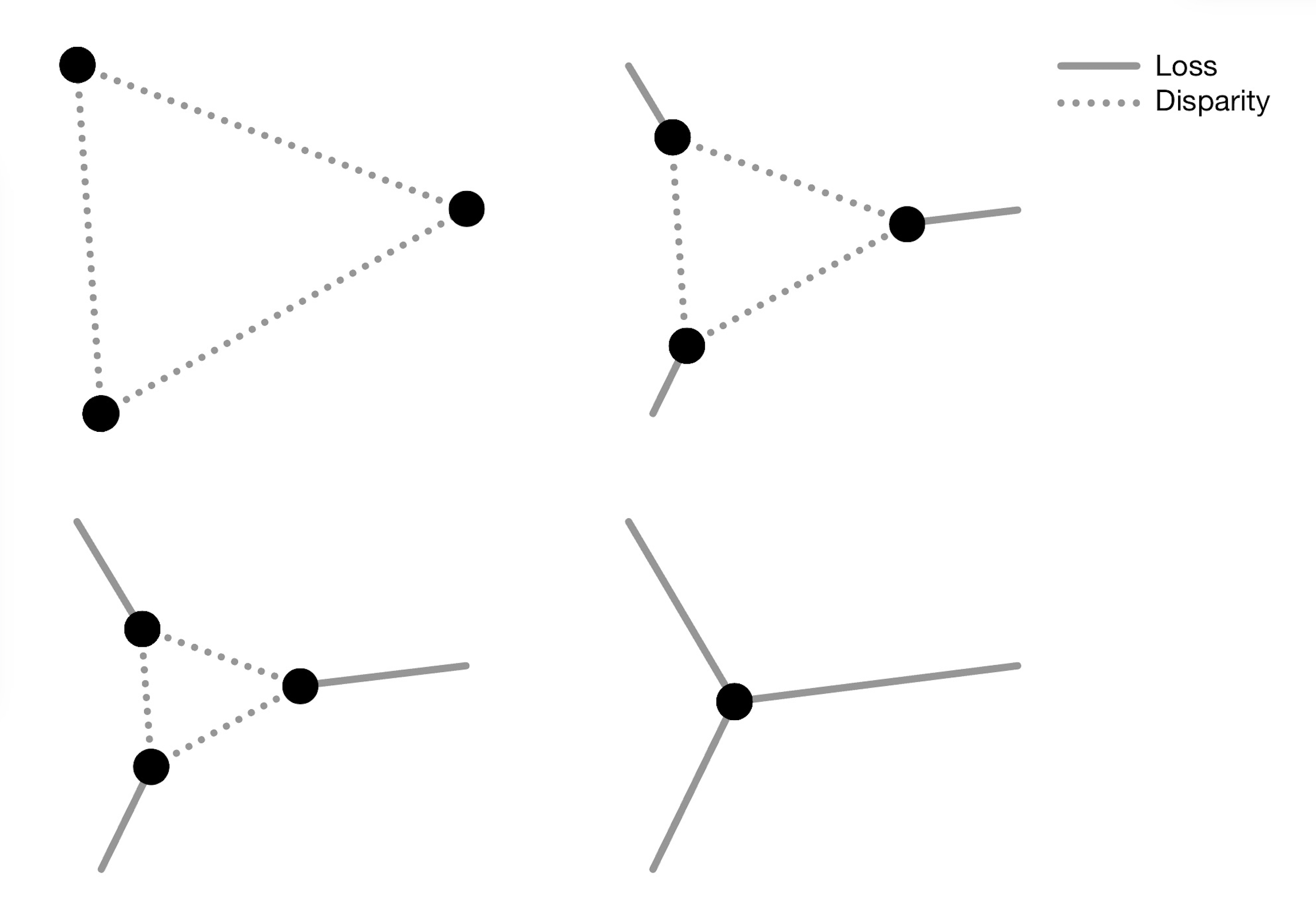}
\caption{In this figure, we have three data points on an Euclidean space traveling along straight lines (Euclidean geodesics) to their average (Euclidean barycenter). Define (1) std := the standard deviation of the three points, (2) solid line (loss) := the average moving (Euclidean) distance away from their original location, and (3) dotted line (disparity) := the average pairwise (Euclidean) distance among them. One can show that $\text{std} \leq \text{solid} + \frac{1}{\sqrt{2}} \text{dotted}$ where equality holds if and only if the three points travel at constant-speed along straight lines to their average.}
\label{geodesic_path_intuition}
\end{figure}
It is straight-forward to verify that (1) $$\text{std}(X) \leq \underbrace{\text{solid}(\tilde{X})}_{\text{utility loss}} + \frac{1}{\sqrt{2}} \underbrace{\text{dotted}(\tilde{X})}_{\text{disparity}},$$ and (2)  equality holds if and only if $\tilde{X} = X(t) := \{(1-t)x_i + tm_x\}_{i = 1}^N$ for $t \in [0,1]$ as $\text{loss}(X(t)) = t\text{std}(X)$ and $\frac{1}{\sqrt{2}}\text{disparity}(X(t)) = (1-t)\text{std}(X)$.
\end{rema}

Since $V$ (the minimum work or energy required for statistical parity) is fixed for the data $(X,Y,Z)$ when one applies $(X,Z)$ to predict $Y$, the above theorem implies that the Pareto frontier between the increased prediction error $L(T)$ and the remaining disparity $D(T)$ is a line that results from the constant-speed geodesics from the marginal conditional expectations to their barycenter on the Wasserstein space. In particular, let $T(t)(\mathbb{E}(Y|X,Z),Z) := \{T(t)(\mathbb{E}(Y|X_z),z)\}_z$, $\lambda$-a.e., $t \in [0,1]$, we arrive at a closed-form solution to Problem \ref{prob:Optimal $L^2$-objective Learning Pareto Frontier}:

\begin{coro}[Pareto optimal fair $L^2$-objective learning]\label{corr:Pareto Optimal Fair L2-objective Learning}
Given $(X,Y,Z)$ satisfying $\mu_z \in \mathcal{P}_{ac}$, $\lambda$-a.e., then \begin{equation}
 f_d(X,Z) := 
\begin{cases}
    T(1- \frac{d}{\sqrt{2}V})(\mathbb{E}(Y|X,Z),Z), & \text{if } d \in [0,\sqrt{2}V]\\
    \mathbb{E}(Y|X,Z), & \text{if } d \in (\sqrt{2}V, \infty)
\end{cases}
\end{equation}
are the unique solutions to Problem \ref{prob:Optimal $L^2$-objective Learning Pareto Frontier} for $d \in [0,\infty)$.
\end{coro}

\begin{proof}
If $d \in (\sqrt{2}V, \infty)$, then it follows from Theorem \ref{th:Geodesics Characterization of the Pareto Frontier} that $D(\mathbb{E}(Y|X,Z)) = D(T(0)) = \sqrt{2}V < d$. Hence, Problem \ref{prob:Optimal $L^2$-objective Learning Pareto Frontier} reduces to the unconstrained $L^2$ projection problem and the optimal solution is $\mathbb{E}(Y|X,Z)$. Now, for a fixed $d \in [0,\sqrt{2}V]$, assume for contradiction that $\exists f \in L^2(\mathcal{X} \times \mathcal{Z}, \mathcal{Y})$ such that $$||Y-f(X,Z)||_2^2 < ||Y - T(t)(\mathbb{E}(Y|X,Z),Z)||_2^2$$ for $t = 1- \frac{d}{\sqrt{2}V}$. Then, let $\overline{f(X,Z)}$ denote the Wasserstein barycenter of $\{f(X,Z)_z\}_z$, we have \begin{align*}
||Y - \overline{f(X,Z)}||_2^2 &  \leq ||Y - f(X,Z)||_2^2 + ||f(X,Z)- \overline{f(X,Z)}||_2^2\\
& < ||Y - T(t)(\mathbb{E}(Y|X,Z), Z)||_2^2 + ||f(X,Z)- \overline{f(X,Z)}||_2^2\\
& = ||Y - \mathbb{E}(Y|X,Z)||_2^2 + L(T(t)) + \frac{1}{\sqrt{2}}D(f(X,Z))\\
& = ||Y - \mathbb{E}(Y|X,Z)||_2^2 + (V - \frac{1}{\sqrt{2}}d) + \frac{1}{\sqrt{2}}d\\
& = ||Y - \mathbb{E}(Y|X,Z)||_2^2 + V
\end{align*}
where the second line follows from the assumption, the third from $L^2$ orthogonal decomposition and Theorem \ref{th:Geodesics Characterization of the Pareto Frontier}, and the forth from the assumption and Theorem \ref{th:Geodesics Characterization of the Pareto Frontier}. The strict inequality above contradicts the optimality of $\overline{\mathbb{E}(Y|X,Z)}$ shown in Lemma \ref{l:Optimal Fair $L^2$-Objective Supervised Learning Characterization}. That proves the optimality of $T(1- \frac{d}{\sqrt{2}V})(\mathbb{E}(Y|X,Z),Z)$ for the fixed $d$. Uniqueness result follows from the uniqueness of $\overline{\mathbb{E}(Y|X,Z)}$ shown in Lemma \ref{l:Optimal Fair $L^2$-Objective Supervised Learning Characterization}. Since the choice of $d \in [0,\sqrt{2}V]$ is arbitrary, we are done.
\end{proof}

We note that Corollary \ref{corr:Pareto Optimal Fair L2-objective Learning} together with Lemma \ref{l:Mccann Interpolation} and Remark \ref{r:Linear Interpolation Formula for Geodesic Path} provide a post-processing approach to (estimate) the Pareto frontier: applying McCann interpolation to the Brenier's maps between the learning outcome sensitive marginals $\{\mathbb{E}(Y|X,Z)_z\}_z$ and their (pseudo-) barycenter. One can apply Algorithm \ref{a:independent} directly with the learning outcome marginals as inputs.

From a theoretical perspective, various metrics of disparity that differ from $D$, the Wasserstein disparity (Definition \ref{d:wasserstein disparity}), can be used and the theoretical results derived in this section provide a lower bound estimation for the Pareto frontier that uses other disparity metrics. The quality of the lower bound can be studied using the relationship between the Wasserstein distance and the defined disparity metric. Also, the present work provides a numerical study on the lower bound estimation in Section 6 to which we refer the interested readers for more details.

In practice, various metrics of disparity are adopted, such as the prediction success ratio (difference from 1) in classification \cite{calmon2017optimized} and the Kolmogorov-Smirnov distance for 1-dimensional regression \cite{chzhen2020fair}.  The proposed estimation of the Pareto frontier leaves the choice of $\alpha$ to practitioners who would face specific fairness requirements and disparity metrics.

\section{Optimal Fair Data Representation for Supervised Learning}\label{s:Optimal Fair Data Representation}

In this section, we study the optimal fair data representation problem, Problem \ref{prob:Optimal Fair Data Representation for Conditional Expectation Estimation}, that is motivated by the current challenges in the pre-processing or synthetic data design approach to fair machine learning.  To solve the problem, we first characterize the exact solution using a dependent and independent Wasserstein barycenter pair, see Lemma \ref{l:Characterization of Optimal Fair Data Representation}. Then, we define a dependent and independent pseudo-barycenter pair via optimal affine maps, and prove that the pair is the exact optimal fair data representation with Gaussian marginals, cf.\ Lemma \ref{l:Solution to the Optimal Fair Data Representation in Gaussian Case} and the optimal affine estimate of the representation with general marginals in Theorem \ref{th:Justification of Pseudo-Barycenter in General Distribution Case}.

\subsection{Wasserstein Barycenter Pair Characterization}

We will prove a characterization of the solutions to Problem \ref{prob:Optimal Fair Data Representation for Conditional Expectation Estimation}. To start, notice that since $(\tilde{X},Z) = T \otimes Id|_{\mathcal{Z}}(X,Z)$ for some measurable map $T \otimes Id|_{\mathcal{Z}}: \mathcal{X} \times \mathcal{Z} \rightarrow \mathcal{X} \times \mathcal{Z}$, we have $\sigma((\tilde{X},Z)) \subset \sigma((X,Z))$. Also, from $\tilde{X} \perp Z$, we have $\sigma(\tilde{X}) \subset \sigma(\tilde{X}) \otimes \sigma(Z) = \sigma((\tilde{X},Z))$. Therefore,  $\sigma(\tilde{X}) \subset \sigma((X,Z))$ and it follows from $L^2$ orthogonal decomposition that

\begin{equation}\label{objective decomposition}
||Y - \mathbb{E}(\tilde{Y}|\tilde{X})||^2_2 = ||Y - \mathbb{E}(Y|X,Z)||^2_2 + ||\mathbb{E}(Y|X,Z) - \mathbb{E}(\tilde{Y}|\tilde{X})||^2_2.
\end{equation} 
The first term on the right hand side can be interpreted as the minimum loss of information by using $(X,Z)$ to predict $Y$. Furthermore, one can decompose the second term on the right hand side of $\eqref{objective decomposition}$:
\begin{align*}
   & ||\mathbb{E}(Y|X,Z) - \mathbb{E}(\tilde{Y}|\tilde{X})||^2_2\\
= & ||\mathbb{E}(Y|X,Z) - \mathbb{E}(Y|\tilde{X},Z)||_2^2 + ||\mathbb{E}(Y|\tilde{X},Z) - \mathbb{E}(\tilde{Y}|\tilde{X})||^2_2\\
= & ||\mathbb{E}(Y|X,Z) - \mathbb{E}(Y|\tilde{X},Z)||_2^2 + \int_{\mathcal{Z}} ||\mathbb{E}(Y_z|\tilde{X}) - \mathbb{E}(\tilde{Y}|\tilde{X})_z||^2_2 d\lambda(z).
\end{align*}
Here, the first equality follows from $L^2$ orthogonal decomposition. The second equality follows from disintegration, the fairness constraint $\tilde{X},\mathbb{E}(\tilde{Y}|\tilde{X}) \perp Z$, and the fact that $\tilde{X} \perp Z$ implies $$\mathbb{E}(Y_z|\tilde{X}) = \mathbb{E}(Y|\tilde{X},Z)_z.$$ See Appendix \ref{A:Section 5 Appendix} for the proof.

Now, the key observation is that, given a fixed $\tilde{X} \perp Z$, the choice of $\tilde{Y}$ depends only on the second term on the right, which forms a Wasserstein barycenter problem with marginals being $\{\mathbb{E}(Y_z|\tilde{X})\}_z$. Hence, the optimal choice of $\tilde{Y}$ is the one which satisfies $\mathbb{E}(\tilde{Y} | \tilde{X}) = \overline{\mathbb{E}(Y| \tilde{X},Z)}$, where $\overline{\mathbb{E}(Y| \tilde{X},Z)})$ is the Wasserstein barycenter of $\{\mathbb{E}(Y_z|\tilde{X})\}_z$. Therefore, we denote the optimal choice of $\tilde{Y}$ to be $\bar{Y}$ which satisfies $\mathbb{E}(\bar{Y}|\tilde{X}) = \overline{\mathbb{E}(Y| \tilde{X},Z)}$.

It remains to find the optimal choice of $\tilde{X}$. The following result shows that the optimal choice is the one admissible $\tilde{X}$ which generates the finest sigma-algebra.

\begin{lem}[Finer sigma-algebra, more accurate optimal fair learning] \label{l:Finer Sigma-algebra Implies More Accurate Optimal Fair Learning}
Let $\tilde{X}, \tilde{X}' \in \{\tilde{X} \in \mathcal{D}|_{\mathcal{X}}: \tilde{X} \perp Z\}$. If $\sigma(\tilde{X}') \subset \sigma(\tilde{X})$, then
\begin{equation}
||\mathbb{E}(Y|X,Z) - \mathbb{E}(\bar{Y}|\tilde{X})||^2_2 \leq ||\mathbb{E}(Y|X,Z) - \mathbb{E}(\bar{Y}'|\tilde{X}')||^2_2
\end{equation}
where $\bar{Y}$ and $\bar{Y}'$ satisfy $\mathbb{E}(\bar{Y}|\tilde{X}) = \overline{\mathbb{E}(Y| \tilde{X},Z)}$ and $\mathbb{E}(\bar{Y}'|\tilde{X}') = \overline{\mathbb{E}(Y'| \tilde{X}',Z)}$.
\end{lem}

\begin{proof}
See Appendix \ref{A:Section 5 Appendix}.
\end{proof}

Therefore, it is clear that our optimal choice of $\tilde{X}$ is the one that generates the finest sigma-algebra while satisfying $\tilde{X} \perp Z$. The following technical lemma shows that the barycenter of $\{X_z\}_{z \in \mathcal{Z}}$ is one of the optimal choices.

\begin{lem}[$\bar{X}$ generates the finest sigma-algebra among admissible]\label{l:X Bar Generates the Finest Sigma-algebra} If $\{\mathcal{L}(X_z)\}_z$ $\subset \mathcal{P}_{2,ac}(\mathcal{X})$ $\lambda$-a.e., then $\sigma((\bar{X},Z)) = \sigma((X,Z))$. In addition, $\sigma(\tilde{X}) \subset \sigma(\bar{X})$ for all $\tilde{X} \in \{\tilde{X} \in \mathcal{D}|_{\mathcal{X}} : \tilde{X} \perp Z\}$.
\end{lem}

\begin{proof}
See Appendix \ref{A:Section 5 Appendix}.
\end{proof}

Therefore, Lemma \ref{l:Finer Sigma-algebra Implies More Accurate Optimal Fair Learning}, Lemma \ref{l:X Bar Generates the Finest Sigma-algebra}, and the choice of $\bar{Y}$ above together provide a characterization of the solution to Problem \ref{prob:Optimal Fair Data Representation for Conditional Expectation Estimation}.

\begin{lem}[Characterization of optimal fair data representation] \label{l:Characterization of Optimal Fair Data Representation} Let $\bar{X}$ and $\overline{\mathbb{E} (Y| \bar{X},Z)}$ denote the respective Wasserstein barycenter of $\{X_z\}_z$ and $\{\mathbb{E}(Y_z| \bar{X})\}_z$. If $\{\mathcal{L} (X_z)\}_z$ $\subset \mathcal{P}_{2,ac}(\mathcal{X})$ and $\{\mathcal{L} (\mathbb{E}(Y|\bar{X},Z)_z)\}_z \subset \mathcal{P}_{2,ac}(\mathcal{Y})$,  then the following are equivalent:
\begin{itemize}
\item $(\tilde{X},\tilde{Y}) \in \arg\min_{(\tilde{X},\tilde{Y}) \in \mathcal{D}} \{ ||Y - \mathbb{E}(\tilde{Y}|\tilde{X})||^2_2: \tilde{X}, \mathbb{E}(\tilde{Y}|\tilde{X},Z) \perp Z \}$.
\item $(\tilde{X},\tilde{Y}) \in \{(\tilde{X},\tilde{Y}) \in \mathcal{D} : \sigma(\tilde{X}) = \sigma(\bar{X}), \mathbb{E}(\tilde{Y}|\bar{X}) = \overline{\mathbb{E} (Y| \bar{X},Z)} \}$.
\end{itemize}
\end{lem}

In Lemma \ref{l:Characterization of Optimal Fair Data Representation}, the choice of $\bar{X}$ is not unique. In fact, any random variable $\tilde{X}$ that satisfies $\sigma(\tilde{X}) = \sigma(\bar{X})$ can be our choice according to Lemma \ref{l:Finer Sigma-algebra Implies More Accurate Optimal Fair Learning} and Lemma \ref{l:X Bar Generates the Finest Sigma-algebra}. This is because any $\tilde{X}$ that satisfies the above conditions gives $\mathbb{E}(Y|\tilde{X}) = \mathbb{E}(Y|\bar{X})$. For both theoretical and computational convenience, we fix our choice to be $\bar{X}$ from now on. \\

\begin{rema}[Application of the optimal fair representation characterization to algorithm design]
In theory, we should always take $\bar{X}$ because we prove that $\bar{X}$ generates the finest sigma-algebra among all the admissible $\tilde{X}$ that is independent of Z. Especially when working with data sets with clear high-dimensional structure such as image data, one should apply more complicated models to estimate the optimal transport map instead of using affine maps. But when working with data with less high-dimensional structure such as tabular data, we hope to take advantage of the simplicity, robustness, and interpretability of linear maps in practice and hence restrict the admissible transport maps to be affine, as mentioned in Remark \ref{r:Applying Pseudo-barycenter vs the Exact Barycenter}. Therefore, we showed that the pseudo-barycenter $X^{\dagger}$, which is equal to $\bar{X}$ in the Gaussian case and solves a relaxed version of the barycenter problem in the general distribution case, can be achieved using optimal affine maps. As a result, we  apply $X^{\dagger}$ in the algorithm design and experiments. Still, if there is no concern about over-fitting or computational cost, it is recommended for strict statistical parity guarantee purposes to compute $\bar{X}$ to improve the result.
\end{rema}

Now, it remains to find $\bar{Y}$ to obtain the optimal fair data representation characterized by Lemma \ref{l:Characterization of Optimal Fair Data Representation}. In general, it is difficult to find $\overline{\mathbb{E} (Y| \bar{X},Z) }$, not to mention find a $\tilde{Y}$ satisfying $\mathbb{E}(\tilde{Y}|\bar{X}) = \overline{\mathbb{E} (Y| \bar{X},Z) }$. The key observation here is that if the Brenier's maps $\{T_{y|\bar{X}}(\cdot,z)\}_z$ that push $\{ \mathbb{E} (Y_z | \bar{X}) \}_z$ forward to $\overline{\mathbb{E} (Y| \bar{X},Z) }$ are affine, then a straight-forward choice in $\bar{Y}$ is $\{T_{y|\bar{X}}(Y_z,z)\}_{z \in \mathcal{Z}} = T_{y|\bar{X}}(Y,Z)$. This step is the key to circumvent the post-processing nature. Therefore, following the same derivation of \eqref{eq:optimal affine estimation of barycenter problem} from \eqref{eq:post-processing constrained optimization for optimal learning outcome} in Section \ref{s:Post-Processing} to guarantee feasibility of affine maps, we relax the fairness constraint to the first two moments in Problem~\ref{prob:Optimal Fair Data Representation for Conditional Expectation Estimation}, and show a pseudo-barycenter pair provides us an exact solution to Problem~\ref{prob:Optimal Fair Data Representation for Conditional Expectation Estimation} in the Gaussian marginal case and the optimal affine estimation in the general marginal case.

\subsection{Fairness with Gaussian Marginals}

Assume $\{(X_z, Y_z)\}_z$ to be non-degenerate Gaussian vectors $\lambda$-a.e. and define the following:

\begin{defi}[Independent pseudo-barycenter: $X^{\dag}$] \label{d:Independent Pseudo-barycenter}

\begin{equation}\label{eq:independent pseudo-barycenter }
X^{\dag} := T_{x}(X,Z),
\end{equation}
where
\begin{equation} \label{eq:independent pseudo-barycenter affine map}
T_{x}(\cdot,z) :=  \Sigma_{X_z}^{-\frac{1}{2}} (\Sigma_{X_z}^{\frac{1}{2}} \Sigma \Sigma_{X_z}^{\frac{1}{2}} )^{\frac{1}{2}} \Sigma_{X_z}^{-\frac{1}{2}}
\end{equation}
and $\Sigma$ is the unique solution to
\begin{equation} \label{eq:independent pseudo-barycenter cov estimation}
\int_{\mathcal{Z}} (\Sigma^{\frac{1}{2}} \Sigma_{X_z} \Sigma^{\frac{1}{2}})^{\frac{1}{2}} d\lambda(z) = \Sigma.
\end{equation}
\end{defi}

\begin{defi}[Dependent pseudo-barycenter: $Y^{\dag}$] \label{d:Dependent Pseudo-barycenter}

\begin{equation}\label{eq:dependent pseudo-barycenter }
Y^{\dag} := T_{y|X^{\dagger}}(Y,Z)
\end{equation}
where

\begin{equation}\label{eq:dependent pseudo-barycenter affine map}
T_{y|X^{\dagger}}(\cdot,z) :=  \Sigma_{Y_z|X^{\dagger}}^{-\frac{1}{2}} (\Sigma_{Y_z|X^{\dagger}}^{\frac{1}{2}} \Sigma \Sigma_{Y_z|X^{\dagger}}^{\frac{1}{2}} )^{\frac{1}{2}} \Sigma_{Y_z|X^{\dagger}}^{-\frac{1}{2}}
\end{equation}
with $\Sigma_{Y_z|X^{\dagger}} := \Sigma_{Y_z X^{\dagger}} \Sigma_{X^{\dagger}}^{-1} \Sigma_{Y_z X^{\dagger}}^T$, and $\Sigma$ is the unique solution to

\begin{equation} \label{eq:dependent pseudo-barycenter cov estimation}
\int_{\mathcal{Z}} (\Sigma^{\frac{1}{2}} \Sigma_{Y_z|X^{\dagger}} \Sigma^{\frac{1}{2}})^{\frac{1}{2}} d\lambda(z) = \Sigma
\end{equation}

\end{defi}

Here, to obtain (an estimation of) the solution to equations \eqref{eq:dependent pseudo-barycenter cov estimation} and \eqref{eq:independent pseudo-barycenter cov estimation}, we apply the iterative method \eqref{eq:iterative method} in Remark \ref{r:solution to baryceneter equation} when designing our algorithm in Section \ref{s:Algorithm}.

Since it is a direct result of Lemma \ref{l:Location-scale Barycenter} that $X^{\dag} = \bar{X}$, the goal is now to show that

\begin{equation}
\mathbb{E}(Y^{\dag}|\bar{X}) = \overline{\mathbb{E}(Y|\bar{X},Z)},
\end{equation}
and therefore by Lemma $\ref{l:Characterization of Optimal Fair Data Representation}$ to conclude $\mathbb{E}(Y^{\dag}|X^{\dag}) = \mathbb{E}(Y^{\dag}|\bar{X})$ indeed minimizes the estimation error while staying independent of $Z$.

To prove the above equation and justify the definition of the pseudo-barycenter, we need the following results: (1) existence and uniqueness of both $\bar{X}$ and $\overline{\mathbb{E}(Y|\bar{X},Z)}$; (2) affinity of the corresponding Brenier's maps $T_x(\cdot,z)$ and $T_{y|X^{\dagger}}(\cdot,z)$. By assumption, we have $\{\mathcal{L}(X_z)\}_z \subset \mathcal{P}_{2,ac}(\mathcal{X})$, and $\{\mathcal{L}(\mathbb{E}(Y_z|\bar{X}))\}_z \subset \mathcal{P}_{2,ac}(\mathcal{Y})$. The existence and uniqueness then follow directly from Lemma $\ref{Existence and Uniqueness of Barycenter}$. It remains to show that the corresponding Brenier's maps are affine. But by Lemma $\ref{l:Location-scale Barycenter}$, if $\{X_z\}_z$ and $\{\mathbb{E}(Y_z|\bar{X})\}_z$ both are from some location-scale family, then the barycenters are also from the corresponding location-scale family and the Brenier's maps are affine.

The following result shows that if $\{Y_z\}_z$ come from the same location-scale family, then $\{\mathbb{E}(Y_z|\bar{X})\}_z$ also belongs to the same location-scale family.

\begin{lem}[Conditional expectation preserves location-scale family] \label{l:Conditional Expectation Preserves Dependent Variable Location-scale Family}
Assume that $\{Y_z\}_z \subset \mathcal{F}(P_0)$ for some $P_0$, then $\{\mathbb{E}(Y_z|\bar{X})\}_z \subset \mathcal{F}(\mathcal{L}(\mathbb{E}(Y_z|\bar{X})))$ for any $z$.
\end{lem}

\begin{proof}
This follows immediately from the existence of positive definite affine transformations among $\{Y_z\}_z$, Lemma $\ref{l:Optimal Affine Map}$, and the linearity of conditional expectation.
\end{proof}

Therefore, given $\{(X_z,Y_z)\}_z$ being Gaussian vectors, we have $\{(\bar{X},Y_z)\}$ being Gaussian vectors, which further implies that $\{\mathbb{E}(Y_z|\bar{X})\}_z$ are Gaussian vectors by Lemma \ref{l:Conditional Expectation Preserves Dependent Variable Location-scale Family}. (We note that it is not necessary to apply Lemma \ref{l:Conditional Expectation Preserves Dependent Variable Location-scale Family} to show $\{\mathbb{E}(Y_z|\bar{X})\}_z$ are Gaussian because it is a well-known result in probability theory, but the lemma becomes necessary later in the case of general marginal distributions.)

\begin{lem}[Solution to the optimal fair data representation in the Gaussian case] \label{l:Solution to the Optimal Fair Data Representation in Gaussian Case}
Let $\{(X_z,Y_z)\}_z$ be Gaussian vectors satisfying $\Sigma_z \succ 0$ $\lambda$-a.e., then there exists a unique barycenter pair $(\bar{X}, \overline{\mathbb{E}(Y| \bar{X},Z)})$ which are Gaussian vectors characterized by the covariance matrix being the unique solution to 
\begin{equation}
\int_{\mathcal{Z}} (\Sigma^{\frac{1}{2}} S \Sigma^{\frac{1}{2}})^{\frac{1}{2}} d\lambda(z) = \Sigma
\end{equation}
for $S \in \{\Sigma_{X_z},\Sigma_{Y_z|X^{\dag}}\}$ respectively, where $\Sigma_{Y_z|X^{\dag}} = \Sigma_{Y_z X^{\dag}} \Sigma_{X^{\dag}}^{-1} \Sigma_{Y_z X^{\dag}}^T$. Moreover, $\{T_x(\cdot,z)\}_z$ and $\{T_{y|X^{\dag}}(\cdot,z)\}_z$ which push $X_z$ and $\mathbb{E}(Y_z|\bar{X})$ respectively to $\bar{X}$ and $\overline{\mathbb{E}(Y|\bar{X},Z)}$ are affine with closed-form \eqref{eq:independent pseudo-barycenter affine map} and \eqref{eq:dependent pseudo-barycenter affine map}. As a result, for $\lambda-a.e. \ z \in \mathcal{Z}$, we have
\begin{equation}
\overline{\mathbb{E}(Y| \bar{X},Z)}_z = T_{y|X^{\dag}}(\mathbb{E}(Y_z | T_x(X_z,z)),z) =  \mathbb{E}(T_{y|X^{\dag}}(Y_z,z) | T_x(X_z,z))
\end{equation}
\end{lem}

\begin{proof}
The existence, uniqueness, and Gaussianity of the barycenter follow from Lemma $\ref{l:Location-scale Barycenter}$, whereas the affinity of corresponding Brenier's maps results from Lemmas $\ref{l:Conditional Expectation Preserves Dependent Variable Location-scale Family}$ and $\ref{l:Optimal Affine Map}$.
\end{proof}

The above result provides us a theoretical foundation to apply the affine maps $\{T_x(\cdot,z)\}_z$ and $\{T_{y|X^{\dag}}(\cdot,z)\}_z$ to $\{X_z\}_z$ and $\{Y_z\}_z$ respectively as a pre-processing step before the training step. 

Furthermore, notice that although $T_{y|X^{\dag}}(\mathbb{E}(Y_z|\bar{X}),z) = \overline{\mathbb{E}(Y_z| \bar{X},Z)}_z$ $ \lambda$-a.e. by construction,  $\{T_{y|X^{\dag}}(Y_z,z) \}_z$ does not agree in general: for $z_1 \neq z_2$,

\begin{equation}
T_{y|X^{\dag}}  (Y_{z_1},{z_1}) \neq T_{y|X^{\dag}} (Y_{z_2}, {z_2}).
\end{equation}
The pseudo-barycenter solves the disagreement by merging them directly.  Despite of the differences among $\{T_{y|X^{\dag}} (Y_z,z) \}_z$, the $L^2$ projections of them on $\sigma(\bar{X})$ agree. Therefore, a direct merging of $\{T_{y|X^{\dag}} (Y_z,z) \}_z$ is simply: $T_{y|X^{\dag}} (Y,Z) = Y^{\dag}$. It follows:

\begin{align*}
\mathbb{E}(Y^{\dag}|X^{\dag}) & = \mathbb{E}(Y^{\dag}|\bar{X}) = \mathbb{E}(T_{y|X^{\dag}}(Y,Z)|\bar{X})\\
& = \int_{\mathcal{Z}} \mathbb{E}( T_{y|X^{\dag}} (Y_z,z) | \bar{X}) d\lambda(z)\\
& = \int_{\mathcal{Z}} T_{y|X^{\dag}} (\mathbb{E}(Y_z| \bar{X}),z) d\lambda(z)\\
& = \int_{\mathcal{Z}} T_{y|X^{\dag}} (\mathbb{E}(Y| \bar{X},Z)_z,z) d\lambda(z)\\
& = \int_{\mathcal{Z}} \overline{\mathbb{E}(Y| \bar{X},Z)}_z d\lambda(z) = \overline{\mathbb{E}(Y| \bar{X},Z)},
\end{align*}
where the second equality follows from disintegration, the third from linearity of $T_{y|\bar{X}}$, and the forth from $\mathbb{E}(Y_z| \bar{X}) = \mathbb{E}(Y| \bar{X},Z)_z$. Therefore, we have proved a result that justifies the definition of the pseudo-barycenter:

\begin{thm}[Justification of $Y^{\dag}$ in Gaussian case] \label{th:Justification of Dependent Pseudo-barycenter in Gaussian Case}
$(X^{\dag}, Y^{\dag})$ is a solution to Problem~\ref{prob:Optimal Fair Data Representation for Conditional Expectation Estimation}
\begin{equation}
\inf_{(\tilde{X},\tilde{Y}) \in \mathcal{D}} \{ ||Y - \mathbb{E}(\tilde{Y}|\tilde{X})||^2_2: \tilde{X}, \mathbb{E}(\tilde{Y}|\tilde{X},Z) \perp Z\},
\end{equation}
if $\{(X_z,Y_z)\}_z$ are non-degenerate Gaussian vectors.
\end{thm}

\subsection{The Case of General Distribution} \label{s:The Case of General Distribution}

In practice, one should not always expect the sensitive marginal data distributions to be Gaussian, and the results we derived under the assumption of Gaussianity may not apply to the general marginal distribution case. Instead, we solve the following relaxed optimal fair data representation problem:
\begin{equation} \label{eq:relaxation of the pre-processing characterization}
\inf_{(\tilde{X},\tilde{Y}) \in \mathcal{D}}  \{ ||Y - \mathbb{E}(\tilde{Y}|\tilde{X})||^2_2: m_{\tilde{X}}, m_{\tilde{Y}|\tilde{X}}, \Sigma_{\tilde{X}}, \Sigma_{\tilde{Y}|\tilde{X}} \perp Z\},
\end{equation}
where $m_{\tilde{Y}|\tilde{X}} := \mathbb{E}(\mathbb{E}(\tilde{Y}|\tilde{X},Z))$ and similarly for $\Sigma_{\tilde{Y}|\tilde{X}}$,  to find the optimal affine estimation of the true solution to the original Problem \ref{prob:Optimal Fair Data Representation for Conditional Expectation Estimation}. The fairness guarantee of the affine estimation is the same as mentioned in Remark \ref{r:Fairness Guarantee of the Relaxation}.

Now, we justify the pseudo-barycenter pair $(X^{\dag}, Y^{\dag})$ in the case of general distributions by proving it is a solution to the relaxed optimal fair $L^2$-objective supervised learning problem \eqref{eq:relaxation of the pre-processing characterization}.  To start, notice that $(X^{\dag}, Y^{\dag}) \in \mathcal{D}$ and satisfies $m_{X^{\dag}}, m_{Y^{\dag}|X^{\dag}}, \Sigma_{X^{\dag}}, \Sigma_{Y^{\dag}|X^{\dag}} \perp Z$ by construction and therefore is admissible.

\begin{rema}[Finest sigma-algebra vs.\ most variance] \label{r:Finest Sigma Algebra vs. Most Variance}
Due to the relaxation, the admissible $\tilde{X} \in \mathcal{D}|_\mathcal{X}$ are no longer required to be independent of $Z$. Furthermore, without the assumption of Gaussianity, $X^{\dag}$ is no longer equal to $\bar{X}$. As a result, although one can still prove $\sigma((X,Z)) = \sigma((X^{\dag},Z))$ by following the same argument in the proof of Lemma \ref{l:X Bar Generates the Finest Sigma-algebra} as in the Gaussian case, but this fact now cannot imply $\sigma(\tilde{X}) \subset \sigma(X^{\dag})$ due to the lack of independence condition. Instead, the present work shows that $\var(\tilde{X}) \leq \var(X^{\dag})$ for all admissible $\tilde{X} \in \mathcal{D}|_{\mathcal{X}}$, which in general implies $\sigma(\tilde{X}) \subset \sigma(X^{\dag})$. For example, whenever set inclusion forms an order between $\sigma(\tilde{X})$ and $\sigma(X^{\dag})$, then it is true that $\var(\tilde{X}) \leq \var(X^{\dag})$ implies $\sigma(\tilde{X}) \subset \sigma(X^{\dag})$. As a result, we still fix $X^{\dag}$ as our optimal choice among all the admissible $\tilde{X} \in \mathcal{D}|_{\mathcal{X}}$.
\end{rema}

In addition, for any $\Sigma \succ 0$, define

\begin{equation}
T_{\Sigma,x} := \Sigma_{X_z}^{-\frac{1}{2}}(\Sigma_{X_z}^{\frac{1}{2}} \Sigma \Sigma_{X_z}^{\frac{1}{2}})^{\frac{1}{2}}\Sigma_{X_z}^{-\frac{1}{2}}
\end{equation}

\begin{equation}
T_{\Sigma} := \Sigma_{Y_z|X^{\dag}_z}^{-\frac{1}{2}}(\Sigma_{Y_z|X^{\dag}_z}^{\frac{1}{2}} \Sigma \Sigma_{Y_z|X^{\dag}_z}^{\frac{1}{2}})^{\frac{1}{2}}\Sigma_{Y_z|X^{\dag}_z}^{-\frac{1}{2}}
\end{equation}
where $\Sigma_{Y_z|X^{\dag}_z} := \mathbb{E}((\mathbb{E}(Y_z|X^{\dag}_z) - m_{Y_z})(\mathbb{E}(Y_z|X^{\dag}_z) - m_{Y_z})^T)$ and $\mathbb{E}(Y_z|X^{\dag}_z) := \mathbb{E}(Y|X^{\dag},Z)_z$. Now, the goal is to show $(X^{\dag}, Y^{\dag})$ is indeed a solution to the relaxed problem $\eqref{eq:relaxation of the pre-processing characterization}$, under the following two assumptions:

\begin{itemize} \label{assumptions for general distribution}
\item[1] Set inclusion forms an order between $X^{\dag}$ and all $\tilde{X} \in \{\tilde{X} \in \mathcal{D}|_{\mathcal{X}} : m_{\tilde{X}}, \Sigma_{\tilde{X}} \perp Z\}$.
\item[2] $\Sigma_{Y_z|X^{\dag}_z} = \Sigma_{Y_z X^{\dag}_z} \Sigma_{X^{\dag}_z}^{-1} \Sigma_{Y_z X^{\dag}_z}^T$. 
\end{itemize}

\begin{rema}[Applicability of the assumptions] \label{r:Applicability of the Assumptions}
For the first assumption, Lemma \ref{l:Variance Reduction of the Barycenter} below guarantees that $X^{\dag}$ generates the finest sigma-algebra among all the admissible sigma-algebras. In other words, for any admissible $\tilde{X}$, either it generates a coarser sigma-algebra than $\sigma(X^{\dag})$ or the two sigma-algebras do not contain each other. In other words, there is no admissible $\tilde{X}$ such that $\sigma(X^{\dag}) \subset \sigma(\tilde{X})$.

The second assumption allows us to directly compute the covariance matrix of $\mathbb{E}(Y_z|X^{\dag}_z)$ from $\Sigma_{Y_z X^{\dag}_z}$ and  $\Sigma_{X^{\dag}_z}$. The second assumption is necessary to keep our pre-processing approach. In general, $\mathbb{E}(Y_z|X^{\dag}_z)$ is not a linear function of $X^{\dag}_z$ as in the Gaussian case. When the second assumption is not true, our pre-processing approach uses $\Sigma_{Y_z X^{\dag}_z} \Sigma_{X^{\dag}_z}^{-1} \Sigma_{Y_z X^{\dag}_z}^T$ as our best affine estimate of $\Sigma_{Y_z|X^{\dag}_z}$.
\end{rema}

To that end, we need the following result on the relationship among the variance of the original distribution, the variance of the barycenter, and the Wasserstein distance.

\begin{lem} [Variance reduction of Wasserstein barycenter \cite{tabak2018explanation}] \label{l:Variance Reduction of the Barycenter}
Given $X$ satisfies $\{\mathcal{L}(X_z)\}_z \subset \mathcal{P}_{2,ac}(\mathcal{X})$ and $\bar{X}$ satisfies $\mathcal{L}(\bar{X})$ being the Wasserstein barycenter of $\{\mathcal{L}(X_z)\}$, it follows that

\begin{equation}
||X - \mathbb{E}(X)||_2^2 - ||\bar{X} - \mathbb{E}(\bar{X})||_2^2 = \int_{\mathcal{Z}} \mathcal{W}_2^2(\mathcal{L}(X_z), \mathcal{L}(\bar{X})) d\lambda(z)
\end{equation}

\end{lem}

As a result, we obtain the following:

\begin{lem}[$X^{\dag}$ Contains the largest variance among admissible] \label{l:Pseudo-barycenter Contains the Largest Variance}
$X^{\dag}$ is the unique solution to

\begin{equation}
\sup_{\tilde{X} \in \mathcal{D}|_{\mathcal{X}}} \ \{ \var(\tilde{X}) : m_{\tilde{X}},  \Sigma_{\tilde{X}} \perp Z \}.
\end{equation}

\end{lem}

\begin{proof}
To simplify notation, by the invariance of variance under translation and Lemma \ref{l:Rigid translation property}, we can assume without loss of generality that $m_{X_z} = 0 \ \lambda-a.e.$ in the rest of the proof, which only deal with variance and Wasserstein distance. Now, for $\lambda-a.e. \ z \in \mathcal{Z}$, we have

\begin{align*}
||X_z - T_{\Sigma,x}(X_z,z)||^2_2 = &  ||X_z||^2_2 + ||T_{\Sigma,x}(X_z,z)||^2_2 - 2 \langle X_z, T_{\Sigma,x}(X_z,z) \rangle_2 \\
= & \trace(\Sigma_{X_z})  + \trace(\Sigma) - 2 \mathbb{E}(X_z^T T_{\Sigma,x}(X_z,z)) \\
= & \trace(\Sigma_{X_z})  + \trace(\Sigma) - 2 \langle T_{\Sigma,x},  \Sigma_{X_z} \rangle_F \\
= & \trace(\Sigma_{X'_z})  + \trace(\Sigma) - 2 \langle T_{\Sigma,x},  \Sigma_{X'_z} \rangle_F \\
= & ||X'_z - T_{\Sigma,x}(X'_z,z)||^2_2\\
= & \mathcal{W}_2^2 (\mathcal{L}(X'_z), \mathcal{L}(T_{\Sigma,x}(X'_z)))
\end{align*}
where $X' \sim \mathcal{N}(m_{X}, \Sigma_{X})$ is the Gaussian analog of $X$ and $\langle \cdot, \cdot \rangle_F$ is the Frobenius inner product.

Similarly, by the disintegration theorem, we also have for $S \in \{X, X^{\dag}\}$

\begin{equation}
\var(S) = ||S||^2_2 = \int_{\mathcal{Z}} ||S_z||^2_2 d\lambda = \int_{\mathcal{Z}} \trace(\Sigma_{S_z}) d\lambda.
\end{equation}
Therefore, it follows from Lemma \ref{l:Variance Reduction of the Barycenter} that

\begin{align*}
\var(X) - \var(X^{\dag}) = & \var(X') - \var((X')^{\dag})\\
= & \var(X') - \var(\bar{X'})\\
= & \int_{\mathcal{Z}} \mathcal{W}_2^2(\mathcal{L}(X'_z), \mathcal{L}(\bar{X'})) d\lambda(z).
\end{align*}

Finally, assume there exists a $\tilde{X} \in \mathcal{D}|_{\mathcal{X}}$ such that $\var(X^{\dag}) \leq \var(\tilde{X})$. It follows $\var(X') - \var(\tilde{X'}) \leq \var(X') - \var((X')^{\dag}) = \var(X') - \var(\bar{X'})$. But since $m_{\tilde{X'}}, \Sigma_{\tilde{X'}} \perp Z$, we have $\tilde{X'} \perp Z$ as $\tilde{X'}$ is Gaussian by construction. In other words,  there exists a $\tilde{X'} \perp Z$ such that

\begin{equation}
\int_{\mathcal{Z}} \mathcal{W}_2^2(\mathcal{L}(X'_z), \mathcal{L}(\tilde{X'})) d\lambda(z) \leq \int_{\mathcal{Z}} \mathcal{W}_2^2(\mathcal{L}(X'_z), \mathcal{L}(\bar{X'})) d\lambda(z)
\end{equation}
which contradicts the uniqueness of $\bar{X'}$.
\end{proof}

The above lemma shows that $\var(\tilde{X}) \leq \var(X^{\dag})$ for all admissible $\tilde{X} \in \mathcal{D}|_{\mathcal{X}}$ satisfies $m_{\tilde{X}},  \Sigma_{\tilde{X}} \perp Z$, which together with the first assumption imply $\sigma(\tilde{X}) \subset \sigma(\bar{X})$ in practice. Therefore, from now on, we fix the choice of $\tilde{X}$ to be $X^{\dag}$ and prove the general characterization result based on the two assumptions listed above.

It remains to justify the choice of $Y^{\dag}$. To do so, we need the following lemma, which provides a multi-marginal characterization of the optimal affine map.

\begin{lem}[Projection Lemma for conditional expectations]\label{l:Projection Lemma for Conditional Expectations}
Given $m_{Y_z|X^{\dag}_z} = 0$ and $\Sigma_{Y_z|X^{\dag}_z} \succ 0$ $\lambda$-a.e., for any $\Sigma \succ 0$,

\begin{equation}
\inf_{\mathbb{E}(\tilde{Y}|X^{\dag}): \Sigma_{\tilde{Y}_z|X^{\dag}_z} = \Sigma} \int_{\mathcal{Z}} \mathcal{W}^2_2(\mathcal{L}(\mathbb{E}(Y_z|X^{\dag}_z)),  \mathcal{L}(\mathbb{E}(\tilde{Y}_z|X^{\dag}_z))) d\lambda(z)
\end{equation}
admits a unique solution, denoted by $Y^{\dag}_{\Sigma}$, that has the form 

\begin{equation}
Y^{\dag}_{\Sigma} := T_{\Sigma}(Y,Z)
\end{equation}
where $T_{\Sigma}(\cdot,z) := \Sigma_{\tilde{Y}_z|X^{\dag}_z}^{-\frac{1}{2}} (\Sigma_{\tilde{Y}_z|X^{\dag}_z}^{\frac{1}{2}} \Sigma  \Sigma_{\tilde{Y}_z|X^{\dag}_z}^{\frac{1}{2}})^{\frac{1}{2}} \Sigma_{\tilde{Y}_z|X^{\dag}_z}^{-\frac{1}{2}}$

\end{lem}

\begin{proof}
This is a direct corollary from Lemma \ref{l:Projection Lemma}.
\end{proof}

Finally, we are ready to prove the justification of the pseudo-barycenter in the case of general distributions.

\begin{thm}[Justification of $(X^{\dag}, Y^{\dag})$ in general distribution case]\label{th:Justification of Pseudo-Barycenter in General Distribution Case}
$\mathbb{E}(Y^{\dag} |X^{\dag})$ is a solution to

\begin{equation}
\inf_{(\tilde{X},\tilde{Y}) \in \mathcal{D}}  \{ ||Y - \mathbb{E}(\tilde{Y}|\tilde{X})||^2_2: m_{\tilde{X}}, m_{\tilde{Y}|\tilde{X}}, \Sigma_{\tilde{X}}, \Sigma_{\tilde{Y}|\tilde{X}} \perp Z\}
\end{equation}
under the assumptions: (1) set inclusion forms an order between $X^{\dag}$ and all $\tilde{X} \in \{\tilde{X} \in \mathcal{D}|_{\mathcal{X}} : m_{\tilde{X}}, \Sigma_{\tilde{X}} \perp Z\}$; and (2) $\Sigma_{Y_z|X^{\dag}_z} = \Sigma_{Y_z X^{\dag}_z} \Sigma_{X^{\dag}_z}^{-1} \Sigma_{Y_z X^{\dag}_z}^T$.

\end{thm}

\begin{proof}
The choice of $X^{\dag}$ follows from the first assumption and Lemma \ref{l:Pseudo-barycenter Contains the Largest Variance}. It remains to show that $Y^{\dag}$ is a solution to

\begin{equation}
\inf_{\tilde{Y} \in \mathcal{D}|_{\mathcal{Y}}}  \{ ||Y - \mathbb{E}(\tilde{Y}|X^{\dag})||^2_2: m_{\tilde{Y}|X^{\dag}}, \Sigma_{\tilde{Y}|X^{\dag}} \perp Z\}
\end{equation}
Fix $\Sigma \succ 0$ arbitrary, we have

\begin{equation}
||Y - \mathbb{E}(Y^{\dag}_{\Sigma}|X^{\dag})||^2_2 - ||Y - \mathbb{E}(Y|X^{\dag})||^2_2 = \int_{\mathcal{Z}} ||\mathbb{E}(Y_z - Y^{\dag}_{\Sigma,z} | X^{\dag}_z)||^2_2 d\lambda(z)
\end{equation}
and it follows from Lemma $\ref{l:Projection Lemma for Conditional Expectations}$ that

\begin{align*}
\int_{\mathcal{Z}} ||\mathbb{E}(Y_z - Y^{\dag}_{\Sigma,z} | X^{\dag}_z)||^2_2 d\lambda(z) = & \int_{\mathcal{Z}} \mathcal{W}_2^2(\mathcal{L}(\mathbb{E}(Y_z|X^{\dag}_z)),\mathcal{L}(T_{\Sigma}(\mathbb{E}(Y_z|X^{\dag}_z),z)) d\lambda(z)\\
= & \min_{\nu: \Sigma_{\nu_z} = \Sigma} \int_{\mathcal{Z}} \mathcal{W}_2^2(\mathcal{L}(\mathbb{E}(Y_z|X^{\dag}_z)),\nu_z) d\lambda(z)
\end{align*}
Therefore, $\eqref{eq:relaxation of the pre-processing characterization}$ boils down to the following:
\begin{equation}\label{optimal sigma}
\inf_{\Sigma \succ 0} \{ \int_{\mathcal{Z}} ||\mathbb{E}(Y_z - Y^{\dag}_{\Sigma,z} | X^{\dag}_z)||^2_2 d\lambda(z) \}.
\end{equation}
Finally, notice that 
\begin{align*}
& \int_{\mathcal{Z}} ||\mathbb{E}(Y_z - Y^{\dag}_{\Sigma,z} | X^{\dag}_z)||^2_2 d\lambda(z)\\
= & \int_{\mathcal{Z}} ||\mathbb{E}(Y_z| X^{\dag}_z) -T_{\Sigma}(\mathbb{E}(Y_z| X^{\dag}_z),z) ||^2_2 d\lambda(z)\\
= & \int_{\mathcal{Z}} ||\mathbb{E}(Y_z| X^{\dag}_z)||^2_2  + ||T_{\Sigma}(\mathbb{E}(Y_z| X^{\dag}_z),z) ||^2_2 - 2 \langle \mathbb{E}(Y_z| X^{\dag}_z),  T_{\Sigma}(\mathbb{E}(Y_z| X^{\dag}_z),z)\rangle_2 d\lambda(z)\\
= & \int_{\mathcal{Z}} \trace(\Sigma_{Y_z|X^{\dag}_z})  + \trace(\Sigma) - 2 \mathbb{E}(\mathbb{E}(Y_z| X^{\dag}_z)^T T_{\Sigma}(\mathbb{E}(Y_z| X^{\dag}_z),z))  d\lambda(z)\\
= & \int_{\mathcal{Z}} \trace(\Sigma_{Y_z|X^{\dag}_z})  + \trace(\Sigma) - 2 \langle T_{\Sigma},  \Sigma_{Y_z|X^{\dag}_z} \rangle_F  d\lambda(z)\\
= & \int_{\mathcal{Z}} ||\mathbb{E}(Y_z| X^{\dag}_z)' -T_{\Sigma}(\mathbb{E}(Y_z| X^{\dag}_z)',z) ||^2_2 d\lambda(z)
\end{align*}
where $\langle \cdot,\cdot \rangle_F$ denotes the Frobenius inner product and $X' \sim \mathcal{N}(m_X,\Sigma_X)$ denotes the Gaussian analog of $X$. It follows from the definition of $Y^{\dag}$ and Lemma $\ref{l:Location-scale Barycenter}$ that $\int_{\mathcal{Z}} ||\mathbb{E}(Y_z - Y^{\dag}_z | X^{\dag})||^2_2 d\lambda(z)$ is the lower bound of $\eqref{optimal sigma}$. The proof is complete.
\end{proof}

To conclude, given an arbitrary $L^2$-objective supervised learning model that aims to estimate conditional expectation, the training via $(X^{\dag}, Y^{\dag})$ results in an estimate of $\overline{\mathbb{E}(Y| \bar{X},Z)}$. In other words, any supervised learning model trained via $(X^{\dag}, Y^{\dag})$ is guaranteed to be independent of $Z$ in the location-scale family marginal case (or, to have first two moments independent of $Z$ in the general marginal case), while resulting in the minimum prediction error among all the admissible functions of some specific model due to the training step. Here, the assumption is that the test sample distribution is the same as the training sample distribution, which is a ubiquitous assumption for machine learning.

\subsection{Optimal Fair Data Representation at the Pareto Frontier} \label{s:Optimal Fair Data Representation at the Pareto Frontier}

Finally, we extend the pseudo-barycenter pair, which is the solution to the optimal fair data representation, to the fair data representation at the Pareto frontier using McCann interpolation via a similar approach as we derived the post-processing Pareto frontier in Section \ref{s:Pareto Frontier}. But notice a direct application of Theorem \ref{th:Geodesics Characterization of the Pareto Frontier} does not work here because there is no direct interpolation between $E(Y|X,Z)$ and $\overline{\mathbb{E}(Y|\bar{X},Z)}$ due to the change of the underlying sigma-algebra. Therefore, we apply a diagonal argument, Remark \ref{r:Diagonal Estimate of the Post-processing Pareto Frontier}, to estimate the interpolation between $E(Y|X,Z)$ and $\overline{\mathbb{E}(Y|\bar{X},Z)}$ and thus the fair data representation at the Pareto frontier.

To start, we derive the following post-processing optimal trade-off result directly from Theorem \ref{th:Geodesics Characterization of the Pareto Frontier} for a fixed choice of $\tilde{X} \in \{ \tilde{X} \in \mathcal{D}|_{\mathcal{X}}: \tilde{X} \perp Z\}$. For any $f: \mathcal{X} \times \mathcal{Z} \rightarrow \mathcal{Y}$, define $L_{y|\tilde{X}}$, $D_{y|\tilde{X}}$, and $V_{y|\tilde{X}}$ as follows:
\begin{equation}
L_{y|\tilde{X}}(f(\tilde{X},Z)) :=  (\int_{\mathcal{Z}} ||\mathbb{E}(Y_z|\tilde{X}) - f(\tilde{X},Z)_z||_2^2 d\lambda(z))^{\frac{1}{2}}
\end{equation}
\begin{equation} 
D_{y|\tilde{X}}(f(\tilde{X},Z)):= (\int_{\mathcal{Z}^2} \mathcal{W}_2^2(f(\tilde{X},Z)_{z_1} ,f(\tilde{X},Z)_{z_2}) d\lambda({z_1})d\lambda({z_2}))^{\frac{1}{2}}.
\end{equation}
To simplify notation, for any $T': \mathcal{Y} \times \mathcal{Z} \rightarrow \mathcal{Y}$, we also define the following:
\begin{equation}
L_{y|\tilde{X}}(T') :=  (\int_{\mathcal{Z}} ||\mathbb{E}(Y_z|\tilde{X}) - T'_z(\mathbb{E}(Y_z|\tilde{X}) )||_2^2 d\lambda(z))^{\frac{1}{2}}
\end{equation}
\begin{equation} 
D_{y|\tilde{X}}(T'):= (\int_{\mathcal{Z}^2} \mathcal{W}_2^2((T'_{z_1})_{\sharp}\mathcal{L}(\mathbb{E}(Y_{z_1}|\tilde{X})),(T'_{z_2})_{\sharp}\mathcal{L}(\mathbb{E}(Y_{z_2}|\tilde{X}) d\lambda({z_1})d\lambda({z_2}))^{\frac{1}{2}}.
\end{equation}
Also, let $T$ denote the optimal transport map from $\{\mathbb{E}(Y_z|\tilde{X})\}_z$ to the barycenter $\overline{\mathbb{E}(Y|\tilde{X},Z)}$, let $T(t), t \in [0,1]$ be the McCann interpolation, and define
\begin{align}
V_{y|\tilde{X}} := L_{y|\tilde{X}}(T) & =( \int_{\mathcal{Z}} ||\mathbb{E}(Y_z|\tilde{X}) - T_z(\mathbb{E}(Y_z|\tilde{X}) )||_2^2 d\lambda(z))^{\frac{1}{2}}\\
& = ( \int_{\mathcal{Z}} ||\mathbb{E}(Y_z|\tilde{X}) - \overline{\mathbb{E}(Y|\tilde{X},Z)}||_2^2 d\lambda(z) )^{\frac{1}{2}}.
\end{align} 
Then the result below follows directly similar to the proof of Theorem~\ref{th:Geodesics Characterization of the Pareto Frontier}.

\begin{coro}[Pareto frontier for conditional expectation on fixed sigma-algebra] \label{corr:pre pareto characterization}
Given $L_{y|\tilde{X}}$, $D_{y|\tilde{X}}$, and $V_{y|\tilde{X}}$ defined above, we have

\begin{equation}
V_{y|\tilde{X}} \leq L_{y|\tilde{X}}(f(\tilde{X},Z)) + \frac{1}{\sqrt{2}}D_{y|\tilde{X}}(f(\tilde{X},Z))
\end{equation}
where equality holds if and only if $f(\tilde{X},z) = T(t)(\mathbb{E}(Y_z|\tilde{X}),z)$ $\lambda$-a.e. for $t \in [0,1]$ as
\begin{equation}
L_{y|\tilde{X}}(T(t)) = tL_{y|\tilde{X}}(T(0)) = tV_{y|\tilde{X}},
\end{equation}
\begin{equation}
\frac{1}{\sqrt{2}}D_{y|\tilde{X}}(T(t)) = \frac{1}{\sqrt{2}}(1-t)D_{y|\tilde{X}}(T(0)) = (1-t)V_{y|\tilde{X}}.
\end{equation}

\end{coro}

The above result shows that by fixing $\tilde{X} \in \{ \tilde{X} \in \mathcal{D}|_{\mathcal{X}}: \tilde{X} \perp Z\}$, the McCann interpolation between $Id$ and $T_{y|\tilde{X}}$ yields the Pareto frontier from $\mathbb{E}(Y|\tilde{X},Z)$ to $\overline{\mathbb{E}(Y|\tilde{X},Z)}$, which is a weak version of the true frontier from $\mathbb{E}(Y|X,Z)$ to $\overline{\mathbb{E}(Y|\tilde{X},Z)}$. The only difficulty remaining is to coarsen the underlying sigma-algebra from $\sigma(X,Z)$ to $\sigma(\bar{X})$. But by Remark \ref{r:Finest Sigma Algebra vs. Most Variance}, we know that one can coarsen the sigma-algebra by reducing the variance. Therefore, we apply a diagonal argument to estimate the McCann interpolation between $(X,Y)$ and $(\bar{X},\bar{Y})$.

\begin{rema}[Diagonal estimate of the post-processing Pareto frontier] \label{r:Diagonal Estimate of the Post-processing Pareto Frontier}
The key observation is that the optimal affine transport map that pushes $(X,Y)$ forward to $(X^{\dag},Y^{\dag})$ is the pair $(T_x,T_{y|\bar{X}})$. Therefore, McCann interpolation between $\text{Id}$ and $T_x$ can optimally reduce variance and thereby coarsen $\sigma((X,Z))$ to $\sigma(X^{\dag})$, whereas the interpolation betwen $\text{Id}$ and $T_{y|\bar{X}}$ forms an estimation of the geodesic path between $Y$ and $Y^{\dag}$. Therefore, the present work matches the two interpolations diagonally $$(T_x(t), T_{y|\bar{X}}(t)) := ((1-t)Id_x + tT_x, (1-t)Id_y + tT_{y|\bar{X}}),$$ to estimate the true optimal fair data representation at the Pareto frontier.
\end{rema}

Finally, since $X^{\dag}$ and $\mathbb{E}(Y^{\dag}|X^{\dag})$ are the estimation of $\bar{X}$ and $\overline{\mathbb{E}(Y|\bar{X},Z)}$, respectively, as shown in the last section, it follows from Corollary \ref{corr:pre pareto characterization} and Remark \ref{r:Diagonal Estimate of the Post-processing Pareto Frontier} that 
\begin{equation}
\mathbb{E}(T_{y|\bar{X}}(t)(Y,Z)|T_x(t)(X,Z)), t \in [0,1]
\end{equation}
provides a pre-processing estimate of the Pareto frontier from $\mathbb{E}(Y|X,Z)$ to $\overline{\mathbb{E}(Y|\bar{X},Z)}$ that is characterized by Theorem \ref{th:Geodesics Characterization of the Pareto Frontier}.

\section{Algorithm Design}\label{s:Algorithm}

In this section, we propose two algorithms based on the theoretical results above. Algorithm~\ref{a:dependent} is designed for the fair learning outcome in the post-processing approach and for the dependent variable in fair data representation, whereas Algorithm~\ref{a:independent} is designed for the independent variable in fair data representation.
\begin{itemize}
\setlength{\parsep}{-0.2ex}
\setlength{\itemsep}{-0.2ex}
\item[1.] For practitioners who want to generate fair learning outcomes along the Pareto frontier, Algorithm \ref{a:dependent} takes the learning outcomes marginals $\{f(X,Z)_z\}_z$ as input and outputs the learning outcomes at (the optimal affine estimation of) the post-processing estimation of the Pareto frontier: $\{f(X,Z)(t)\}_{t \in [0,1]}$, which is the Wasserstein geodesic paths from the original learning outcome, $f(X,Z)(0)$, to the estimate of the optimal fair learning outcome, $f(X,Z)(1)$. Here, $f(X,Z)(1)$ is the best estimate of the optima fair learning outcome based on the provided learning outcome $\{f(X,Z)_z\}_z$.
\item[2.] For practitioners who want to generate a fair data representation, Algorithm \ref{a:independent} and Algorithm \ref{a:dependent} take in respectively the marginal independent and dependent data: $\{X_z\}_z$ and $\{Y_z\}_z$, then outputs respectively the independent and dependent data representations along the Wasserstein geodesics from the marginals to their pseudo-barycenter: $\{(X^{\dag}(t),Y^{\dag}(t))\}_{t \in [0,1]}$.  So that any conditional expectation estimation supervised learning model trained via $\{(X^{\dag}(t),Y^{\dag}(t))\}_{t \in [0,1]}$ results in (an diagonal affine estimation of) the learning outcome at the Pareto frontier.
\end{itemize}
%Furthermore, for independent and dependent variables, respectively,  we apply Algorithm \ref{a:independent} to test K-means on the pseudo-barycenter and obtains positive numerical results in data matching, which provides an intuitive explanation of how barycenter is a natural solution to fairness problems.  For more details, see Section \ref{s:Numerics}.}

\begin{algorithm}
\SetAlgoLined
\caption{Pseudo-Barycenter Geodesics for Independent Variable}
\label{a:independent}

{\bf Input:} marginal data sets $\{X_z\}_z$, stop criterion $\epsilon$;\\

{\bf Step 1:} Find the optimal barycenter covariance\;

Initialization: $\delta = \infty$, $\Sigma = rand$ or $Id$

 \While{$\delta > \epsilon$}{
  $\Sigma_{new} = \frac{1}{|X|} \sum_z |X_z| (\Sigma^{\frac{1}{2}} \Sigma_{X_z} \Sigma^{\frac{1}{2}})^{\frac{1}{2}}$; \hfill \tcp{\eqref{eq:iterative method}}\
  $\delta = ||\Sigma - \Sigma_{new}||_F$\;
  $\Sigma = \Sigma_{new}$\;
 }
 
{\bf Step 2:} Find the optimal affine transport maps\;
$T_z = \Sigma_{X_z}^{-\frac{1}{2}} (\Sigma_{X_z}^{\frac{1}{2}} \Sigma \Sigma_{X_z}^{\frac{1}{2}})^{\frac{1}{2}} \Sigma_{X_z}^{-\frac{1}{2}}$; \hfill \tcp{\eqref{eq:independent pseudo-barycenter affine map}}
{\bf Step 3:} Find the geodesic path to independent pseudo-barycenter\;
$X^{\dag}_z(t) = T_z(t) (X_z - m_{X_z}) + m_X$; \hfill \tcp{\eqref{eq:independent pseudo-barycenter }}
where $T_z(t) := (1-t)Id + tT_z$, $t \in [0,1]$; \hfill \tcp{\eqref{eq:mccann interpolation}}

{\bf Step 4 (optional):} For binary rows $X_{i \in I}$, reshape $(X^{\dag}(t))_i$ to binary by randomized rounding for all $i \in I$\;
For all $X_{i}$ binary: $p(t) = \frac{(X^{\dag}_z(t))_i}{\max((X^{\dag}_z(t))_i) - \min((X^{\dag}_z(t))_i)}$, $(X^{\dag}_z(t))_i \sim$ Bernoulli$(p(t))$\;
{\bf Step 5 (optional):} If sensitive information needs to be attached, merge the marginals back with mitigating $Z$\;
$X_z^{\dag}(t) = (X_z(t),z(t))$ where $z(t) = (1-t)(z - m_Z) + m_Z$, $t \in [0,1]$

{\bf Output:} $\{\{X_z^{\dag}(t)\}_{z \in \mathcal{Z}}\}_{t \in [0,1]}$

\end{algorithm}

\begin{algorithm}
\SetAlgoLined
\caption{Dependent (or Post-processing) Pseudo-Barycenter Geodesics}
\label{a:dependent}

{\bf Input:} marginal data sets $\{Y_z\}_z$ (post-processing: $\{f(X,Z)_z\}_z$), stop criterion $\epsilon$; \\

{\bf Step 1:} Find the optimal barycenter covariance\;

Initialization: $\delta = \infty$, $\Sigma = rand$ or $Id$

 \While{$\delta > \epsilon$}{
  $\Sigma_{new} = \frac{1}{|Y|} \sum_z |Y_z| (\Sigma^{\frac{1}{2}} \Sigma_{Y_z| X^{\dag}_z} \Sigma^{\frac{1}{2}})^{\frac{1}{2}}$ \hfill \tcp{\eqref{eq:iterative method}}
  (post-processing: $\Sigma_{new} = \frac{1}{|Y|} \sum_z |f(X,Z)_z| (\Sigma^{\frac{1}{2}} \Sigma_{f(X,Z)_z} \Sigma^{\frac{1}{2}})^{\frac{1}{2}}$)\;
  $\delta = ||\Sigma - \Sigma_{new}||_F$\;
  $\Sigma = \Sigma_{new}$\;
 }
 
{\bf Step 2:} Find the optimal affine transport maps\;
$T_z = \Sigma_{Y_z|X^{\dag}_z}^{-\frac{1}{2}} (\Sigma_{Y_z|X^{\dag}_z}^{\frac{1}{2}} \Sigma \Sigma_{Y_z|X^{\dag}_z}^{\frac{1}{2}})^{\frac{1}{2}} \Sigma_{Y_z|X^{\dag}_z}^{-\frac{1}{2}}$ \hfill \tcp{\eqref{eq:dependent pseudo-barycenter affine map}}
(post-processing: $T_z = \Sigma_{f(X,Z)_z}^{-\frac{1}{2}} (\Sigma_{f(X,Z)_z}^{\frac{1}{2}} \Sigma \Sigma_{f(X,Z)_z}^{\frac{1}{2}})^{\frac{1}{2}} \Sigma_{f(X,Z)_z}^{-\frac{1}{2}}$); \hfill \tcp{\eqref{eq:post-processing maps}}\
{\bf Step 3:} Find the geodesic path to dependent pseudo-barycenter\;
$Y^{\dag}_z(t) = T_z(t) (Y_z - m_{Y_z}) + m_Y$ \hfill \tcp{\eqref{eq:dependent pseudo-barycenter }}
where $T_z(t) := (1-t)Id + tT_z, t \in [0,1]$ \hfill \tcp{\eqref{eq:mccann interpolation}}
(post-processing: $f(X,Z)_z(t) = T_z(t) (f(X,Z)_z - m_{f(X,Z)_z}) + m_{f(X,Z)}$); \hfill \tcp{\eqref{eq:post-processing pseudo-barycenter}}
{\bf Step 4 (optional):} For binary rows $Y_{i \in I}$ (post-processing: $(f(X,Z))_{i \in I}$), reshape $(Y^{\dag}(t))_i$ (post-processing: $(f(X,Z)(t))_{i \in I}$) to binary by randomized rounding for all $i \in I$\;
For all $Y_{i}$ binary: $p(t) = \frac{(Y^{\dag}_z(t))_i}{\max((Y^{\dag}_z(t))_i) - \min((Y^{\dag}_z(t))_i)}$, $(Y^{\dag}_z(t))_i \sim$ Bernoulli$(p(t))$\;

{\bf Output:} $\{\{Y^{\dag}_z(t)\}_{z \in \mathcal{Z}}\}_{t \in [0,1]}$ (post-processing: $\{\{f(X,Z)_z(t)\}_{z \in \mathcal{Z}}\}_{t \in [0,1]}$)

\end{algorithm}
The choice of the Frobenius norm  in Step 1 is due to computational efficiency. Any matrix norm would work.

\begin{rema}[Solution to alternative fair data representation constraint] \label{r:Solution to Alternative Fair Data Representation Constraint}
In Section \ref{s:Post-processing and Pre-processing Approach}, the present work shows two alternative fair data representation constraints: (1) $(\tilde{X},\tilde{Y}) \perp Z$ and (2) $\tilde{X} \perp Z$, which offer different trade-offs between fairness protection and utility. If a practitioner applies the alternative constraint, the proposed algorithms can be applied to generate (the optimal affine estimation of) corresponding fair data representation as the following:
\begin{itemize}
\item[1] For $(\tilde{X},\tilde{Y}) \perp Z$, one applies Algorithm \ref{a:independent} to both $\{(X_z, Y_z)\}_z$. This alternative is especially useful when practitioners or data publishers do not know which features would be chosen as independent or dependent.
\item[2] For $\tilde{X} \perp Z$, one applies Algorithm \ref{a:independent} to $\{X_z\}_z$ and leaves $\{Y_z\}$ untouched.
\end{itemize}
\end{rema}

%\begin{coro}[Independence of Linear Regression Results] Let $(\bar{X},Y^{\dag})$ be generated by Algorithm $\ref{independent}$ and $\ref{dependent}$, $\hat{Y}^{\dag}$ be the estimation of linear regression model trained via $(\bar{X},Y^{\dag})$, then
%\begin{equation}
%\Sigma_{\hat{Y}^{\dag}} \perp Z
%\end{equation}
%\end{coro}

%\begin{proof}
%\begin{align*}
%(\hat{Y}^{\dag}_z)^T \hat{Y}^{\dag}_z & = (X^{\dag}_z \beta_z)^T X^{\dag}_z \beta_z\\
%& = (((X^{\dag}_z)^T X^{\dag}_z)^{-1} (X^{\dag}_z)^T Y^{\dag}_z)^T ((X^{\dag}_z)^T X^{\dag}_z) ((X^{\dag}_z)^T X^{\dag}_z)^{-1} (X^{\dag}_z)^T Y^{\dag}_z\\
%& = (Y^{\dag}_z)^T X^{\dag}_z ((X^{\dag}_z)^T X^{\dag}_z)^{-1} (X^{\dag}_z)^T Y^{\dag}_z\\
%& = \Sigma_{Y^{\dag}_z X^{\dag}_z} \Sigma_{X^{\dag}_z}^{-1} \Sigma_{Y^{\dag}_z X^{\dag}_z}^T\\
%& = \Sigma_{Y^{\dag}_z | X^{\dag}_z}
%\end{align*}
%It follows from the construction of $Y^{\dag}_z$ that $\Sigma_{Y^{\dag}_z | X^{\dag}_z} = \Sigma$ is the same for $\lambda$-a.e.  $z \in \mathcal{Z}$. We are done.
%\end{proof}

\section{Empirical Study: Fair Supervised Learning}\label{s:Numerics}

In this section, we present numerical experiments with the proposed Algorithms~\ref{a:independent} and~\ref{a:dependent} from Section~\ref{s:Algorithm}. The proposed fair data representation method is bench-marked against two baselines:

\begin{itemize}
\setlength{\parsep}{-0.2ex}
\setlength{\itemsep}{-0.2ex}
\item[1.] the prediction model trained via the original data (denoted by ``supervised learning name" in the experiment result figure below): supervised learning models trained via data including the sensitive variable provide an estimation of statistical disparity resulting from both disparate treatment and impact.
\item[2.] the prediction model trained via data excluding the sensitive variable (denoted by ``supervised learning name + Excluding Z"): supervised learning models trained via data excluding the sensitive variable provide an estimation of statistical disparity resulting from only disparate impact.
\end{itemize}

\subsection{Benchmark Data and Comparison Methods}

For comparison, we implement the following known methods for different types of supervised learning tests:

\begin{itemize}
\setlength{\parsep}{-0.2ex}
\setlength{\itemsep}{-0.2ex}
\item[1.] For classification test, the present work compares the current state-of-the-art pre-processing methods \cite{calmon2017optimized, zemel2013learning} (``supervised learning name + Calmon or Zemel", the later is also known as ``Learning Fair Representation") with the proposed fair data representation methods (``supervised learning name + pre-proc. Pareto frontier Est. or Pseudo-barycenter").
\item[2.] For uni-variate regression test, we compare the post-processing Wasserstein barycenter based fair regression \cite{chzhen2020fair} (``supervised learning name + Chzhen") with the proposed post-processing pseudo-barycenter methods (``supervised learning name + post-proc. Pareto frontier Est. or Pseudo-barycenter") and the fair data representation methods.
\item[3.] For multi-variate supervised learning test, we compare the post-processing pseudo-barycenter methods with the fair data representation methods.
\end{itemize}

The reasons for this choice are as follows: (1) the known attempts via the pre-processing approach are only available for fair classification; (2) the post-processing Wasserstein barycenter based methods on fair classification are analogous to the one on fair regression, which is shown to outperform other in-processing or post-processing methods in reducing discrimination while preserving accuracy; (3) there exists no practical attempt along the Wasserstein characterization approach to multi-dimensional supervised learning due to the computational complexity of finding the barycenter and the optimal transport maps. 

We adopt the following metrics of accuracy and discrimination that are frequently used in fair machine learning experiments on various data sets: (1) For fair classification, the prediction accuracy, and statistical disparity are quantified respectively by AUC (area under the Receiver Operator Characteristic curve) and

\begin{defi}[Classification discrimination] \label{classification disparity}
$$Discrimination = \max_{z,z' \in \mathcal{Z}}{\Big|\frac{\mathbb{P}(\hat{Y}_z = 1)}{\mathbb{P}(\hat{Y}_{z'} = 1)} - 1\Big|}$$
\end{defi}
as defined in \cite{calmon2017optimized}. (2) For univariate supervised learning, the prediction error and statistical disparity are quantified respectively by MSE (mean squared error, equivalent to the squared $L^2$ norm on sample probability space) and KS (Kolmogorov-Smirnov) distance as in \cite{chzhen2020fair} for indirect comparison purpose. So that readers can compare the proposed methods indirectly with other methods that are tested in \cite{calmon2017optimized,  chzhen2020fair, zemel2013learning} and their references. (3) For univariate and multivariate supervised learning, the prediction error and statistical disparity are quantified respectively by $L^2$ and $\mathcal{W}_2$ (Wasserstein) distances, which are the quantification the current work adopts to prove the Pareto frontier in the above sections.

In addition, we perform tests on four benchmark data sets: CRIME, LSAC, Adult, COMPAS, which are also frequently used in fair learning experiments. A brief summary is given below. For all the test results, we apply 5-fold cross-validation with $50\%$ training and $50\%$ testing split, except for $90\%$ training and $10\%$ testing split in the linear regression test on LSAC due to the high computational cost of the post-processing Wasserstein barycenter method \cite{chzhen2020fair}. Therefore, interested readers can also compare the pseudo-barycenter test results indirectly to other methods tested in \cite{calmon2017optimized,chzhen2020fair}.\\

\begin{table}[htbp]
    \centering
    %\begin{tabularx}{\textwidth}{| X | X | X |}
    \begin{tabularx}{\textwidth} { sbttt }
 \hline
 Data set & Tests & Data size & dim($X$) & dim($Y$)\\
 \hline
UCI Adult  & logit regression, random forest  & 162805  & 16 & 1\\
 \hline
COMPAS  & logit regression, random forest  & 26390 & 7 & 1\\
 \hline
LSAC  & linear regression, ANN & 20454  & 9 & 1\\
 \hline
CRIME  & linear regression, ANN & 1994 & 97 & 1\\
 \hline
 CRIME  & linear regression, ANN & 1994 & 87 & 11\\
 \hline
\end{tabularx}
  
\end{table}

\begin{itemize}
\setlength{\parsep}{-0.2ex}
\setlength{\itemsep}{-0.2ex}
\item Communities and Crime Data Set (CRIME) contains the social, economic, law executive, and judicial data of communities in the United States with 1994 examples \cite{redmond2002data}. The task of univariate learning is to predict the number of crimes per $10^5$ population using the rest of the information on the data set. Here, race is the sensitive information and, for (indirect) comparison purposes, we made race a binary categorical variable of whether the percentage of the African American population (racepctblack) is greater than $30\%$.

In multivariate supervised learning on CRIME, we keep the same sensitive variable. But the learning task is to predict the following vector that represents the local housing and rental market information: (low quartile occupied home value, median home value, high quartile home value, low quartile rent, median rent,  high quartile rent, median gross rent, number of immigrants, median number of bedrooms, number of vacant households, number of crimes).

\item LSAC National Longitudinal Bar Passage Study data set (LSAC) contains social, economic, and personal data of law school students with 20454 examples \cite{wightman1998lsac}. The goal of univariate models is to predict the students' GPA using other information on the data set. Here, race is the sensitive variable and, for (indirect) comparison purposes, we make it a binary variable on whether the student is non-white. 

\item UCI Adult Data Set (Adult) contains the 1994 Census data with 162805 examples \cite{asuncion2007uci}. The goal is to predict the binary categorization (whether gross annual income is greater than 50k) using age, education years, and gender, where gender is the sensitive information.

\item Correctional Offender Management Profiling for Alternative Sanctions (COMPAS) is a benchmark set of data from Broward County, Florida for algorithmic bias studies \cite{angwin2022machine}. Following \cite{calmon2017optimized}, the goal here is to predict whether an individual would commit any violent crime while race is the sensitive binary variable (African-American and Caucasian).

\end{itemize}

\subsection{Numerical Result}

In this subsection, we summarize the experimental results\footnote{The code for the results of our experiments is available online at: \url{github.com/xushizhou/fair_data_representation}}.

The classification test result is summarized in Figure $\ref{classification comparison}$ below.  Here, the vertical and horizontal axes are AUC and Discrimination defined in Definition $\ref{classification disparity}$. That is, the more upper-left, the better the result. The first row of Figure $\ref{classification comparison}$ shows the results of logistic regression (left) and random forest (right) on Adult whereas the second shows the corresponding results on COMPAS.

\begin{figure}[H]
\centering
\includegraphics[width=\textwidth]{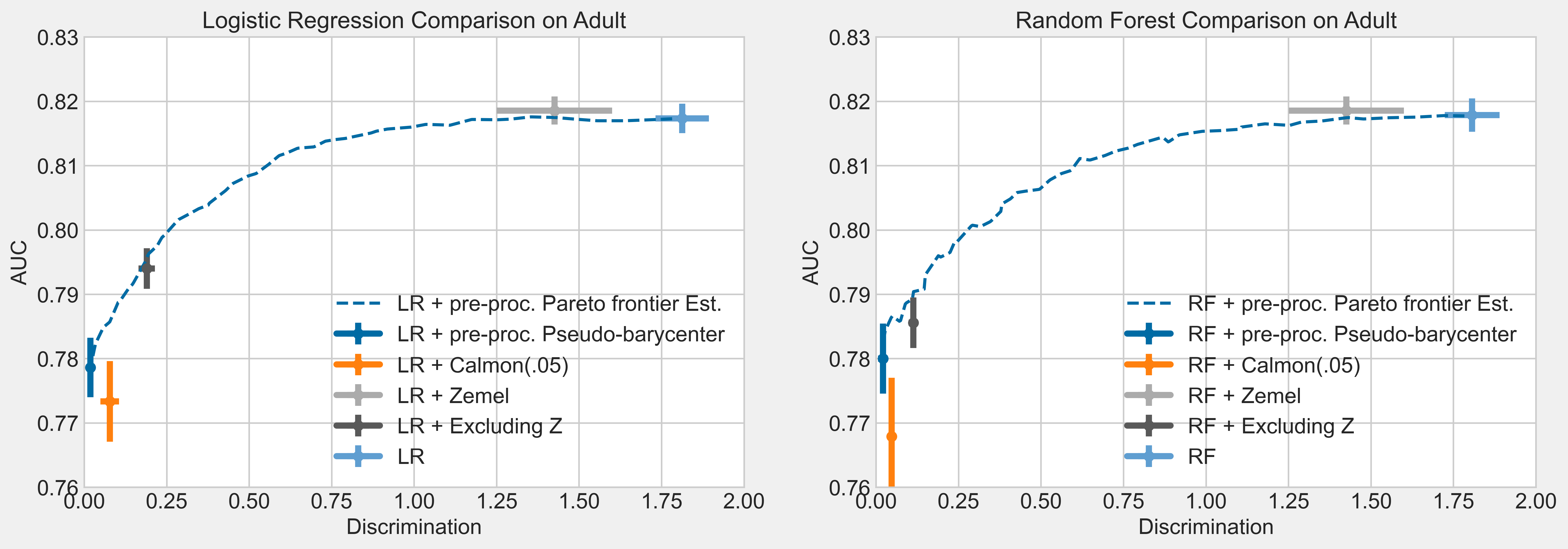}\hfill
\includegraphics[width=\textwidth]{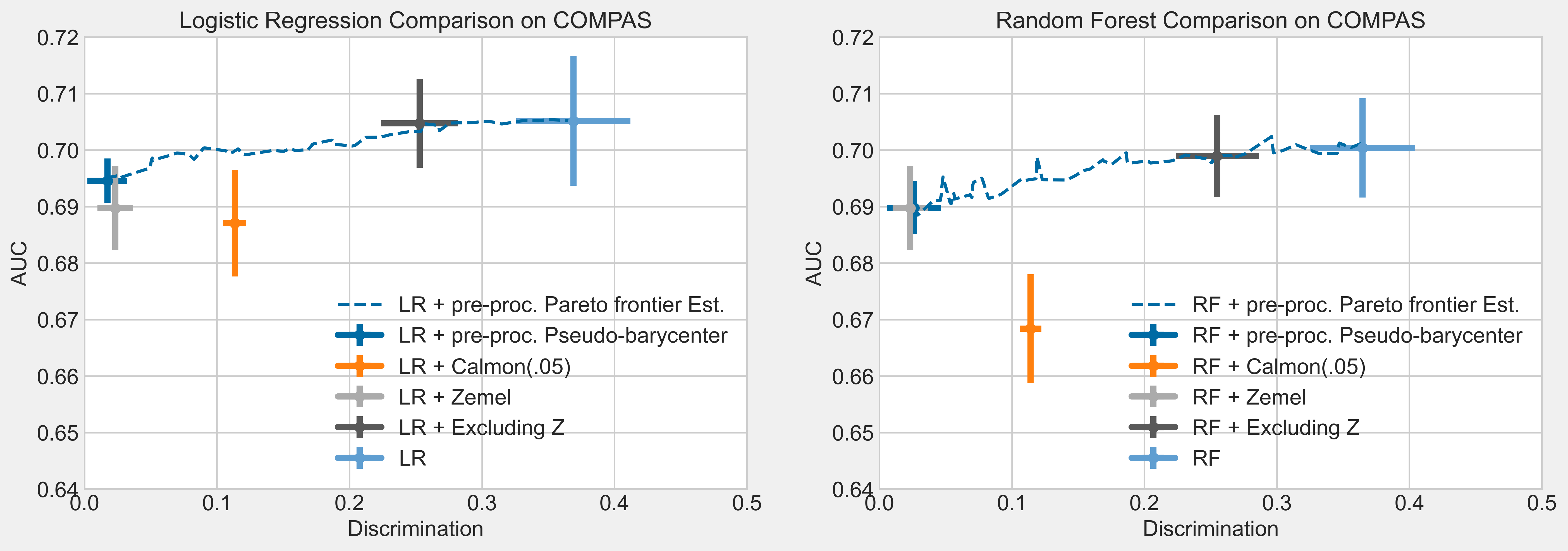}
\caption{As shown in the classification test above, the proposed fair data representation method (+ Pre-proc. Pareto frontier Est. or Pseudo-barycenter) outperforms the other methods (+ Zemel or + Calmon) in estimating the optimal fair learning outcome. It reduces the Discrimination metric to nearly zero while keeping the relatively high level of AUC with both logistic regression (LR) and random forest (RF) on both Adult and COMPAS. Furthermore, fair data representation method offers flexibility in choosing the desired trade-off while other methods only estimate a random point near the Pareto frontier.}
\label{classification comparison}
\end{figure}

We note that there exists a large disparate impact in the learning outcome on COMPAS due to the relatively small difference between the ``Discrimination" of learning outcome on the original data (LR and RF) and the outcome on the data excluding $Z$ (LR and RF + Excluding $Z$). Therefore, a further reduction of statistical disparity is needed. In contrast, the relatively large difference in the Adult data set implies a small disparate impact. That is, a simple exclusion of the sensitive variable $Z$ results in a significant improvement in fairness.

For further reduction of statistical disparity, it is clear from the experiment results on both COMPAS and Adult that the estimation via the Wasserstein geodesics to Pseudo-barycenter (LR and RF + Pseudo-barycenter) consistently outperforms LR and RF + Calmon by obtaining lower Discrimination with higher AUC.

In addition, although ``LR and RF + Zemel" achieves a point near the Pareto frontier estimated by the proposed Pseudo-barycenter methods, the point estimation is rather random. Hence, ``+ Zemel" is not consistent in estimating the optimal fair learning outcome (the end point of the Pareto curve). Practitioners cannot know which point on the Pareto frontier is estimated by ``+ Zemel". In comparison, the pseudo-barycenter methods are consistent in estimating the optimal fair learning outcome. In addition, they providef the entire Pareto frontier, and hence offer practitioners the flexibility to choose the desired trade-off. Moreover, the proposed method works for any model that aims to estimate conditional expectation, including classification and regression, while ``+ Zemel" only works for classification.

The univariate regression test result on the LSAC and the one on CRIME are shown respectively in Figure \ref{univariate comparison LSAC} and \ref{univariate comparison CRIME} below. Here, the vertical and horizontal axes in the first rows are MSE and KS distance. The corresponding axes in the second row are the $L^2$-quantified test error and the $\mathcal{W}_2$ distance that quantifies the remaining statistical disparity among sensitive groups. Therefore, the more lower-left, the better is the result in both rows. The two supervised learning methods we use are linear regression and artificial neural networks (ANN with 4 linearly stacked layers where each of the first three layers has 32 units all with ReLu activation while the last has 1 unit with linear activation).

\begin{figure}[H]
\centering
\includegraphics[width=\textwidth]{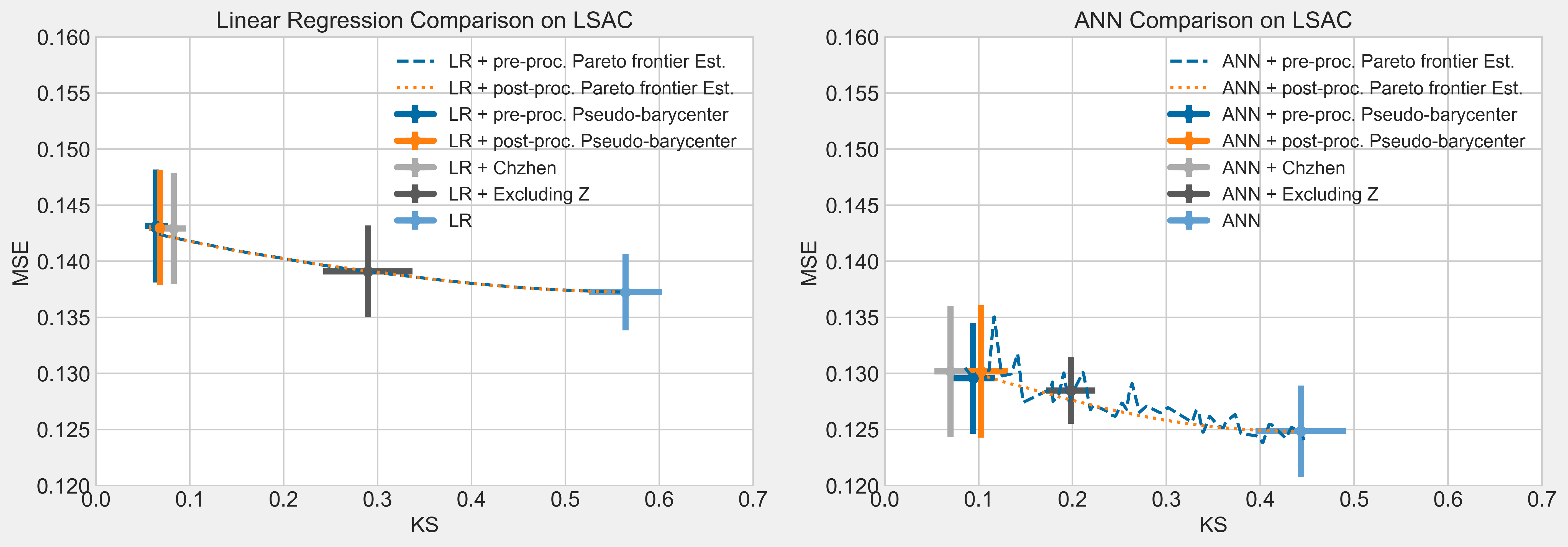}\hfill
\includegraphics[width=\textwidth]{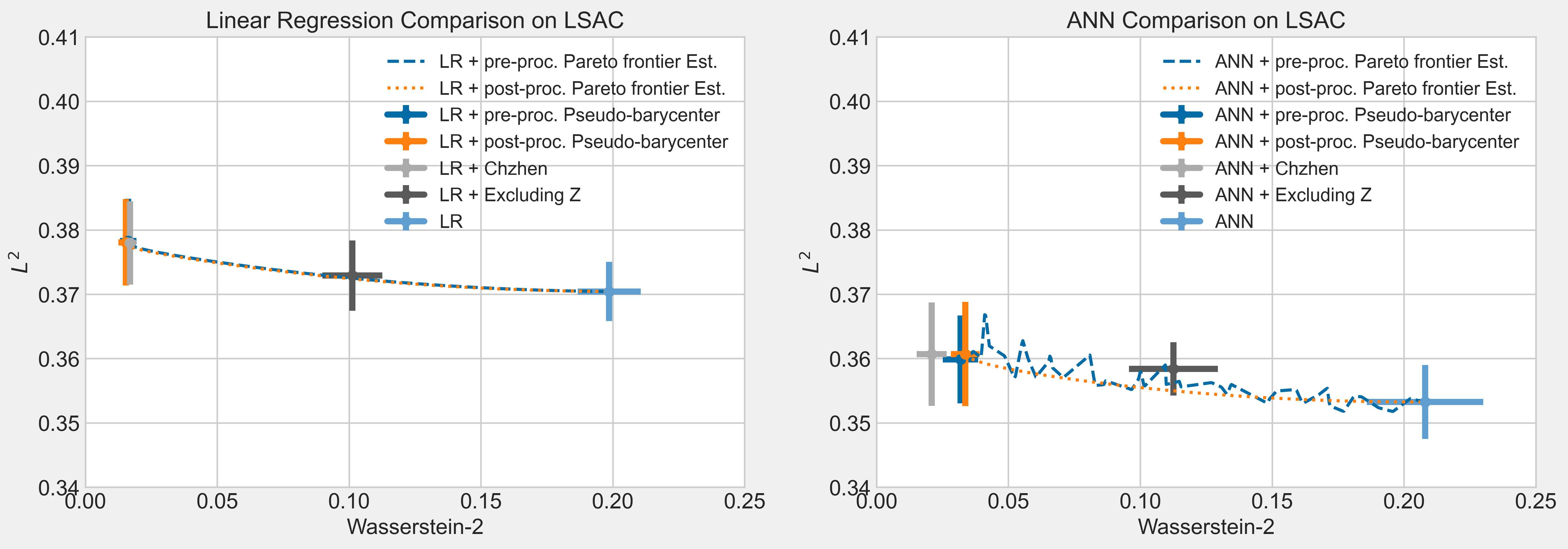}\hfill
\caption{As shown in the univariate regression test on LSAC above, the proposed fair data representation method (+ pre-proc. Pareto frontier Est. or Pseudo-barycenter) and the post-processing pseudo-barycenter geodesics method (+ post-proc. Pareto frontier Est. or Pseudo-barycenter) achieved similar performance as the exact barycenter method (+ Chzhen). The proposed methods outperformed ``+ Chzhen" with linear regression and were exceeded with the artificial neural network, both by a narrow margin. But the performance of the proposed methods is achieved at $0.0128 \%$ of the time costs ``+ Chzhen" (see Figure \ref{time table} below). In addition, the proposed methods offer the flexibility of choosing the desired (optimal) trade-off between utility loss (MSE or $L^2$-loss) and statistical disparity (KS or $\mathcal{W}_2$ distance), whereas ``+ Chzhen" only estimate the end point of the Pareto curve.}
\label{univariate comparison LSAC}
\end{figure}

In the regression tests, post-processing Pareto frontier estimation via ANN is smooth while the pre-processing estimation is not. Here, the smoothness is due to the McCann interpolation between the identity matrix and the optimal transport map in the post-processing approach. The non-smoothness is due to the randomness in training the neural network. When testing fair data representations via ANN, one has to train the neural network for the data representation at every time $t \in [50]$. Hence, the randomness in ANN training results in the non-smoothness in the Pareto frontier estimation via fair data representations.

On the LSAC data set, the proposed methods (+ pre-proc. Pseudo-barycenter and + post-proc. Pseudo-barycenter) obtains a similar performance as the post-processing exact Wasserstein barycenter method (+ Chzhen): the proposed methods outperformed the exact method in the linear regression test and were outperformed by the exact method in the non-linear artificial neural network tests, which is consistent with our theoretical results. But the performance of the proposed methods is achieved at $0.81$ seconds on average, whereas the average time cost of ``+ Chzhen" is $6365.98$ seconds (see Figure \ref{time table} below). In addition, we gained the flexibility in choosing the desired trade-off, computational efficiency, model selection, parameter tuning, and composition.

\begin{figure}[H]
\centering
\includegraphics[width=\textwidth]{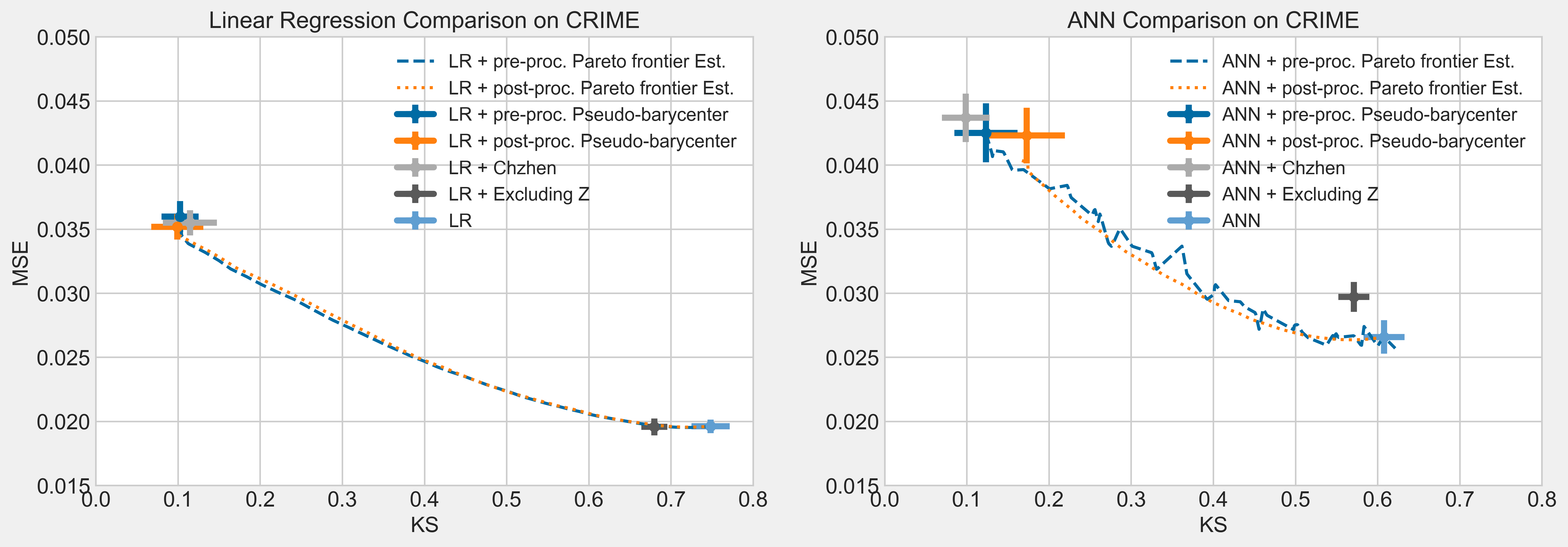}\hfill
\includegraphics[width=\textwidth]{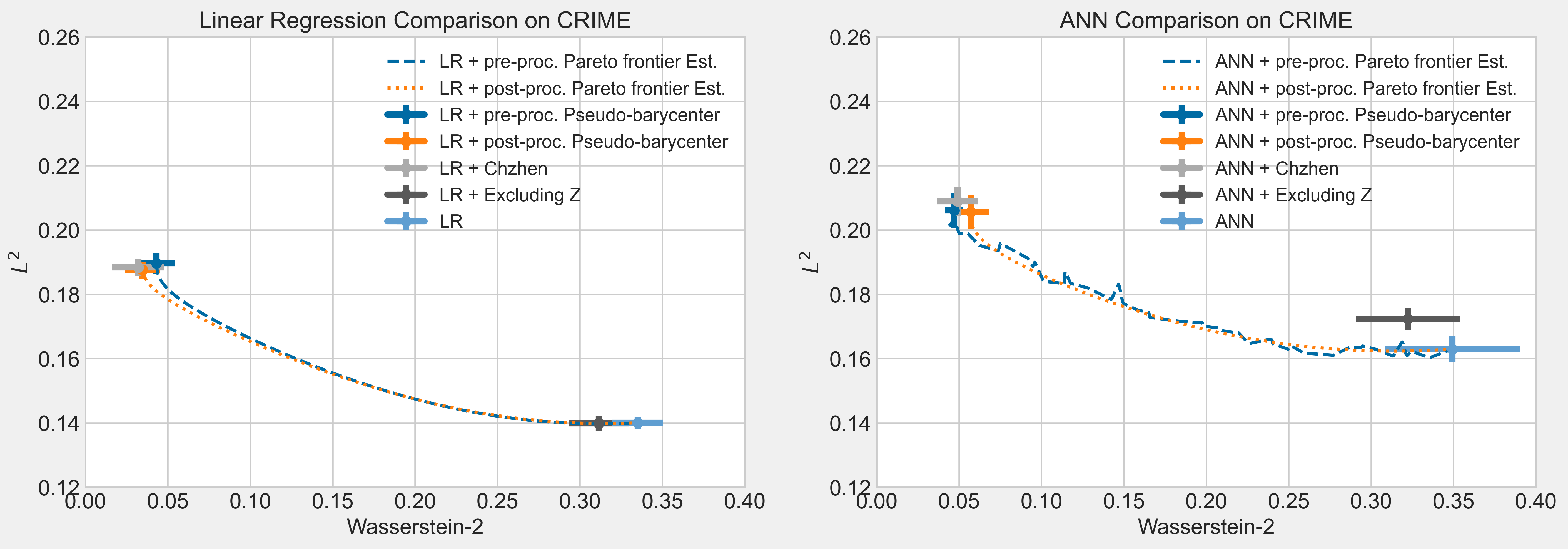}\hfill
\caption{As shown above, the fair data representation method ( + pre-proc. Pareto frontier Est. or Pseudo-barycenter) achieved the same, if not better, performance as the exact barycenter method (+ Chzhen) in estimating the optimal learning outcome. In addition, the fair data representations method offers flexibility in choosing a desired (optimal) trade-off between utility and fairness.}
\label{univariate comparison CRIME}
\end{figure}

For CRIME data, the small difference between the KS of learning outcome on the original data (LR and ANN) and the one on the data excluding the sensitive variable (LR and ANN + Excluding $Z$) implies a significant disparate impact. This observation and the multi-dimensional test below agree with the following statement in \cite{christian2020alignment}: ``Simply removing the `protected attribute' is insufficient. As long as the model takes in features that are correlated with, say, gender or race, avoiding explicitly mentioning it will do little good."

In Figure \ref{univariate comparison CRIME}, it is clear that the fair data representation methods (+ pre-proc. Pareto frontier Est. or Pseudo-barycenter) achieved the same, if not better, performance as the comparison method (+ Chzhen): the proposed method was outperformed by ``+ Chzhen" with linear regression and outperformed ``+ Chzhen" with artificial neural network, both by a narrow margin. But the performance of the fair data representation method is achieved at $4.735 \%$ of the time costs ``+ Chzhen." In addition, the fair data representation method provides (an estimation of) the entire Pareto frontier and works for multivariate supervised learning (see Figure \ref{multivariate comparison} below), whereas ``+ Chzhen" only estimates the end point of the Pareto frontier and only works in the univariate learning.

\begin{rema}
One possible explanation for the proposed method to outperform the exact post-processing Wasserstein barycenter method (``+ Chzhen") is the following: Although \cite{chzhen2020fair} is designed specifically for univariate learning and the KS distance by matching the sensitive marginal cumulative distribution functions, such matching on training data can lead to over-fitting. Therefore, the resulting optimal transport map fits the training data too well to be optimal for the test data.
\end{rema}

Next, we show the multivariate supervised learning on CRIME data to provide a high-dimensional baseline, to which later proposed machine learning fairness methods on high-dimensional data can compare. The vertical and horizontal axes are the $L^2$ test error and the $\mathcal{W}_2$ distance among sensitive groups. Hence, the more lower-left, the better the result.

\begin{figure}[H]
\centering
\includegraphics[width=\textwidth]{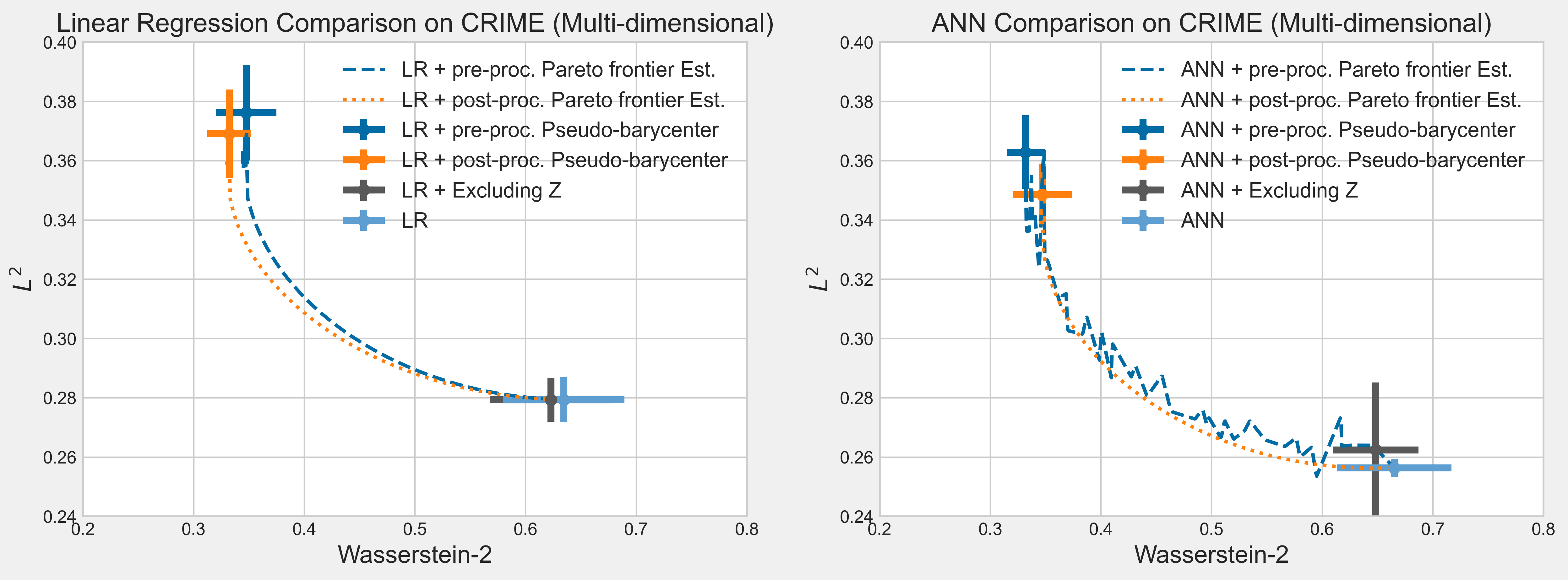}
\caption{As shown above, the fair data representation method (+ pre-proc. Pareto frontier Est. or Pseudo-barycenter) achieves similar performance to the post-processing pseudo-barycenter method (+ post-proc. Pareto frontier Est. or Pseudo-barycenter).}
\label{multivariate comparison}
\end{figure}

Due to the relatively high dimensionality of $X$ (87-dimensional) and $Y$ (11-dimensional), the probabilistic dependence and correlation between the learning outcome and the sensitive variable $Z$ becomes more difficult to remove. It is clear that (LR or ANN + Excluding Z) now removes almost none of the statistical disparity compared to the learning outcome on the original data.

To show the difference in practical computational cost among the comparison methods, we include the following processing time table, where the unit of time is second, and the simulations were run on a 2019 Macbook pro with Intel i9 processor.

\begin{figure}[H]
\centering
\includegraphics[width=\textwidth]{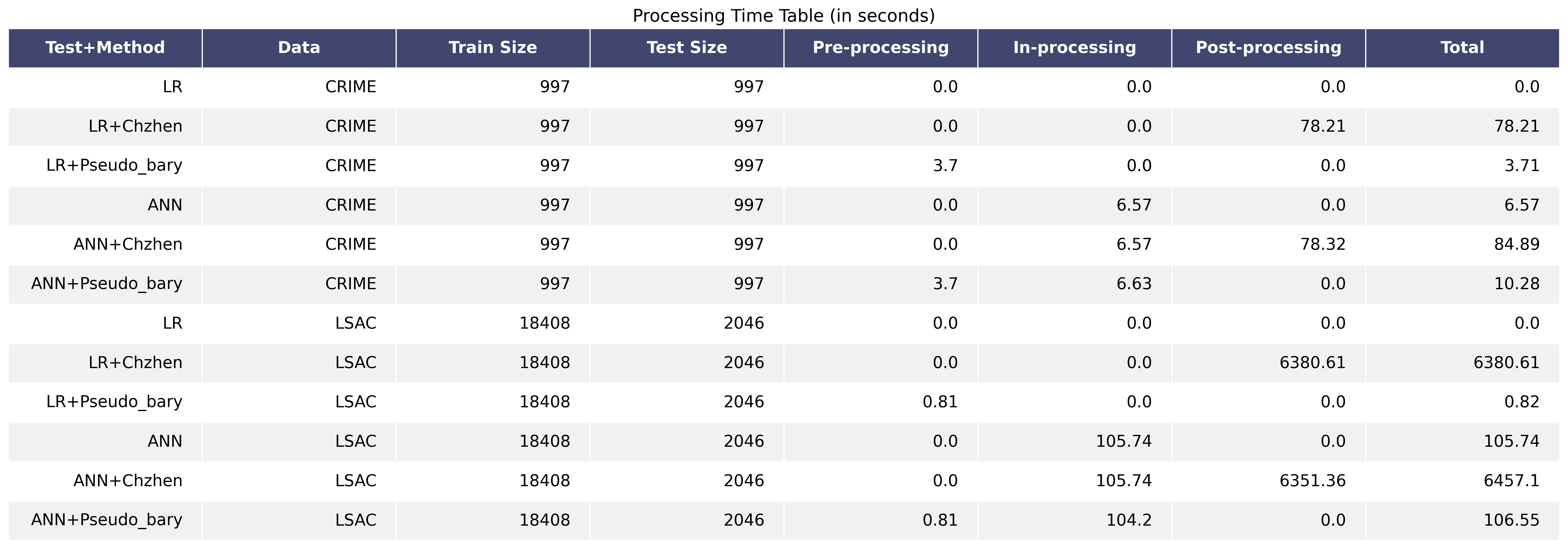}
\caption{As shown in the table above, the computational cost of the pseudo-barycenter method is significantly lower than the cost of the known post-processing methods: on average 7836 times faster on LSAC and 21 times faster on CRIME in a single train-test cycle for a single supervised learning model. Furthermore, in model selection or composition, the pre-processing time is a fixed one-time cost while the post-processing time is additive. (See point 4 below for a more detailed explanation)}
\label{time table}
\end{figure}

Now, we show the major advantages of the proposed method compared to the post-processing ones, such as \cite{chzhen2020fair, jiang2020wasserstein, gouic2020projection}:

\begin{enumerate}
\setlength{\parsep}{-0.2ex}
\setlength{\itemsep}{-0.2ex}
\item Flexibility in Trade-off: The pre-processing method provides an estimation for the entire Pareto frontier and thereby allows practitioners to balance between prediction error and disparity. In contrast, the known post-processing method merely estimates the starting (left) point of the frontier.

\item Sensitive data privacy protection: The geodesics to the pseudo-barycenter allow practitioners to suppress the sensitive information remaining in the data to the desired level. That is, given the resulting suppressed data, anyone who has leaked data from the training or decision stage can merely extract the level of sensitive information up to the pre-determined remaining level. For example, if one chooses to suppress as much sensitive information as possible by setting $t = 1$, then it follows from the construction of dependent and independent pseudobarycenter, it is guaranteed that any unsupervised learning method that uses only the first two moments of the sample data distribution, such as the K-means and PCA, would be unable to extract any information about $Z$ from $X^{\dag}$ or $f_{Y^{\dagger}}(X^{\dag})$.

\item Computational efficiency in high-dimensional learning: As summarized in Figure $\ref{time table}$, the computation of the pseudo-barycenter estimation of the optimal fair learning outcome is significantly faster than the computation of the exact barycenter via the post-processing matching cdf approach, especially on the LSAC data which has a larger sample size.

\item Flexibility in model selection, modification, and composition: in practice, one needs to repeat the training process multiple times to compare different supervised learning algorithms or parameters. The proposed fair data representation method has a fixed pre-processing time while the processing time of post-processing methods is additive. For example, if a practitioner needs to compare linear regression and ANN on LSAC as shown in Figure $\ref{time table}$ and repeat the training process $N$ times for parameter tuning or validation purpose, the total processing time for pseudo-barycenter method is $0.81 + N(0.0025 + 104.2)$ while the processing time for the post-processing method is $N(0.003 + 6380.61 + 105.738 + 6351.36)$.

%\item Shedding light on fairness in unsupervised learning: the pseudo-barycenter also allows unsupervised learning algorithms to result in diverse (with respect to $Z$) solutions. See the following section and figures for more details.
\end{enumerate}

\section*{Acknowledgement}

The authors want to thank the referees, whose profound and detailed feedback greatly enhanced the quality and clarity of this paper.
The authors acknowledge support from NSF DMS-2027248, NSF DMS-2208356, NSF CCF-1934568, NIH R01HL16351, and DE‐SC0023490.

%\newpage

\appendix
\section{Appendix: Proof of Results in Section 2} \label{A:Section 2 Appendix}

\subsection{Proof of Lemma \ref{l:Rigid translation property}}

\begin{proof}
\begin{align*}
\mathcal{W}_2^2(\mu,\nu) & = \int ||x - y||^2 d\gamma^*(x,y)\\
& = \int ||((x - m_{\mu}) - (y - m_{\nu})) + (m_{\mu} - m_{\nu})||^2 d\gamma^*(x,y)\\
& =  \int ||(x - m_{\mu}) - (y - m_{\nu})||^2 d\gamma^*(x,y) +  ||m_{\mu} - m_{\nu}||^2\\
& \geq \mathcal{W}_2^2(\mu',\nu') + ||m_{\mu} - m_{\nu}||^2\\
& = \int ||x - y||^2 d(\gamma')^*(x,y) +  ||m_{\mu} - m_{\nu}||^2\\
& = \int ||(x + m_{\mu}) - (y + m_{\nu})||^2 d(\gamma')^*(x,y)\\
& \geq \mathcal{W}_2^2(\mu,\nu)
\end{align*}
where $\gamma^*$ and $(\gamma')^*$ denote the optimal transport plan for $(\mu,\nu)$ and $(\mu',\nu')$, respectively. The first inequality results from the fact that $\gamma'(x,y) := \gamma^*(x - m_{\mu}, y - m_{\nu}) \in \prod(\mu',\nu')$, the second inequality from $\gamma(x,y) := (\gamma')^*(x + m_{\mu}, y + m_{\nu}) \in \prod(\mu,\nu)$, and the equalities from direct expansion.
\end{proof}

\subsection{Proof of Lemma \ref{l:Location-scale Barycenter}}

\begin{proof}
Existence and uniqueness follow directly from Theorem~\ref{Existence and Uniqueness of Barycenter}. For the equivalent multi-marginal coupling problem, there exists an optimal solution $\gamma^* = \mathcal{L}(\{X_z\}_z)$. It follows from Remark \ref{equivalence between couple and bary} that $\bar{X} = T(\{X_z\}_z)$ where $\mathcal{L}(\bar{X})$ is the Wasserstein barycenter. Therefore, the Gaussianity of barycenter results from linearity of $T$ in the finite $|\mathcal{Z}|$ case, and the fact that the set of Gaussian distribution is closed in $(\mathcal{P}_{2,ac}, \mathcal{W}_2)$ when $|\mathcal{Z}|$ is infinite. The characterization equation is proved in the case of finite $|\mathcal{Z}|$ in \cite{agueh2011barycenters}. For infinite $|\mathcal{Z}|$, the equation still holds due to the continuity of the covariance function on $(\mathcal{P}_{2,ac}, \mathcal{W}_2)$. The sufficiency and necessity of the equation follows from the following characterization of the barycenter via Brenier's maps $\{T_{\bar{X}X_z}\}_z$ derived in \cite{agueh2011barycenters}:
\begin{equation}
\int_{\mathcal{Z}} T_{\bar{X}X_z} d\lambda(z) = Id.
\end{equation}
It follows from the explicit form of $\{T_{\bar{X}X_z}\}_z$ in Lemma $\ref{l:Optimal Affine Map}$ that 
\begin{align*}
& \int_{\mathcal{Z}} T_{\bar{X}X_z} d\lambda(z) = \int_{\mathcal{Z}} \Sigma_{\bar{X}}^{-\frac{1}{2}} (\Sigma_{\bar{X}}^{\frac{1}{2}} \Sigma_{X_z} \Sigma_{\bar{X}}^{\frac{1}{2}} )^{\frac{1}{2}} \Sigma_{\bar{X}}^{-\frac{1}{2}} d\lambda(z) = Id\\
\iff &  \Sigma_{\bar{X}}^{\frac{1}{2}} \Sigma_{\bar{X}}^{-\frac{1}{2}} \int_{\mathcal{Z}} (\Sigma_{\bar{X}}^{\frac{1}{2}} \Sigma_{X_z} \Sigma_{\bar{X}}^{\frac{1}{2}} )^{\frac{1}{2}} d\lambda(z) \Sigma_{\bar{X}}^{-\frac{1}{2}} \Sigma_{\bar{X}}^{\frac{1}{2}} = \Sigma_{\bar{X}}^{\frac{1}{2}} Id \Sigma_{\bar{X}}^{\frac{1}{2}}\\
\iff & \int_{\mathcal{Z}} (\Sigma_{\bar{X}}^{\frac{1}{2}} \Sigma_{X_z} \Sigma_{\bar{X}}^{\frac{1}{2}} )^{\frac{1}{2}} d\lambda(z) = \Sigma_{\bar{X}}.
\end{align*}
\end{proof}

%\newpage

\section{Appendix: Proof of Results in Section 4} \label{A:Section 4 Appendix}

\subsection{Proof of Lemma \ref{l:Mccann Interpolation}}

\begin{proof}
First, it follows from the triangle inequality that $$\mathcal{W}_2(\mu_0,\mu_1) \leq \mathcal{W}_2(\mu_0,\mu_s) + \mathcal{W}_2(\mu_s,\mu_t) + \mathcal{W}_2(\mu_t,\mu_1)$$ for any $s,t \in [0,1]$. 
On the other hand, it follows from the definition of $\mu_t$ that for $s,t \in [0,1]$
\begin{align*}
\mathcal{W}_2^2(\mu_s,\mu_t) & \leq \int_{(\mathbb{R}^d)^2} ||x-y||^2 d(\pi_s)_{\sharp}\gamma(x) \otimes d(\pi_t)_{\sharp}\gamma(y)\\
& =  \int_{(\mathbb{R}^d)^2} ||\pi_s(x,y)-\pi_t(x,y)||^2 d\gamma(x,y) \\
& =  \int_{(\mathbb{R}^d)^2} ||(1-s)x + sy - (1-t)x - ty||^2 d\gamma(x,y)\\
& = \int_{(\mathbb{R}^d)^2} ||(t-s)x - (t-s)y||^2 d\gamma(x,y)\\
& = |t-s|^2\int_{(\mathbb{R}^d)^2} ||x - y||^2 d\gamma(x,y) = |t-s|^2\mathcal{W}_2^2(\mu_0,\mu_1),
\end{align*}
where the first equation results from definition of $\mathcal{W}_2$. 
Given the above two facts,  we complete the proof by contradiction. Assume $\exists s,t \in [0,1]$ such that $\mathcal{W}_2(\mu_s,\mu_t) < |t-s|\mathcal{W}_2(\mu_0,\mu_1)$, then 
\begin{align*}
\mathcal{W}_2(\mu_0,\mu_1) & \leq \mathcal{W}_2(\mu_0,\mu_s) + \mathcal{W}_2(\mu_s,\mu_t) + \mathcal{W}_2(\mu_t,\mu_1)\\
& < |s|\mathcal{W}_2(\mu_0,\mu_1) + |t-s|\mathcal{W}_2(\mu_0,\mu_1) + |1-t|\mathcal{W}_2(\mu_t,\mu_1)\\
& = \mathcal{W}_2(\mu_0,\mu_1).
\end{align*}
\end{proof}

\subsection{Proof of Theorem \ref{th:Geodesics Characterization of the Pareto Frontier}}

\begin{proof}
First, we derive the inequality from the triangle inequality and the optimality of $\{T(\cdot,z)\}_z$: Let $f: \mathcal{X} \times \mathcal{Z} \rightarrow \mathcal{Y}$ be an arbitrary measurable function. It follows that
\begin{align*}
V & \leq (\int_{\mathcal{Z}} ||\mathbb{E}(Y|X,Z)_z - \overline{f(X,Z)}_z ||_2^2 d\lambda(z))^{\frac{1}{2}} \\
& \leq L(f(X,Z)) + (\int_{\mathcal{Z}} ||f(X,Z)_z - \overline{f(X,Z)}_z ||_2^2 d\lambda(z))^{\frac{1}{2}}\\
& \leq L(f(X,Z)) + (\int_{\mathcal{Z}} \mathcal{W}_2^2( \mathcal{L}(f(X,Z)_z) ,\overline{\mathcal{L}(f(X,Z)_z)}) d\lambda(z))^{\frac{1}{2}}\\
& = L(f(X,Z)) + (\frac{1}{2} \int_{\mathcal{Z}^2} \mathcal{W}_2^2 (\mathcal{L}(f(X,Z)_{z_1}),\mathcal{L}(f(X,Z)_{z_2})) d\lambda(z_1)d\lambda(z_2))^{\frac{1}{2}}\\
& = L(f(X,Z)) + \frac{1}{\sqrt{2}} D(f(X,Z)).
\end{align*}
Here, the penultimate equation results from the fact that, for any $\{\nu_z\}_z \subset \mathcal{P}_{2,ac}(\mathbb{R}^d)$,
\begin{equation}
\int_{\mathcal{Z}^2} \mathcal{W}_2^2(\nu_{z_1},\nu_{z_2}) d\lambda(z_1)d\lambda(z_2) = 2 \int_{\mathcal{Z}} \mathcal{W}_2^2(\nu_{z},\bar{\nu}) d\lambda(z),
\end{equation}
where $\bar{\nu}$ is the Wasserstein barycenter of $\{\nu_z\}_z$.  Now, we show that the lower bound is achieved if and only if $f(X,Z) = T(t)(\mathbb{E}(Y|X,Z),Z), t \in [0,1]$. Let $ t \in [0,1]$, $T_z := T(\cdot,z)$, and $\mu_z :=  \mathcal{L}(\mathbb{E}(Y|X,Z)_z)$. It follows from Lemma \ref{l:Mccann Interpolation} and Remark \ref{r:Linear Interpolation Formula for Geodesic Path} that:

\begin{align*}
V & = (\int_{\mathcal{Z}} \mathcal{W}_2^2(\mu_z,\bar{\mu}) d\lambda(z))^{\frac{1}{2}}\\
& \leq (\int_{\mathcal{Z}} \mathcal{W}_2^2(\mu_z,T_z(t)_{\sharp}\mu_z) d\lambda(z))^{\frac{1}{2}} + (\int_{\mathcal{Z}} \mathcal{W}_2^2(T_z(t)_{\sharp}\mu_z,\bar{\mu}) d\lambda(z))^{\frac{1}{2}}\\
& = (t^2\int_{\mathcal{Z}} \mathcal{W}_2^2(\mu_z,\bar{\mu}) d\lambda(z))^{\frac{1}{2}} + ((1-t)^2\int_{\mathcal{Z}} \mathcal{W}_2^2(\mu_z,\bar{\mu}) d\lambda(z))^{\frac{1}{2}}\\
& = tV + (1-t)V = V.
\end{align*}
Therefore, the second inequality is an equality where the first term is $L(T(t))$:

\begin{align*}
L(T(t)) & = (\int_{\mathcal{Z}} ||\mathbb{E}(Y|X,Z)_z - T_z(t)(\mathbb{E}(Y|X,Z)_z )||_2^2 d\lambda(z))^{\frac{1}{2}}\\
& = (\int_{\mathcal{Z}} \mathcal{W}_2^2(\mu_z,T_z(t)_{\sharp}\mu_z) d\lambda(z))^{\frac{1}{2}}\\
& = t(\int_{\mathcal{Z}} \mathcal{W}_2^2(\mu_z,\bar{\mu}) d\lambda(z))^{\frac{1}{2}} =  tV.
\end{align*}

For the second term, we claim that it equals $\frac{1}{\sqrt{2}} D(T(t))$. To see this, we need to first show $\overline{T_z(t)_{\sharp} \mu_z} = \bar{\mu}$. Indeed, if not, then $\int_{\mathcal{Z}} \mathcal{W}_2^2(T_z(t)_{\sharp} \mu_z,\overline{T_z(t)_{\sharp} \mu_z}) d\lambda(z)$ is strictly less than $\int_{\mathcal{Z}} \mathcal{W}_2^2(T_z(t)_{\sharp} \mu_z,\bar{\mu}) d\lambda(z)$ by the definition and uniqueness of $\overline{T_z(t)_{\sharp} \mu_z}$. It follows that

\begin{align*}
& (\int_{\mathcal{Z}} \mathcal{W}_2^2(\mu_z,\overline{T_z(t)_{\sharp} \mu_z}) d\lambda(z))^{\frac{1}{2}}\\
\leq & (\int_{\mathcal{Z}} \mathcal{W}_2^2(\mu_z,T_z(t)_{\sharp} \mu_z) d\lambda(z))^{\frac{1}{2}} + (\int_{\mathcal{Z}} \mathcal{W}_2^2(T_z(t)_{\sharp} \mu_z,\overline{T_z(t)_{\sharp} \mu_z}) d\lambda(z))^{\frac{1}{2}}\\
< & L(T(t)) + (\int_{\mathcal{Z}} \mathcal{W}_2^2(T_z(t)_{\sharp} \mu_z,\bar{\mu}) d\lambda(z))^{\frac{1}{2}}\\
= & (\int_{\mathcal{Z}} \mathcal{W}_2^2(\mu_z,\bar{\mu}) d\lambda(z))^{\frac{1}{2}},
\end{align*}
which contradicts the definition and uniqueness of $\bar{\mu}$. Therefore,

\begin{align*}
D(T(t)) & =( \int_{\mathcal{Z}^2} \mathcal{W}_2^2(T_{z_1}(t)_{\sharp}\mu_{z_1},T_{z_2}(t)_{\sharp}\mu_{z_2}) d\lambda(z_1)d\lambda(z_2))^{\frac{1}{2}}\\
& = (2\int_{\mathcal{Z}} \mathcal{W}_2^2(T_z(t)_{\sharp}\mu_z,\overline{T_z(t)_{\sharp}\mu_z}) d\lambda(z))^{\frac{1}{2}}\\
& = \sqrt{2}(\int_{\mathcal{Z}} \mathcal{W}_2^2(T_z(t)_{\sharp}\mu_z,\bar{\mu}) d\lambda(z))^{\frac{1}{2}}\\
& = \sqrt{2}((1-t)^2\int_{\mathcal{Z}} \mathcal{W}_2^2(\mu_z,\bar{\mu}) d\lambda(z))^{\frac{1}{2}}\\
& = \sqrt{2}(1-t)V.
\end{align*}
That completes the proof.
\end{proof}

\section{Appendix: Proof of Results in Section 5} \label{A:Section 5 Appendix}

\subsection{Proof of $\tilde{X} \perp Z$ implies $\mathbb{E}(Y_z|\tilde{X}) = \mathbb{E}(Y|\tilde{X},Z)_z$}

\begin{proof}
Let $\tilde{X} \perp Z$ and assume for contradiction that $\mathbb{E}(Y_z|\tilde{X}) \neq \mathbb{E}(Y|\tilde{X},Z)_z$. Then, we have
\begin{align*}
||Y - \mathbb{E}(Y|\tilde{X},Z)||_2^2
& = \int_{\mathcal{Z}} ||Y_z - f^*(\tilde{X},Z)_z||_2^2 d\lambda\\
& = \int_{\mathcal{Z}} ||Y_z - f^*(\tilde{X},z)||_2^2 d\lambda\\
& > \int_{\mathcal{Z}} ||Y_z - \mathbb{E}(Y_z|\tilde{X})||_2^2 d\lambda\\
& = \int_{\mathcal{Z}} ||Y_z - \tilde{f}_z(\tilde{X})||_2^2 d\lambda\\
\end{align*}
where the first line follows from disintegration and the fact that there exists a measurable function $f^*: \mathcal{X} \times \mathcal{Z} \rightarrow \mathcal{Y}$ such that $f^*(\tilde{X},Z) = \mathbb{E}(Y|\tilde{X},Z)$, the second from $\tilde{X} \perp Z$, the third line follows from orthogonal projection property of conditional expectation and the assumption, and the forth from the fact that there exists a measurable function $\tilde{f}_z: \mathcal{X}  \rightarrow \mathcal{Y}$ such that $\tilde{f}_z(\tilde{X}) = \mathbb{E}(Y_z|\tilde{X})$.
Now, define $\tilde{f}: \mathcal{X} \times \mathcal{Z} \rightarrow \mathcal{Y}$ by $\tilde{f}(\cdot,z) := \tilde{f}_z$ for $\lambda$-a.e. $z \in \mathcal{Z}$. It follows that
\begin{align*}
||Y - \mathbb{E}(Y|\tilde{X},Z)||_2^2
& > \int_{\mathcal{Z}} ||Y_z - \tilde{f}_z(\tilde{X})||_2^2 d\lambda\\
& = \int_{\mathcal{Z}} ||Y_z - \tilde{f}(\tilde{X},z)||_2^2 d\lambda\\
& = ||Y - \tilde{f}(\tilde{X},Z)||_2^2\\
& = ||Y - \mathbb{E}(Y|\tilde{X},Z)||_2^2 + ||\mathbb{E}(Y|\tilde{X},Z) - \tilde{f}(\tilde{X},Z)||_2^2.
\end{align*}
That implies $ ||\mathbb{E}(Y|\tilde{X},Z) - \tilde{f}(\tilde{X},Z)||_2^2 < 0$, a contradiction. This completes the proof.
\end{proof}

\subsection{Proof of Lemma \ref{l:Finer Sigma-algebra Implies More Accurate Optimal Fair Learning}}

\begin{proof}
Let $\tilde{X}, \tilde{X}' \in \{ \tilde{X} \in \mathcal{D}_{\mathcal{X}}: \tilde{X} \perp Z\}$ satisfy $\sigma(\tilde{X}') \subset \sigma(\tilde{X})$. We have
\begin{align*}
& ||\mathbb{E}(Y|X,Z) - \mathbb{E}(\bar{Y}|\tilde{X},Z)||_2^2 - ||\mathbb{E}(Y|X,Z) - \mathbb{E}(\bar{Y}'|\tilde{X}',Z)||_2^2\\
= & ||\mathbb{E}(Y|X,Z) - \overline{\mathbb{E}(Y|\tilde{X},Z)}||_2^2 - ||\mathbb{E}(Y|X,Z) - \overline{\mathbb{E}(Y|\tilde{X}',Z)}||_2^2
\end{align*}
Notice that
\begin{equation*}
||\mathbb{E}(Y|X,Z) - \overline{\mathbb{E}(Y|\tilde{X},Z)}||_2^2 = ||\mathbb{E}(Y|X,Z) - \mathbb{E}(Y|\tilde{X},Z)||_2^2 + \int_{\mathcal{Z}} \mathcal{W}_2^2(\mu_z, \bar{\mu}) d\lambda
\end{equation*}
where $\mu_z := \mathcal{L}(\mathbb{E}(Y|\tilde{X},Z)_z)$ and $\bar{\mu} := \overline{\mathcal{L}(\mathbb{E}(Y|\tilde{X},Z))}$. Also, we define $\mu'_z$ and $\bar{\mu}'$ analogously to have
\begin{align*}
& ||\mathbb{E}(Y|X,Z) - \overline{\mathbb{E}(Y|\tilde{X}',Z)}||_2^2\\
= & ||\mathbb{E}(Y|X,Z) - \mathbb{E}(Y|\tilde{X}',Z)||_2^2 + \int_{\mathcal{Z}} \mathcal{W}_2^2(\mu'_z, \bar{\mu}') d\lambda\\
= & ||\mathbb{E}(Y|X,Z) - \mathbb{E}(Y|\tilde{X},Z)||_2^2 + ||\mathbb{E}(Y|\tilde{X},Z) - \mathbb{E}(Y|\tilde{X}',Z)||_2^2 + \int_{\mathcal{Z}} \mathcal{W}_2^2(\mu'_z, \bar{\mu}') d\lambda.
\end{align*}
Combining the above, we have
\begin{align*}
& ||\mathbb{E}(Y|X,Z) - \mathbb{E}(\bar{Y}|\tilde{X},Z)||_2^2 - ||\mathbb{E}(Y|X,Z) - \mathbb{E}(\bar{Y}'|\tilde{X}',Z)||_2^2\\
= & \int_{\mathcal{Z}} \mathcal{W}_2^2(\mu_z, \bar{\mu})d\lambda - \int_{\mathcal{Z}} \mathcal{W}_2^2(\mu'_z, \bar{\mu}')d\lambda - ||\mathbb{E}(Y|\tilde{X},Z) - \mathbb{E}(Y|\tilde{X}',Z)||_2^2.
\end{align*}
It remains to show that $\int_{\mathcal{Z}} \mathcal{W}_2^2(\mu_z, \bar{\mu})d\lambda < \int_{\mathcal{Z}} \mathcal{W}_2^2(\mu'_z, \bar{\mu}')d\lambda + ||\mathbb{E}(Y|\tilde{X},Z) - \mathbb{E}(Y|\tilde{X}',Z)||_2^2$. Indeed, assume for contradiction that $\int_{\mathcal{Z}} \mathcal{W}_2^2(\mu'_z, \bar{\mu}')d\lambda + ||\mathbb{E}(Y|\tilde{X},Z) - \mathbb{E}(Y|\tilde{X}',Z)||_2^2 \leq \int_{\mathcal{Z}} \mathcal{W}_2^2(\mu_z, \bar{\mu})d\lambda$, then we have
\begin{align*}
\int_{\mathcal{Z}} \mathcal{W}_2^2(\mu_z, \bar{\mu}')d\lambda & \leq ||\mathbb{E}(Y|\tilde{X},Z) - \mathbb{E}(Y|\tilde{X}',Z)||_2^2 + \int_{\mathcal{Z}} \mathcal{W}_2^2(\mu'_z, \bar{\mu}')d\lambda\\
& \leq \int_{\mathcal{Z}} \mathcal{W}_2^2(\mu_z, \bar{\mu})d\lambda.
\end{align*}
This contradicts the optimality and uniqueness of $\bar{\mu}$ by Lemma \ref{l:Optimal Fair $L^2$-Objective Supervised Learning Characterization}. Therefore, we prove by contradiction that $\int_{\mathcal{Z}} \mathcal{W}_2^2(\mu_z, \bar{\mu})d\lambda < \int_{\mathcal{Z}} \mathcal{W}_2^2(\mu'_z, \bar{\mu}')d\lambda + ||\mathbb{E}(Y|\tilde{X},Z) - \mathbb{E}(Y|\tilde{X}',Z)||_2^2$ and, hence, $$||\mathbb{E}(Y|X,Z) - \mathbb{E}(\bar{Y}|\tilde{X},Z)||_2^2 - ||\mathbb{E}(Y|X,Z) - \mathbb{E}(\bar{Y}'|\tilde{X}',Z)||_2^2 < 0.$$ That completes the proof.
\end{proof}

\subsection{Proof of Lemma \ref{l:X Bar Generates the Finest Sigma-algebra}}

\begin{proof}
We first prove $\sigma((\bar{X},Z)) = \sigma((X,Z))$. Since $\mathcal{L}(X_z) \subset \mathcal{P}_{2,ac}$, it follows from Lemma~\ref{l:Optimal Fair $L^2$-Objective Supervised Learning Characterization} that there exists a measurable map $T: \mathcal{X} \times \mathcal{Z} \rightarrow \mathcal{X}$ such that $T(X_z,z) = \bar{X}_z$ $\lambda$-a.e., where $\bar{X}$ denotes the Wasserstein barycenter of $\{X_z\}_z$. Define $T \otimes Id|_{\mathcal{Z}} : \mathcal{X} \times \mathcal{Z} \rightarrow \mathcal{X} \times \mathcal{Z}$, we have $T \otimes Id|_{\mathcal{Z}}$ is $\mathcal{X} \times \mathcal{Z}/\mathcal{X} \times \mathcal{Z}$-measurable and satisfies $T \otimes Id|_{\mathcal{Z}}((X,Z)) = (\bar{X},Z)$. That implies $\sigma((\bar{X},Z)) \subset \sigma((X,Z))$. Furthermore, since $\mathcal{L}(\bar{X}) \in \mathcal{P}_{2,ac}$, it follows from Brenier's theorem \cite{brenier1991polar} that there exists $T^{-1}(\cdot,z)$ such that $T^{-1}(\bar{X}_z,z) = X_z$. Therefore, we have $(T \otimes Id|_{\mathcal{Z}})^{-1} = T^{-1} \otimes Id|_{\mathcal{Z}}$ is $\mathcal{X} \times \mathcal{Z}/\mathcal{X} \times \mathcal{Z}$-measurable and satisfies $(T \otimes Id|_{\mathcal{Z}})^{-1}((\bar{X},Z)) = (X,Z)$.  That implies $\sigma((X,Z))  \subset \sigma((\bar{X},Z))$. That completes the proof of $\sigma((\bar{X},Z)) = \sigma((X,Z))$.
Now, we show $\sigma(\tilde{X}) \subset \sigma(\bar{X})$. From the construction of $\tilde{X}$, we have $\sigma((\tilde{X},Z)) \subset \sigma((\bar{X},Z)) = \sigma((X,Z))$.  But $\tilde{X} \perp Z$ implies that, for any $B_X \in \mathcal{B}_{\mathcal{X}}$, we can construct $B_X \times \mathcal{Z} \in \mathcal{B}_{\mathcal{X}} \otimes \mathcal{B}_{\mathcal{Z}}$. In addition, due to $\sigma((\tilde{X},Z)) \subset \sigma((\bar{X},Z))$, there exists $B_{XZ}' \in \mathcal{B}_{\mathcal{X}} \otimes \mathcal{B}_{\mathcal{Z}}$ such that $(\bar{X},Z)^{-1}(B_{XZ}') = (X,Z)^{-1}(B_X \times \mathcal{Z})$. Lastly, $\bar{X} \perp Z$ also implies that there exists $B_X' \in \mathcal{B}_{\mathcal{X}}$ satisfying $B_{XZ}' = B_X' \times \mathcal{Z}$. It follows that
\begin{equation}
\tilde{X}^{-1}(B_X) = (\tilde{X},Z)^{-1}(B_X \times \mathcal{Z}) = (X,Z)^{-1}(B_X' \times \mathcal{Z}) = X^{-1}(B_X')
\end{equation}
Since our choice of $B_X \in \mathcal{B}_{\mathcal{X}}$ is arbitrary, it follows that $\sigma(\tilde{X}) \subset \sigma(\bar{X})$. Finally, since our choice of $\tilde{X} \in \{\tilde{X} \in \mathcal{D}|_{\mathcal{X}} : \tilde{X} \perp Z \}$ is arbitrary, we are done.
\end{proof}

\vskip 0.2in
\bibliography{References_Fairness_Barycenter}

\begin{thebibliography}{45}
\providecommand{\natexlab}[1]{#1}
\providecommand{\url}[1]{\texttt{#1}}
\expandafter\ifx\csname urlstyle\endcsname\relax
  \providecommand{\doi}[1]{doi: #1}\else
  \providecommand{\doi}{doi: \begingroup \urlstyle{rm}\Url}\fi

\bibitem[Adamson(2011)]{adamson2011ricci}
B.~L. Adamson.
\newblock Ricci v. {D}e{S}tefano: Procedural {A}ctivism (?).
\newblock \emph{National Black Law Journal (University of California, Los
  Angeles)}, 24:\penalty0 11--01, 2011.

\bibitem[Agueh and Carlier(2011)]{agueh2011barycenters}
M.~Agueh and G.~Carlier.
\newblock Barycenters in the {W}asserstein space.
\newblock \emph{SIAM Journal on Mathematical Analysis}, 43\penalty0
  (2):\penalty0 904--924, 2011.

\bibitem[Altschuler and Boix-Adsera(2022)]{altschuler2022wasserstein}
J.~M. Altschuler and E.~Boix-Adsera.
\newblock Wasserstein barycenters are {NP}-hard to compute.
\newblock \emph{SIAM Journal on Mathematics of Data Science}, 4\penalty0
  (1):\penalty0 179--203, 2022.

\bibitem[{\'A}lvarez-Esteban et~al.(2016){\'A}lvarez-Esteban, Del~Barrio,
  Cuesta-Albertos, and Matr{\'a}n]{alvarez2016fixed}
P.~C. {\'A}lvarez-Esteban, E.~Del~Barrio, J.~Cuesta-Albertos, and
  C.~Matr{\'a}n.
\newblock A fixed-point approach to barycenters in {W}asserstein space.
\newblock \emph{Journal of Mathematical Analysis and Applications},
  441\penalty0 (2):\penalty0 744--762, 2016.

\bibitem[Angwin et~al.(2022)Angwin, Larson, Mattu, and
  Kirchner]{angwin2022machine}
J.~Angwin, J.~Larson, S.~Mattu, and L.~Kirchner.
\newblock Machine {B}ias.
\newblock In \emph{Ethics of data and analytics}, pages 254--264. Auerbach
  Publications, 2022.

\bibitem[Aristotle et~al.(1984)]{aristotle1984complete}
J.~B. Aristotle et~al.
\newblock \emph{The complete works of Aristotle}, volume~2.
\newblock Princeton University Press Princeton, 1984.

\bibitem[Asuncion and Newman(2007)]{asuncion2007uci}
A.~Asuncion and D.~Newman.
\newblock {UCI} machine learning repository, 2007.

\bibitem[Berk et~al.(2017)Berk, Heidari, Jabbari, Joseph, Kearns, Morgenstern,
  Neel, and Roth]{berk2017convex}
R.~Berk, H.~Heidari, S.~Jabbari, M.~Joseph, M.~Kearns, J.~Morgenstern, S.~Neel,
  and A.~Roth.
\newblock A convex framework for fair regression.
\newblock \emph{arXiv preprint arXiv:1706.02409}, 2017.

\bibitem[Bhatia(2009)]{bhatia2009positive}
R.~Bhatia.
\newblock \emph{Positive Definite Matrices}.
\newblock Princeton University Press, 2009.

\bibitem[Blumrosen(1972)]{blumrosen1972strangers}
A.~W. Blumrosen.
\newblock Strangers in paradise: Griggs v. {D}uke {P}ower {C}o. and the concept
  of employment discrimination.
\newblock \emph{Mich. L. Rev.}, 71:\penalty0 59, 1972.

\bibitem[Brenier(1991)]{brenier1991polar}
Y.~Brenier.
\newblock Polar factorization and monotone rearrangement of vector-valued
  functions.
\newblock \emph{Communications on pure and applied mathematics}, 44\penalty0
  (4):\penalty0 375--417, 1991.

\bibitem[Calders and {\v{Z}}liobait{\.e}(2013)]{calders2013unbiased}
T.~Calders and I.~{\v{Z}}liobait{\.e}.
\newblock Why unbiased computational processes can lead to discriminative
  decision procedures.
\newblock In \emph{Discrimination and Privacy in the Information Society: Data
  mining and profiling in large databases}, pages 43--57. Springer, 2013.

\bibitem[Calmon et~al.(2017)Calmon, Wei, Vinzamuri, Natesan~Ramamurthy, and
  Varshney]{calmon2017optimized}
F.~Calmon, D.~Wei, B.~Vinzamuri, K.~Natesan~Ramamurthy, and K.~R. Varshney.
\newblock Optimized pre-processing for discrimination prevention.
\newblock \emph{Advances in neural information processing systems}, 30, 2017.

\bibitem[Cao and Yang(2015)]{cao2015towards}
Y.~Cao and J.~Yang.
\newblock Towards making systems forget with machine unlearning.
\newblock In \emph{2015 IEEE symposium on security and privacy}, pages
  463--480. IEEE, 2015.

\bibitem[Carlier and Ekeland(2010)]{carlier2010matching}
G.~Carlier and I.~Ekeland.
\newblock Matching for teams.
\newblock \emph{Economic theory}, 42:\penalty0 397--418, 2010.

\bibitem[Chouldechova and Roth(2018)]{chouldechova2018frontiers}
A.~Chouldechova and A.~Roth.
\newblock The frontiers of fairness in machine learning.
\newblock \emph{arXiv preprint arXiv:1810.08810}, 2018.

\bibitem[Christian(2020)]{christian2020alignment}
B.~Christian.
\newblock \emph{The alignment problem: Machine learning and human values}.
\newblock WW Norton \& Company, 2020.

\bibitem[Chzhen et~al.(2020)Chzhen, Denis, Hebiri, Oneto, and
  Pontil]{chzhen2020fair}
E.~Chzhen, C.~Denis, M.~Hebiri, L.~Oneto, and M.~Pontil.
\newblock Fair regression with {W}asserstein barycenters.
\newblock \emph{Advances in Neural Information Processing Systems},
  33:\penalty0 7321--7331, 2020.

\bibitem[Cooper et~al.(1997)Cooper, Hutchinson, et~al.]{cooper1997plato}
J.~M. Cooper, D.~S. Hutchinson, et~al.
\newblock \emph{Plato: complete works}.
\newblock Hackett Publishing, 1997.

\bibitem[Cuesta-Albertos et~al.(1996)Cuesta-Albertos, Matr{\'a}n-Bea, and
  Tuero-Diaz]{cuesta1996lower}
J.~A. Cuesta-Albertos, C.~Matr{\'a}n-Bea, and A.~Tuero-Diaz.
\newblock On lower bounds for the l2-{W}asserstein metric in a {H}ilbert space.
\newblock \emph{Journal of Theoretical Probability}, 9\penalty0 (2):\penalty0
  263--283, 1996.

\bibitem[Dwork et~al.(2012)Dwork, Hardt, Pitassi, Reingold, and
  Zemel]{dwork2012fairness}
C.~Dwork, M.~Hardt, T.~Pitassi, O.~Reingold, and R.~Zemel.
\newblock Fairness through awareness.
\newblock In \emph{Proceedings of the 3rd innovations in theoretical computer
  science conference}, pages 214--226, 2012.

\bibitem[Ekeland(2010)]{ekeland2010existence}
I.~Ekeland.
\newblock Existence, uniqueness and efficiency of equilibrium in hedonic
  markets with multidimensional types.
\newblock \emph{Economic Theory}, 42:\penalty0 275--315, 2010.

\bibitem[Feldman et~al.(2015)Feldman, Friedler, Moeller, Scheidegger, and
  Venkatasubramanian]{feldman2015certifying}
M.~Feldman, S.~A. Friedler, J.~Moeller, C.~Scheidegger, and
  S.~Venkatasubramanian.
\newblock Certifying and removing disparate impact.
\newblock In \emph{proceedings of the 21th ACM SIGKDD international conference
  on knowledge discovery and data mining}, pages 259--268, 2015.

\bibitem[Gouic et~al.(2020)Gouic, Loubes, and Rigollet]{gouic2020projection}
T.~L. Gouic, J.-M. Loubes, and P.~Rigollet.
\newblock Projection to fairness in statistical learning.
\newblock \emph{arXiv preprint arXiv:2005.11720}, 2020.

\bibitem[Hajian and Domingo-Ferrer(2012)]{hajian2012methodology}
S.~Hajian and J.~Domingo-Ferrer.
\newblock A methodology for direct and indirect discrimination prevention in
  data mining.
\newblock \emph{IEEE transactions on knowledge and data engineering},
  25\penalty0 (7):\penalty0 1445--1459, 2012.

\bibitem[Hardt et~al.(2016)Hardt, Price, and Srebro]{hardt2016equality}
M.~Hardt, E.~Price, and N.~Srebro.
\newblock Equality of opportunity in supervised learning.
\newblock \emph{Advances in neural information processing systems}, 29, 2016.

\bibitem[Hu and Chen(2018)]{hu2018short}
L.~Hu and Y.~Chen.
\newblock A short-term intervention for long-term fairness in the labor market.
\newblock In \emph{Proceedings of the 2018 World Wide Web Conference}, pages
  1389--1398, 2018.

\bibitem[Jiang et~al.(2020)Jiang, Pacchiano, Stepleton, Jiang, and
  Chiappa]{jiang2020wasserstein}
R.~Jiang, A.~Pacchiano, T.~Stepleton, H.~Jiang, and S.~Chiappa.
\newblock {W}asserstein fair classification.
\newblock In \emph{Uncertainty in artificial intelligence}, pages 862--872.
  PMLR, 2020.

\bibitem[Kamiran and Calders(2012)]{kamiran2012data}
F.~Kamiran and T.~Calders.
\newblock Data preprocessing techniques for classification without
  discrimination.
\newblock \emph{Knowledge and information systems}, 33\penalty0 (1):\penalty0
  1--33, 2012.

\bibitem[Kim and Pass(2017)]{kim2017wasserstein}
Y.-H. Kim and B.~Pass.
\newblock {W}asserstein barycenters over {R}iemannian manifolds.
\newblock \emph{Advances in Mathematics}, 307:\penalty0 640--683, 2017.

\bibitem[Le~Gouic and Loubes(2017)]{le2017existence}
T.~Le~Gouic and J.-M. Loubes.
\newblock Existence and consistency of {W}asserstein barycenters.
\newblock \emph{Probability Theory and Related Fields}, 168:\penalty0 901--917,
  2017.

\bibitem[McCann(1997)]{mccann1997convexity}
R.~J. McCann.
\newblock A convexity principle for interacting gases.
\newblock \emph{Advances in mathematics}, 128\penalty0 (1):\penalty0 153--179,
  1997.

\bibitem[of~the President(2014)]{executive2014big}
E.~O. of~the President.
\newblock Big data: Seizing opportunities, preserving values.
\newblock \emph{President PACT report}, 2014.

\bibitem[Pass(2013)]{pass2013optimal}
B.~Pass.
\newblock Optimal transportation with infinitely many marginals.
\newblock \emph{Journal of Functional Analysis}, 264\penalty0 (4):\penalty0
  947--963, 2013.

\bibitem[Redmond and Baveja(2002)]{redmond2002data}
M.~Redmond and A.~Baveja.
\newblock A data-driven software tool for enabling cooperative information
  sharing among police departments.
\newblock \emph{European Journal of Operational Research}, 141\penalty0
  (3):\penalty0 660--678, 2002.

\bibitem[Santambrogio(2015)]{santambrogio2015optimal}
F.~Santambrogio.
\newblock Optimal transport for applied mathematicians.
\newblock \emph{Birk{\"a}user, NY}, 55\penalty0 (58-63):\penalty0 94, 2015.

\bibitem[Silvia et~al.(2020)Silvia, Ray, Tom, Aldo, Heinrich, and
  John]{silvia2020general}
C.~Silvia, J.~Ray, S.~Tom, P.~Aldo, J.~Heinrich, and A.~John.
\newblock A general approach to fairness with optimal transport.
\newblock In \emph{Proceedings of the AAAI Conference on Artificial
  Intelligence}, volume 34(04), pages 3633--3640, 2020.

\bibitem[Sweeney(2013)]{sweeney2013discrimination}
L.~Sweeney.
\newblock Discrimination in online ad delivery: Google ads, black names and
  white names, racial discrimination, and click advertising.
\newblock \emph{Queue}, 11\penalty0 (3):\penalty0 10--29, 2013.

\bibitem[Tabak and Trigila(2018)]{tabak2018explanation}
E.~G. Tabak and G.~Trigila.
\newblock Explanation of variability and removal of confounding factors from
  data through optimal transport.
\newblock \emph{Communications on Pure and Applied Mathematics}, 71\penalty0
  (1):\penalty0 163--199, 2018.

\bibitem[Villani(2021)]{villani2021topics}
C.~Villani.
\newblock \emph{Topics in optimal transportation}, volume~58.
\newblock American Mathematical Soc., 2021.

\bibitem[Villani et~al.(2009)]{villani2009optimal}
C.~Villani et~al.
\newblock \emph{Optimal transport: old and new}, volume 338.
\newblock Springer, 2009.

\bibitem[Wightman(1998)]{wightman1998lsac}
L.~F. Wightman.
\newblock {LSAC} {N}ational {L}ongitudinal {B}ar {P}assage {S}tudy. {LSAC}
  {R}esearch {R}eport {S}eries.
\newblock 1998.

\bibitem[Zafar et~al.(2017)Zafar, Valera, Gomez~Rodriguez, and
  Gummadi]{zafar2017fairness}
M.~B. Zafar, I.~Valera, M.~Gomez~Rodriguez, and K.~P. Gummadi.
\newblock Fairness beyond disparate treatment \& disparate impact: Learning
  classification without disparate mistreatment.
\newblock In \emph{Proceedings of the 26th international conference on world
  wide web}, pages 1171--1180, 2017.

\bibitem[Zemel et~al.(2013)Zemel, Wu, Swersky, Pitassi, and
  Dwork]{zemel2013learning}
R.~Zemel, Y.~Wu, K.~Swersky, T.~Pitassi, and C.~Dwork.
\newblock Learning fair representations.
\newblock In \emph{International conference on machine learning}, pages
  325--333. PMLR, 2013.

\bibitem[Zhou et~al.(2021)Zhou, Zhang, Nair, Singhal, Chen, and
  Sudjianto]{zhou2021bias}
N.~Zhou, Z.~Zhang, V.~N. Nair, H.~Singhal, J.~Chen, and A.~Sudjianto.
\newblock Bias, {F}airness, and {A}ccountability with {AI} and {ML}
  {A}lgorithms.
\newblock \emph{arXiv preprint arXiv:2105.06558}, 2021.

\end{thebibliography}

\end{document}